\documentclass[11pt]{article}
\usepackage{amsmath, amssymb, amsthm}
\usepackage{bm}
\usepackage{wrapfig}
\usepackage{graphicx}
\usepackage{subfig}
\usepackage{algorithm}
\usepackage{algorithmic}
\usepackage{fullpage}
\graphicspath{{figs/}}

\usepackage[svgnames]{xcolor}
\usepackage{listings}
\usepackage{mathrsfs} 

\usepackage{hyperref}

\usepackage{caption}
\usepackage[utf8]{inputenc} 
\usepackage[T1]{fontenc}    
\usepackage{lmodern}

\usepackage{booktabs}       
\usepackage{amsfonts}       
\usepackage{nicefrac}       
\usepackage{microtype}      

\newcommand\goto{\rightarrow}

\numberwithin{equation}{section}

\DeclareMathOperator*{\plim}{plim}
\DeclareMathOperator*{\argmin}{argmin}
\DeclareMathOperator*{\PP}{\mathbb{P}}
\DeclareMathOperator*{\E}{\mathbb{E}}

\DeclareMathOperator{\prox}{\textup{prox}}

\DeclareMathOperator{\supp}{supp}
\DeclareMathOperator{\sgn}{sign}
\def\ve{\varepsilon}

\def\im{{\rm im}}

\newcommand{\supps}{\mathrm{supp}^{\star}}

\newcommand{\R}{\mathbb{R}}

\newcommand{\be}{\begin{equation}}
	\newcommand{\ee}{\end{equation}}
\newcommand{\ben}{\begin{equation*}}
	\newcommand{\een}{\end{equation*}}

\newcommand{\norm}[1]{\lVert#1\rVert}
\renewcommand{\mathbf}{\bm}

\newcommand{\bet}{\boldsymbol{\beta}}
\newcommand{\p}{p}
\newcommand{\n}{n}

\newcommand{\B}{\mathbf{B}}
\newcommand{\X}{\mathbf{X}}
\newcommand{\w}{\mathbf{w}}
\newcommand{\y}{\mathbf{y}}
\newcommand{\x}{\mathbf{x}}
\newcommand{\rr}{\mathbf{r}}
\newcommand{\blam}{\boldsymbol \lambda}
\newcommand{\thet}{\boldsymbol \theta}
\newcommand{\z}{\mathbf{z}}
\newcommand{\F}{\textsf{F}}
\newcommand{\bfalph}{\boldsymbol \alpha}


\title{Algorithmic Analysis and Statistical Estimation of SLOPE via Approximate Message Passing}
\author{Zhiqi Bu\thanks{Department of Applied Mathematics and Computational Science, University of Pennsylvania, Philadelphia, PA 19104, USA. Email: {\tt zbu@sas.upenn.edu} } 
 \and Jason Klusowski\thanks{Department of Statistics, Rutgers University, New Brunswick, NJ 08854, USA. Email: {\tt jason.klusowski@rutgers.edu}} 
 \and Cynthia Rush\thanks{Department of Statistics, Columbia University, New York, NY 10027, USA. Email: {\tt cynthia.rush@columbia.edu}} 
  \and Weijie Su\thanks{Department of Statistics, University of Pennsylvania, Philadelphia, PA 19104, USA. Email: {\tt suw@wharton.upenn.edu }  This work was supported in part by NSF $\#1217023$.} 
}

\newtheorem{theorem}{Theorem}
\newtheorem{othertheorem}{othertheorem}[section]
\newtheorem{lemma}[othertheorem]{Lemma}
\newtheorem{corollary}[othertheorem]{Corollary}
\newtheorem{proposition}[othertheorem]{Proposition}

\newtheorem{definition}[othertheorem]{Definition}
\newtheorem{fact}[othertheorem]{Fact}
\newtheorem{rem}[othertheorem]{Remark}

\newtheorem{theoremG}{Theorem G.\ignorespaces}

\begin{document}
\maketitle

\begin{abstract}
SLOPE is a relatively new convex optimization procedure for high-dimensional linear regression via the sorted $\ell_1$ penalty: the larger the rank of the fitted coefficient, the larger the penalty. This non-separable penalty renders many existing techniques invalid or inconclusive in analyzing the SLOPE solution. In this paper, we develop an asymptotically exact characterization of the SLOPE solution under Gaussian random designs through solving the SLOPE problem using approximate message passing (AMP). This algorithmic approach allows us to approximate the SLOPE solution via the much more amenable AMP iterates. Explicitly, we characterize the asymptotic dynamics of the AMP iterates relying on a recently developed state evolution analysis for non-separable penalties, thereby overcoming the difficulty caused by the sorted $\ell_1$ penalty. Moreover, we prove that the AMP iterates converge to the SLOPE solution in an asymptotic sense, and numerical simulations show that the convergence is surprisingly fast. Our proof rests on a novel technique that specifically leverages the SLOPE problem. In contrast to prior literature, our work not only yields an asymptotically sharp analysis but also offers an algorithmic, flexible, and constructive approach to understanding the SLOPE problem.
\end{abstract}


\section{Introduction}\label{sec_intro}

Consider observing linear measurements $\y \in \R^{\n}$ that are modeled by the equation
\be
\y = \X \bet + \w,
\label{eq:model}
\ee
where $\X \in \R^{\n \times \p}$ is a known measurement matrix, $\bet \in \R^p$ is an unknown signal, and $\w \in \R^{\n}$ is the measurement noise. Among numerous methods that seek to recover the signal $\bet$ from the observed data, especially in the setting where $\bet$ is sparse and $p$ is larger than $n$, SLOPE has recently emerged as a useful procedure that allows for estimation and model selection \cite{SLOPE1}. This method reconstructs the signal by solving the minimization problem
\be
\widehat{\bet} := \arg \min_{\bm b} ~ \frac{1}{2}\| \y - \X \bm b\|^2 + \sum_{i=1}^{\p} \lambda_i|\bm b|_{(i)},
\label{eq:SLOPE_est}
\ee
where $\|\cdot\|$ denotes the $\ell_2$ norm, $\lambda_1\geq \cdots \geq \lambda_p\geq 0$ (with at least one strict inequality) is a sequence of thresholds, and $|\bm b|_{(1)}\geq \cdots \geq |\bm b|_{(p)}$ are the order statistics of the fitted coefficients in absolute value. The regularizer $\sum \lambda_i|\bm b|_{(i)}$ is a \textit{sorted $\ell_1$-norm} (denoted as $J_{\blam}(\bm b)$ henceforth), which is \textit{non-separable} due to the sorting operation involved in its calculation. Notably, SLOPE has two attractive features that are not simultaneously present in other methods for linear regression including the LASSO \cite{lasso_paper} and knockoffs \cite{barber2015controlling}. Explicitly, on the estimation side, SLOPE achieves minimax estimation properties under certain random designs \textit{without} requiring any knowledge of the sparsity degree of $\bet$ \cite{SLOPE2,bellec2018slope}. On the testing side, SLOPE controls the false discovery rate in the case of independent
predictors \cite{SLOPE1,brzyski2018group}. For completeness, we remark that \cite{bondell2008simultaneous,zeng2014decreasing,figueiredo2016ordered} proposed similar non-separable regularizers to encourage grouping of correlated
predictors.

This work is concerned with the algorithmic aspects of SLOPE through the lens of \textit{approximate message passing} (AMP) \cite{amp2, amp1, krz12, Rangan11}. AMP is a class of computationally efficient and easy-to-implement algorithms for a broad range of statistical estimation problems, including compressed sensing and the LASSO \cite{lassorisk}. When applied to SLOPE, AMP takes the following form: at initial iteration $t=0$, assign $\bet^{0}=\bm 0, \z^{0}=\bm y$, and for $t \geq 0,$
\begin{subequations}\label{eq:amp_slope}
\begin{align}
\bet^{t+1}&=\prox_{J_{\thet_t}}(\X^\top \z^t+\bet^t),
\label{eq:AMP0}
\\
\z^{t+1}&=\y-\X\bet^{t+1}+\frac{\z^{t}}{n}\Big[\nabla \prox_{J_{\thet_{t}}} (\X^\top \z^{t}+\bet^{t})\Big].   \label{eq:AMP1}
\end{align}
\end{subequations}
The non-increasing sequence $\thet_t$ is proportional to $\bm\lambda = (\lambda_1, \lambda_2, \ldots, \lambda_p)$ and will be given explicitly in Section \ref{sec:calibr-betw-bmlambda}. Here, $\prox_{J_{\thet}}$ is the proximal operator of the sorted $\ell_1$ norm, that is,
\be
\label{eq:prox}
\prox_{J_{\thet}}(\bm x) := \argmin_{\bm b} ~\frac12 \|\bm x - \bm b\|^2 + J_{\thet}(\bm b),
\ee
and $\nabla \prox_{J_{\theta}}$ denotes the divergence of the proximal operator (see an equivalent, but more explicit form, of this algorithm in Section~\ref{sec:calibr-betw-bmlambda} and further discussion of SLOPE and the prox operator in Section~\ref{sec:prel-slope-amp}). Compared to the proximal gradient descent (ISTA) \cite{chambolle1998nonlinear,daubechies2004iterative,parikh2014proximal}, AMP has an extra correction term in its residual step that adjusts the iteration in a non-trivial way and seeks to provide improved convergence performance \cite{amp1}.

\begin{figure}
	\begin{minipage}{0.53\textwidth}
		\includegraphics[width=\textwidth]{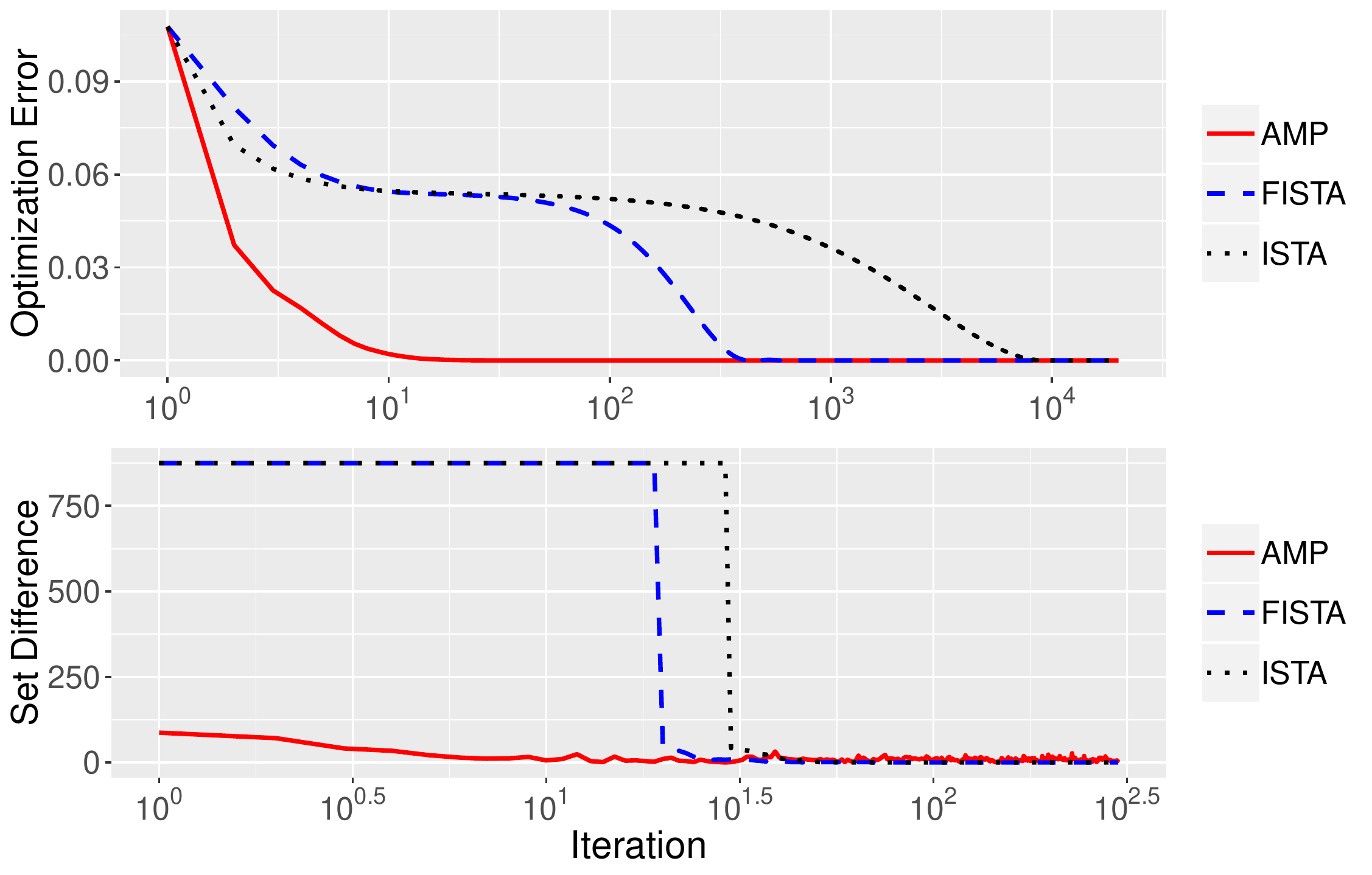}
		\captionof{figure}{Optimization errors, $||\bet^t-\widehat{\bet}||^2/\p$, and (symmetric) set difference of $\supp(\bet^t)$ and $\supp(\widehat{\bet})$.}
		\label{fig:AMPfaster}
	\end{minipage}
\hfill
	\begin{minipage}{0.45\textwidth}
			\resizebox{\linewidth}{!}{
			\begin{tabular}{lc|ccccc}
				\toprule 
				&&\multicolumn{5}{c}{Optimization errors}
				\\
				\midrule 
				& Set Diff & $10^{-2}$ &$10^{-3}$  & $10^{-4}$& $10^{-5}$& $10^{-6}$\\
				\midrule 
				ISTA  & 60 & 4048  & 7326  &   8569   & 9007 &9161\\
				FISTA & 47 & 275  & 374   &   412   & 593 &604 \\
				AMP   & 30 &  6   & 13     &  22    & 32  &40 \\
				\bottomrule 
			\end{tabular}
		}
	\captionof{table}{First iteration $t$ for which there is zero set difference or optimization error $||\bet^t-\widehat{\bet}||^2/\p$ falls below a threshold.}
	\label{tab:table1}
	\caption*{Figure \ref{fig:AMPfaster} and Table \ref{tab:table1} Details: Design $X$ is $500 \times 1000$ with i.i.d.\ $\mathcal{N}(0,1/500)$ entries. True signal $\bm\beta$ is i.i.d.\ Gaussian-Bernoulli: $\mathcal{N}(0, 1)$ with probability $0.1$ and 0 otherwise. Noise variance $\sigma^2_w=0$.  A careful calibration between the thresholds $\bm \theta_t$ in AMP and $\blam$ is SLOPE is used (details in Sec.\ \ref{sec:calibr-betw-bmlambda}).}
	\end{minipage}
\end{figure}

The \textit{empirical} performance of AMP in solving SLOPE under i.i.d.\ Gaussian matrix $\X$ is illustrated in Figure~\ref{fig:AMPfaster} and Table~\ref{tab:table1}, which suggest the superiority of AMP over ISTA and FISTA \cite{beck2009fast}---perhaps the two most popular proximal gradient descent methods---in terms of speed of convergence in this setting. However, the vast AMP literature thus far remains silent on whether AMP \textit{provably} solves SLOPE and, if so, whether one can leverage AMP to get insights into the statistical properties of SLOPE. This vacuum in the literature is due to the \textit{non-separability} of the SLOPE regularizer, making it a major challenge to apply AMP to SLOPE directly. In stark contrast, AMP theory has been rigorously applied to the LASSO \cite{lassorisk}, showing both good empirical performance and nice theoretical properties of solving the LASSO using
AMP. Moreover, AMP in this setting allows for asymptotically exact statistical characterization of its output, which converges to the LASSO solution, thereby providing a powerful tool in fine-grained analyses of the LASSO \cite{lassorisk2, lassopath, lassoamp,rush1}.

\textbf{Main contributions}. In this work, we prove that the AMP algorithm \eqref{eq:amp_slope} solves the SLOPE problem in an asymptotically \textit{exact} sense under independent Gaussian random designs. Our proof uses the recently extended AMP theory for non-separable denoisers \cite{nonseparable} and applies this tool to derive the state evolution that describes the asymptotically exact behaviors of the AMP iterates $\bm\beta^t$ in \eqref{eq:amp_slope}. The next step, which is the core of our proof, is to relate the AMP estimates to the SLOPE solution. This presents several challenges that \textit{cannot} be resolved only within the AMP framework. In particular, unlike the LASSO, the number of nonzeros in the SLOPE solution can exceed the number of observations. This fact imposes substantially more difficulties on showing that the distance between the
SLOPE solution and the AMP iterates goes to zero than in the LASSO case due to the possible \textit{non-strong convexity} of the SLOPE problem, even restricted to the solution support. To overcome these challenges, we develop novel techniques that are tailored to the characteristics of the SLOPE solution. For example, our proof relies on the crucial property of SLOPE that the \textit{unique} nonzero components of its solution never outnumber the observation units.

As a byproduct, our analysis gives rise to an \textit{exact} asymptotic characterization of the SLOPE solution under independent Gaussian random designs through leveraging the statistical aspect of the AMP theory. In more detail, the probability distribution of the SLOPE solution is completely specified by a few parameters that are the solution to a certain fixed-point equation in an asymptotic sense. This provides a powerful tool for fine-grained statistical analysis of SLOPE as it was for the LASSO problem. We note that a recent paper \cite{SLOPEasymptotic}---which takes an entirely different path---gives an asymptotic characterization of the SLOPE solution that matches our asymptotic analysis deduced from our AMP theory for SLOPE. However, our AMP-based approach is more
algorithmic in nature and offers a more concrete connection between the finite-sample behaviors of the SLOPE problem and its asymptotic distribution via the computationally efficient AMP algorithm.

\textbf{Paper outline}. In Section~\ref{sec:calibr-betw-bmlambda} we develop an AMP algorithm for finding the SLOPE estimator in \eqref{eq:SLOPE_est}.  Specifically, it is through the threshold values $\bm \theta_t$ in the AMP algorithm in \eqref{eq:amp_slope} that one can ensure the AMP estimates converge to the SLOPE estimator with parameter $\blam$, so in Section~\ref{sec:calibr-betw-bmlambda} we provide details for how one should calibrate the thresholds of the AMP iterations in \eqref{eq:amp_slope} in order for the algorithm to solve SLOPE cost in \eqref{eq:SLOPE_est}.  Then in Section \ref{sec:asympt-char-slope}, we state theoretical guarantees showing that the AMP algorithm solves the SLOPE optimization asymptotically and we leverage theoretical guarantees for the AMP algorithm to exactly characterize the mean square error (more generally, any pseudo-Lipschitz error) of the SLOPE estimator in the large system limit. 
This is done by applying recent theoretical results for AMP algorithms that use a non-separable non-linearity \cite{nonseparable}, like the one in \eqref{eq:amp_slope}.  Finally, Sections~\ref{sec_mainproof}-\ref{sec:conds} prove rigorously the theoretical results stated in Section~\ref{sec:asympt-char-slope} and we end with a discussion in Section~\ref{sec_discussion}.


\section{Algorithmic Development}
\label{sec:calibr-betw-bmlambda}

To begin with, we state assumptions under which our theoretical results will hold and give some preliminary ideas about SLOPE that will be useful in the development of the AMP algorithm.  

\textbf{Assumptions.} Concerning the linear model \eqref{eq:model} and parameter vector in \eqref{eq:SLOPE_est}, we assume:

\begin{itemize}
\item[\textbf{(A1)}] The measurement matrix $\X$ has independent and identically-distributed (i.i.d.) Gaussian entries that have mean $0$ and variance $1/{\n}$.
\item[\textbf{(A2)}]  The signal $\bet$ has elements that are i.i.d.\ $B$, with $\E(B^2\max\{0,\log B\}) < \infty$.
\item[\textbf{(A3)}] 
The noise $\bm w$ is elementwise i.i.d.\ $W$, with $\sigma_w^2:=\E(W^2)<\infty$.
\item[\textbf{(A4)}]\label{assumptionA4} The vector $\blam(\p) = (\lambda_1, \ldots, \lambda_p)$ is elementwise i.i.d.\ $\Lambda$, with $\E(\Lambda^2)<\infty$ and $\min\{\blam(\p)\} > 0$.

\item[\textbf{(A5)}] The ratio $n/\p$ approaches a constant $\delta \in (0,\infty)$ in the large system limit, as $\n, \p \rightarrow \infty$.
\end{itemize}
\textbf{Remark: (A4)} can be relaxed as $\lambda_1,\ldots,\lambda_p$ having an empirical distribution that converges weakly to probability measure $\Lambda$ on $\R$ with $\E(\Lambda^2)<\infty$ and $\|\blam(p)\|^2/p\to\E(\Lambda^2)$ and $\min\{\blam(\p)\} > 0$.  A similar relaxation can be made for the distributional assumptions \textbf{(A2)} and \textbf{(A3)}.

\textbf{SLOPE preliminaries.}
For a vector $\mathbf{v} \in \R^{\p}$, the divergence of the proximal operator, $\nabla\prox_{f}(\mathbf{v})$, is given by the following:
\begin{align}
\nabla  \prox_{f}(\mathbf{v})&:= \sum_{i=1}^{\p} \frac{\partial }{\partial v_i}  [\prox_{f}(\mathbf{v})]_i =  \Big(\frac{\partial}{\partial v_1},\frac{\partial}{\partial v_2}, \ldots, \frac{\partial}{\partial v_{\p}}\Big)\cdot \prox_{f}(\mathbf{v}),
\end{align}
where \cite[proof of Fact 3.4]{SLOPE2},
\begin{align}
\frac{\partial [ \prox_{J_{\blam}}(\mathbf{v})]_i }{\partial v_j}=
\begin{cases}
\frac{\text{sign}([ \prox_{J_{\blam}}(\mathbf{v})]_i ) \cdot\text{sign}([ \prox_{J_{\blam}}(\mathbf{v})]_j)}{\text{\#\{$1\leq k\leq \p: | [ \prox_{J_{\blam}}(\mathbf{v})]_k  |=| [ \prox_{J_{\blam}}(\mathbf{v})]_j  |$\}}}, &\text{if } | [\prox_{J_{\blam}}(\mathbf{v})]_j|=\lvert [ \prox_{J_{\blam}}(\mathbf{v})]_i \lvert, \\
0, &\text{otherwise}.
\end{cases}
\label{eq:dif_prox}
\end{align}
Hence the divergence takes the simplified form
\begin{align}
\nabla\prox_{J_{\blam}}(\mathbf{v})&=\|\prox_{J_{\blam}}(\mathbf{v})\|_0^*,
\label{eq:div_prox}
\end{align}
where $\|\cdot\|_0^*$ counts the unique non-zero magnitudes in a vector, e.g.\ $\|(0, 1, -2, 0, 2)\|_0^* = 2$. This explicit form of divergence not only waives the need to use approximation in calculation but also speed up the recursion, since it only depends on the proximal operator as a whole instead of on $\thet_{t-1},\X,\z^{t-1},\bet^{t-1}$. Therefore, we have
\begin{lemma}
In AMP, \eqref{eq:AMP1} is equivalent to
$
\z^{t+1}=\y-\X\bet^{t+1}+\frac{\z^{t}}{\delta \p}\|\bet^{t+1}\|_0^*.
$
\end{lemma}
\noindent
Other details and background on SLOPE and the prox operator are found in Section~\ref{sec:prel-slope-amp}.  Now we discuss the details of an AMP algorithm that can be used for finding the SLOPE estimator in \eqref{eq:SLOPE_est}.

\subsection{AMP Background}
An attractive feature of AMP is that its statistical properties can be exactly characterized at each iteration $t$, at least asymptotically, via a one-dimensional recursion known as state evolution \cite{amp2, nonseparable, rush1,javanmard2013state}.  Specifically, it can be shown that the pseudo-data, meaning the input $\X^\top \z^t+\bet^t$ for the estimate of the unknown signal in \eqref{eq:AMP0}, is asymptotically equal in distribution to the true signal plus independent, Gaussian noise, i.e.\ $\bet + \tau_t \mathbf{Z}$, where the noise variance $\tau_t$ is defined by the state evolution.  For this reason, the function used to update the estimate in \eqref{eq:AMP0}, in our case, the proximal operator, $\prox_{J_{\thet_t}}( \cdot)$, is usually referred to as a `denoiser' in the AMP literature.   

This statistical characterization of the pseudo-data was first rigorously shown to be true in the case of `separable' denoisers by Bayati and Montanari \cite{amp2}, and an analysis of the rate of this convergence was given in \cite{rush1}.  A `separable' denoiser is one that applies the same (possibly non-linear) function to each element of its input.  Recent work \cite{nonseparable} proves that the pseudo-data has distribution  $\bet + \tau_t \mathbf{Z}$ asymptotically, even when the `denoisers' used in the AMP algorithm are non-separable, like the SLOPE prox operator in \eqref{eq:AMP0}.

As mentioned previously, the dynamics of the AMP iterations are tracked by a recursive sequence referred to as the state evolution, defined as follows.    For $\mathbf{B}$ elementwise i.i.d.\ $B$ independent of 
$\mathbf{Z} \sim \mathcal{N}(0, \mathbb{I}_{\p})$, let $\tau_0^2 =  \sigma_w^2 +  \E[B^2]/\delta$ and for $t \geq 0$,
\be
\begin{split}
\tau_{t+1}^2 & = \sigma_w^2 + \lim_{\p} \frac{1}{\delta \p}  \E \norm{  \prox_{J_{\thet_t}}(\B + \tau_t \mathbf{Z}) - \B}^2.
\label{eq:SE2}
\end{split}
\ee
Below we make rigorous the way that the recursion in \eqref{eq:SE2} relates to the AMP iteration  \eqref{eq:amp_slope}.

We note that throughout, we let $ \mathcal{N}(\mu, \sigma^2)$ denote the Gaussian density with mean $\mu$ and variance $\sigma^2$ and we use $\mathbb{I}_{\p}$ to indicate a $\p \times \p$ identity matrix.

\subsection{Analysis of the AMP State Evolution}

As the state evolution \eqref{eq:SE2} predicts the performance of the AMP algorithm  \eqref{eq:amp_slope} (the pseudo-data, $\X^\top \z^t+\bet^t$, is asymptotically equal in distribution $\bet + \tau_t \mathbf{Z}$), it is of interest to study the large $t$ asymptotics of \eqref{eq:SE2}.  Moreover, recall that through the sequence of thresholds $\thet_t$, one can relate the AMP algorithm to the SLOPE estimator in \eqref{eq:SLOPE_est} for a specific $\blam$, and the explicit form of this calibration, given in Section~\ref{sec:calibration}, is motivated by such asymptotic analysis of the state evolution.

 It turns out that a finite-size approximation, which we denote $\tau_{t}^2(\p)$, will be easier to analyze than \eqref{eq:SE2}.  The definition of $\tau_{t+1}^2(\p)$ is stated explicitly in \eqref{eq:SE2_new} below.
 Throughout the work, we will define thresholds $\thet_t := \bfalph \tau_{t}(\p)$ for every iteration $t$ where the vector $\bfalph$ is fixed via a calibration made explicit in Section~\ref{sec:calibration}. 
We can interpret this to mean that within the AMP algorithm, $\bfalph$ plays the role of the regularizer parameter, $\blam$. 
Now we define $\tau_{t+1}^2(\p)$, for large $\p$, as a finite-sample approximation to \eqref{eq:SE2}, namely
\be
\begin{split}
	\tau_{t+1}^2(\p) & = \sigma_w^2 +  \frac{1}{\delta \p}  \E \norm{  \prox_{J_{\bfalph \tau_t(\p)}}(\bet + \tau_t(\p) \mathbf{Z}) - \bet}^2,
	\label{eq:SE2_new}
\end{split}
\ee
where the difference between \eqref{eq:SE2_new} and the state evolution \eqref{eq:SE2} is via the large system limit in $\p$.  When we refer to the recursion in \eqref{eq:SE2_new}, we will always specify the $\p$ dependence explicitly as $\tau_t(\p).$    An analysis of the limiting properties (in $t$) of \eqref{eq:SE2_new} is given in Theorem~\ref{thm:SE1} below, after which it is then argued that because interchanging limits and differentiation is justified, the large $t$ analysis of \eqref{eq:SE2_new} holds for \eqref{eq:SE2} as well.
Before presenting Theorem~\ref{thm:SE1}, however, we give the following result which motivates why the AMP iteration should relate at all to the SLOPE estimator. 
\begin{lemma}\label{lm:slope_stationary}
Any stationary point $\widehat \bet$ (with corresponding $\widehat \z$) in the AMP algorithm  \eqref{eq:AMP0}-\eqref{eq:AMP1} with $\thet_* = \bfalph \tau_*$ is a minimizer of the SLOPE cost function in \eqref{eq:SLOPE_est} with
\[\blam = \thet_*\Big(1 - \frac1{\delta \p} \left(\nabla\prox_{J_{\thet_*}}(\widehat\bet + \X^\top \widehat \z) \right) \Big)=\thet_*\Big(1 - \frac{1}{n} \left\|\prox_{J_{\thet_*}}(\widehat\bet + \X^\top \widehat \z) \right\|_0^* \Big).\]
\end{lemma}
\begin{proof}[Proof of Lemma~\ref{lm:slope_stationary}]
Denote, $\omega := ( \nabla \prox_{J_{\thet_*}} (\widehat\bet + \X^\top \widehat \z) )/({\delta \p})$.  Now, by stationarity,
\begin{align}
\widehat\bet &= \prox_{J_{\thet_*}}(\widehat\bet + \X^\top \widehat \z), \qquad \text{ and } \qquad \widehat \z = \y - \X\widehat\bet + \frac{\widehat \z}{\delta \p}(\nabla \prox_{J_{\thet_*}} (\widehat\bet + \X^\top \widehat \z) ). 
\label{eq:station2}
\end{align}
From \eqref{eq:station2}, notice that
$
\widehat \z = \frac{\y - \X \widehat\bet}{1 - \omega}.
$
By Fact \ref{fact:sub_prox}, $\X^\top \widehat \z \in \partial J_{\thet_*}(\widehat\bet)$, where $ \partial J_{\thet_*}(\widehat\bet)$ is the subgradient  of $J_{\thet_*}(\cdot)$ at $\widehat\bet$ (a precise definition of a subgradient is given in Section~\ref{sec:prel-slope-amp}).  Then, $\X^\top \widehat \z = \frac{\X^\top (\y - \X \widehat\bet)}{1- \omega} \in J_{\thet_*}(\widehat\bet),$ and therefore $\X^\top (\y - \X \widehat\bet) \in J_{\thet_*(1- \omega)}(\widehat\bet)$
which is \textit{exactly} the stationary condition of SLOPE with regularization parameter $\blam = (1-\omega)\thet_*$, as desired.
\end{proof}

Now we present Theorem~\ref{thm:SE1}, which provides results about the $t$ asymptotics of the recursion in \eqref{eq:SE2_new} and its proof is given in Appendix \ref{app_SE}.   First, some notation must be introduced: let $\bm A_{\min}(\delta)$ be the set of solutions to
\be
\label{eq:littlef_def}
\delta = f(\bfalph), \quad \text{where} \quad   f(\bfalph) := \frac{1}{\p} \sum_{i = 1}^{\p} \E \Big \{\Big(1-|[\prox_{J_{\bfalph}}(\mathbf{Z})]_i| \sum_{j \in I_i}\alpha_j \Big)\Big/{ [\bm D(\prox_{J_{\bfalph}}(\mathbf{Z}))]_i}\Big\}.
\ee %
Here $\odot$ represents elementwise multiplication of vectors and for vector $\bm v \in \R^{\p}$, $\bm D$ is defined elementwise as $[\bm D(\bm v)]_i=\#\{j : |v_j|=|v_i|\}$ if  $v_i\neq 0$ and $\infty$ otherwise. Let $I_i = \{j: 1 \leq j \leq \p \text{ and } |[\prox_{J_{\bfalph}}(\mathbf{Z})]_j| =  |[\prox_{J_{\bfalph}}(\mathbf{Z})]_i|\}$.  The expectation in \eqref{eq:littlef_def} is taken with respect to $\mathbf{Z},$ a $\p$-length vector of i.i.d.\ standard Gaussians.  Finally, for $\bm u\in\R^m,$ the notation $\langle\bm u\rangle:=\sum_{i=1}^m u_i/m$ and we say $\bm u$ is strictly larger than $\bm v \in\R^m$ if $u_i>v_i$ for all elements $i \in \{1, 2, \ldots, m\}$.  For the simple case of $p=2$, we illustrate an example of the set $\bm A_{\min}(\delta)$ in Figure~\ref{fig:A_min}.

\begin{theorem}\label{thm:SE1}
For any $\bfalph$ strictly larger than at least one element in the set $\bm A_{\min}(\delta)$, the recursion in \eqref{eq:SE2_new} has a unique fixed point that  we denote as $\tau^2_*(\p)$. Then $\tau_t(\p) \rightarrow \tau_*(\p)$ monotonically for any initial condition. Define a function $\F: \R \times \R^{\p} \rightarrow \R$ as
\be
\F\Big(\tau^2(\p), \bfalph \tau(\p)\Big) := \sigma_w^2 + \frac{1}{\delta \p}  \E \norm{  \prox_{J_{\bfalph \tau(\p)}}(\mathbf{B} + \tau(\p) \mathbf{Z}) - \mathbf{B}}^2,
\label{eq:SE_F}
\ee
where $\mathbf{B}$ is elementwise i.i.d.\ $B$ independent of $\mathbf{Z} \sim \mathcal{N}(0, \mathbb{I}_{\p})$, so that $\tau_{t+1}^2(\p) = \F(\tau^2_t(\p), \bfalph \tau_t(\p))$.  Then $\lvert  \frac{\partial \F}{\partial \tau^2(\p)} (\tau^2(\p), \bfalph \tau(\p)) \lvert < 1$ at $\tau(\p)= \tau_*(\p)$. Moreover, for $f(\bfalph)$ defined in \eqref{eq:littlef_def}, we show that  $f(\bfalph)= \delta \lim_{\tau(\p)\to\infty} {d \F}/{d\tau^2(\p)}$.
\end{theorem}
Beyond providing the large $t$ asymptotics of the state evolution sequence, notice that Theorem \ref{thm:SE1} gives necessary conditions on the calibration vector $\bfalph$ under which the recursion in \eqref{eq:SE2_new}, and equivalently, the calibration detailed in Section~\ref{sec:calibration} below are well-defined.

Recall that it is actually the state evolution in \eqref{eq:SE2} (and not that in \eqref{eq:SE2_new}) that predicts the performance of the AMP algorithm, and therefore we would really like a version of Theorem \ref{thm:SE1} studying the large system limit in $\p$.  We argue that because interchanging differentiation and the limit, the proof of Theorem \ref{thm:SE1} analyzing \eqref{eq:SE2_new}, can easily be used to give an analogous result 
for \eqref{eq:SE2}.  
In particular analyzing \eqref{eq:SE2} via the strategy given in the proof of Theorem \ref{thm:SE1} requires that we study the partial derivative of
$
\lim_p  \E \norm{  \prox_{J_{\bfalph \tau}}(\mathbf{B} + \tau \mathbf{Z}) - \mathbf{B}}^2/(\delta \p),
$
with respect to $ \tau^2 $. Indeed, to directly make use our proof for the finite-$\p$ case given in Theorem \ref{thm:SE1}, it is enough that
\be
\label{eq:interchange}
\frac{\partial}{\partial \tau^2}\lim_p  \E \norm{  \prox_{J_{\bfalph \tau}}(\mathbf{B} + \tau \mathbf{Z}) - \mathbf{B}}^2/(\delta \p) = \lim_p \frac{\partial}{\partial \tau^2} \E \norm{  \prox_{J_{\bfalph \tau}}(\mathbf{B} + \tau \mathbf{Z}) - \mathbf{B}}^2/(\delta \p).
\ee
Note that we already have an argument (based on dominated convergence for fixed $ p $, see \eqref{eq:total_deriv_F_v1} and Lemma \ref{lem:DC_res1}) showing that
$$
\frac{\partial}{\partial \tau^2} \E \norm{  \prox_{J_{\bfalph \tau}}(\mathbf{B} + \tau \mathbf{Z}) - \mathbf{B}}^2 =  \E \Big\{\frac{\partial}{\partial \tau^2}\norm{  \prox_{J_{\bfalph \tau}}(\mathbf{B} + \tau \mathbf{Z}) - \mathbf{B}}^2\Big\}.
$$
The next lemma gives us a roadmap for how to proceed (c.f., \cite[Theorem 7.17]{rudin1964principles}) to justify the interchange in \eqref{eq:interchange}.
\begin{lemma}
Suppose $ \{g_m\} $ is a sequence of functions that converge pointwise to $ g $ on a compact domain $ D $ and whose derivatives $ \{ g'_m \} $ converge uniformly to a function $ h $ on $ D $. Then $ h = g' $ on $ D $.
\end{lemma}
Therefore, taking $\{g_p\}=\{\F(\tau^2(p), \bfalph \tau(p))\}$, it suffices to show that if
\begin{align*}
\frac{\partial \F}{\partial \tau^2(p)} (\tau^2(p), \bfalph \tau(p)) & = \frac{\partial}{\partial \tau^2(p)}  \E \norm{  \prox_{J_{\bfalph \tau(p)}}(\mathbf{B} + \tau(p) \mathbf{Z}) - \mathbf{B}}^2/(\delta \p),
\end{align*}
then the sequence $ \{ \frac{\partial \F}{\partial \tau^2} (\tau^2, \bfalph \tau) \}_p $ converges uniformly as $ p \rightarrow \infty $.
The main tool for proving such a result is given in the following lemma.
\begin{lemma}
Suppose $ \{g_m\} $ is a sequence of $ L $-Lipschitz functions (where $ L $ is independent of $ m $) that converge pointwise to a function $ g $ on a compact domain $ D $. Then, the convergence is also uniform on $ D $.
\end{lemma}
Using this lemma, the essential idea is to show that there exists a constant $ L > 0 $, independent of $ p $, such that for all $ p $ and all $ \tau_1 $, $ \tau_2 $ in a bounded set $ D = \{ \tau : 0 < r \leq |\tau| \leq R \} $,
$$
\Big |\frac{\partial \F}{\partial \tau^2} (\tau^2_1, \bfalph \tau_1)-\frac{\partial \F}{\partial \tau^2} (\tau^2_2, \bfalph \tau_2) \Big| \leq L|\tau_1-\tau_2|.
$$
This follows  by the mean value theorem and \eqref{eq:final_deriv_new}, with $ L = \sup_{p, \tau \in D}|\frac{\partial}{\partial \tau^2}\frac{\partial \F}{\partial\tau^2}(\tau^2, \bfalph \tau)| < +\infty $.
\begin{rem}
	The boundedness of $\{\tau_t(p)\}$ is guaranteed by Proposition \ref{inverse_possible}. In particular, since $\bfalph$ satisfies the assumption of Theorem \ref{thm:SE1}, Proposition \ref{inverse_possible} guarantees $\blam$ is bounded and, consequently, so is $\tau$ (see the calibration in \eqref{eq:blam_alph_mapping} below).
\end{rem}

\subsection{Threshold Calibration} \label{sec:calibration}

Motivated by Lemma \ref{lm:slope_stationary} and the result of Theorem \ref{thm:SE1}, we define  a calibration from  the regularization parameter $\blam$, to the corresponding threshold $\bfalph$ used to define the AMP algorithm. In practice, we will be given finite-length $\blam$ and then we want to design the AMP iteration to solve the corresponding SLOPE cost.  We do this by choosing $\bfalph$ as the vector that solves $\blam=\blam(\bfalph)$ where
\begin{align}
\label{eq:blam_alph_mapping}
\blam(\bfalph)&:=\bfalph\tau_*(\p)\Big(1-\frac{1}{n}\E\| \prox_{J_{\bfalph \tau_*(\p)}}(\mathbf{B}+\tau_*(\p) \mathbf{Z})\|_0^*\Big),
\end{align}
where $\mathbf{B}$ is elementwise i.i.d.\ $B$ independent of $\mathbf{Z} \sim \mathcal{N}(0, \mathbb{I}_{\p})$ and $\tau_*(\p)$ is the limiting value defined in Theorem \ref{thm:SE1}.  We note the fact that the calibration in \eqref{eq:blam_alph_mapping} sets $\bfalph$ as a vector \emph{in the same direction} as $\blam$, but that is scaled by a constant value (for each $\p$), where the scaling constant value is 
$
\tau_*(\p)(1-\E\|\prox_{J_{\bfalph\tau_*(\p)}}(\mathbf{B}+\tau_*(\p) \mathbf{Z})\|_0^*/n).
$

In Proposition~\ref{inverse_possible} we show that the  calibration \eqref{eq:blam_alph_mapping} and its inverse $\blam\mapsto \bfalph(\blam)$ are well-defined
\begin{wrapfigure}[13]{r}{0.3\textwidth}
	\includegraphics[width=0.3\textwidth]{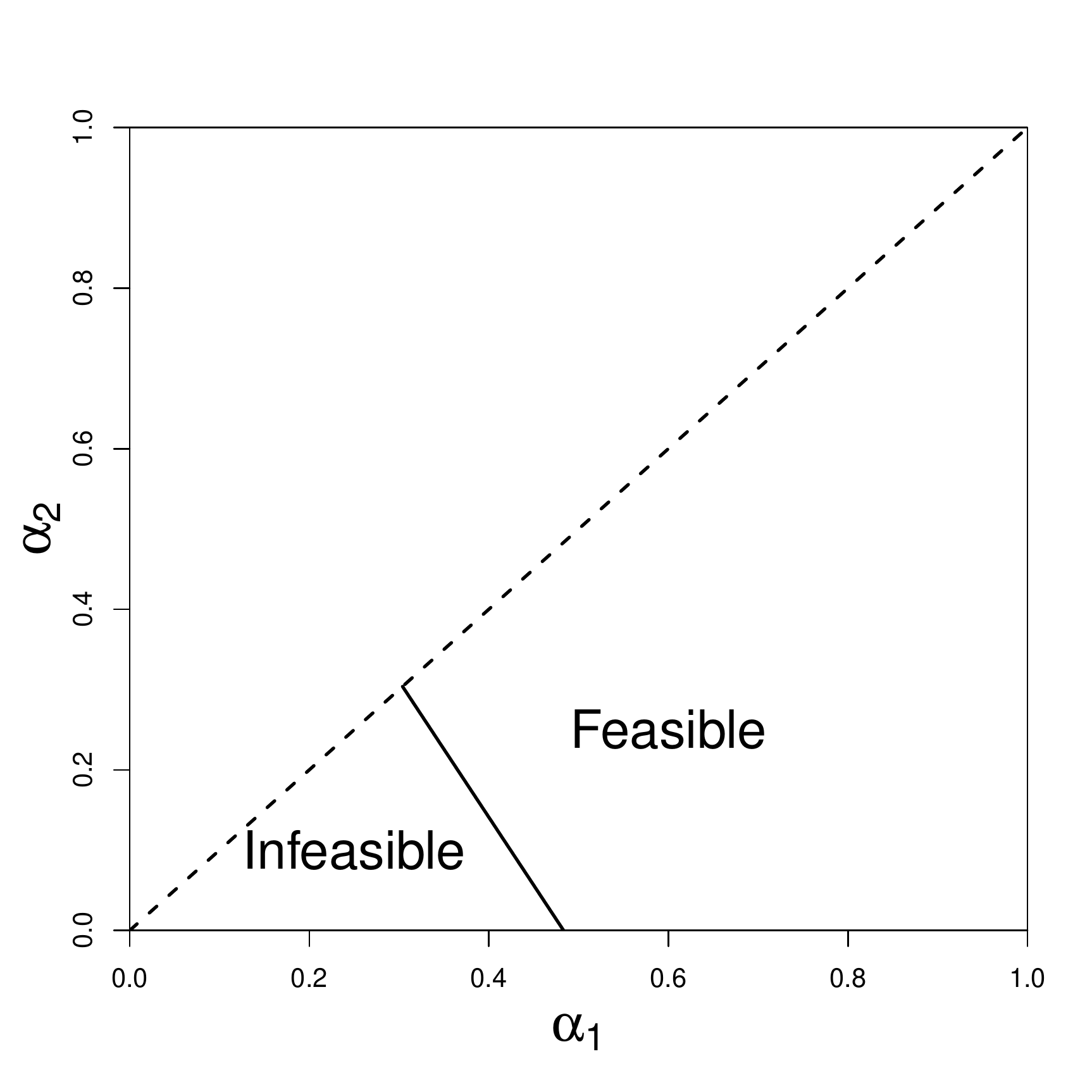}		
	\caption{$\bm A_{\min}$ (black curve) when $p=2$ and $\delta=0.6$.}
	\label{fig:A_min}
\end{wrapfigure}
and in Algorithm~\ref{alg:cali} we show that determining the calibration is straightforward in practice. 

\begin{proposition}
	\label{inverse_possible}
The function $\bfalph\mapsto \blam(\bfalph)$ defined in \eqref{eq:blam_alph_mapping} is continuous on $\{\bfalph : f(\bfalph)<\delta\}$ for $f(\cdot)$ defined in \eqref{eq:littlef_def}  with $\blam(\bm A_{\min})=-\infty$ and $\lim_{\bfalph\to\infty}\blam(\bfalph)=\infty$ (where the limit is taken elementwise). Therefore the function $\blam\mapsto \bfalph(\blam)$ satisfying \eqref{eq:blam_alph_mapping} exists. As $\p\goto\infty$, the function $\bfalph\mapsto \blam(\bfalph)$ becomes invertible (given $\blam$, $\bfalph$ satisfying \eqref{eq:blam_alph_mapping} exists uniquely). Furthermore, the inverse function is continuous non-decreasing for any $\blam>\bm 0$.
\end{proposition}
In \cite[Proposition 1.4 (first introduced in \cite{donoho2011noise}) and Corollary 1.7]{lassorisk} this is proven rigorously for the analogous LASSO calibration and in Appendix \ref{app_SE} we show how to adapt this proof to SLOPE case.
This proposition motivates Algorithm \ref{alg:cali} which uses a bisection method to find the unique $\bfalph$ for each $\blam$. It suffices to find two guesses of $\bfalph$ parallel to $\blam$ that, when mapped via \eqref{eq:blam_alph_mapping}, sandwich the true $\blam$.
\begin{algorithm} 
	\caption{Calibration from $\bm{\lambda}\to\bm{\alpha}$}
	\begin{algorithmic}
		\STATE{1. Initialize $\alpha_1=\alpha_{\min}$ such that $\alpha_{\min}\bm\ell\in\bm A_{\min}$, where $\bm\ell:=\bm{\lambda}/\lambda_1$; Initialize $\alpha_2=2\alpha_1$}
		\WHILE{ $L(\alpha_2)<0$ where $L:\R\to\R; \alpha\mapsto\sgn(\blam(\alpha\bm\ell)-\blam)$}
		\STATE{2. Set $\alpha_1=\alpha_2, \alpha_2=2\alpha_2$}
		\ENDWHILE
		\STATE{3. \textbf{return BISECTION} $(L(\alpha),\alpha_1,\alpha_2)$}
	\end{algorithmic}
	\label{alg:cali}
	\quad\\
Remark: $\sgn(\blam(\cdot)-\blam)\in\R$ is well-defined since $\blam(\cdot)\parallel\blam$ implies all entries share the same sign. The function ``\textbf{BISECTION}$(L,a,b)$'' finds the root of $L$ in $[a,b]$ via the bisection method.
\end{algorithm}

The calibration in \eqref{eq:blam_alph_mapping} is exact when $\p \rightarrow \infty$, so we study the mapping between $\bfalph$ and $\blam$ in this limit. Recall from \textbf{(A4)}, that the sequence of vectors $\{\blam(\p)\}_{\p \geq 0}$ are drawn i.i.d.\ from distribution $\Lambda$. 
It follows that the sequence $\{\bfalph(\p)\}_{\p \geq 0}$ defined for each $\p$ by the finite-sample calibration \eqref{eq:blam_alph_mapping} are i.i.d.\ from a distribution $A$, where $A$ satisfies $\E(A^2)<\infty$, and is defined via
\begin{align}
\label{eq:lambda_func}
\Lambda &=A \tau_* \Big( 1 - \lim_{\p}\frac{1}{\delta \p} \E || \prox_{J_{\bm A(\p)\tau_*}}
( \mathbf{B}+ \tau_* \mathbf{Z}) ||_0^* \Big),
\end{align}
We note, moreover, that the calibrations presented in this section are well-defined:

\begin{fact}\label{fact:limits}
	The limits in \eqref{eq:SE2}  and \eqref{eq:lambda_func} exist.
\end{fact}
This fact is proven in Appendix \ref{app_fact}. One idea used in the proof of Fact \ref{fact:limits} is that the prox operator is \emph{asymptotically} separable, a result shown by \cite[Proposition 1]{SLOPEasymptotic}.   Specifically, for sequences of input, $\{\bm v{(p)}\}$, and thresholds, $\{\bm \lambda{(p)}\}$, having empirical distributions that weakly converge to distributions $V$ and $\Lambda$, respectively, then there exists a limiting scalar function $h(\cdot):= h(\bm v(p);V,\Lambda)$ (determined by $V$ and $\Lambda$) of the proximal operator $\prox_{J_{\blam}}(\bm v(\p))$. Further details are given in Lemma \ref{lem:yue_sep} in Section~\ref{sec:asympt-char-slope}.  Using $h(\cdot) := h(\cdot; B + \tau_* Z, A \tau_*)$, this argument implies that \eqref{eq:SE2}  can be represented as
\ben
\tau_*^2 := \sigma_w^2 + \frac{1}{\delta}  \E (h(B+\tau_* Z)-B)^2,
\een
and if we denote $m$ as the Lebesgue measure, then the limit in \eqref{eq:lambda_func} can be represented as
\begin{align}
\PP\left(B+\tau_* Z\in \Big\{x \,\,\Big|  \,\, h(x)\neq 0 \quad  \text{and} \quad m\{z  \,\,|  \,\, |h(z)|=|h(x)|\}=0\Big\}\right).
\label{eqn:non flat}
\end{align}
In other words, the limit in \eqref{eq:lambda_func} is the Lebesgue measure of the domain of the quantile function of $h$ for which the quantile of $ h $ assumes unique values (i.e., is not flat).


\section{Asymptotic Characterization of SLOPE}
\label{sec:asympt-char-slope}

\subsection{AMP Recovers the SLOPE Estimate}
Here we show that the AMP algorithm converges in $\ell_2$ to the SLOPE estimator, implying that the AMP iterates can be used as a surrogate for the global optimum of the SLOPE cost function. The schema of the proof is similar to \cite[Lemma 3.1]{lassorisk}, however, major differences lie in the fact that the proximal operator used in the AMP updates \eqref{eq:AMP0}-\eqref{eq:AMP1} is non-separable.  
We sketch the proof here, and a forthcoming article will be devoted to giving a complete and detailed argument.
\begin{theorem}
\label{thm:main_result2}
Under assumptions  \textbf{(A1)} - \textbf{(A5)}, for the output of the AMP algorithm in \eqref{eq:AMP0} and the SLOPE estimate \eqref{eq:SLOPE_est},%
\be
\plim_{p \goto \infty} ~ \frac{1}{\p}\|\widehat{\bet} -  \bet^{t}\|^2 = c_t, \quad \text{where}\quad \lim_{t \rightarrow \infty} c_t = 0.
\label{eq:sim1}
\ee
\end{theorem}

The proof of Theorem  \ref{thm:main_result2} can be found in Section \ref{sec_mainproof}.  At a high level, the proof requires dealing carefully with the fact that the SLOPE cost function, $
	\mathcal{C}(\mathbf{b}) := \frac{1}{2}\| \y - \X \mathbf{b}\|^2+J_{\blam}(\mathbf{b}),
$ given in \eqref{eq:SLOPE_est} is \emph{not} necessarily strongly convex, meaning that we could encounter the undesirable situation where $\mathcal{C}(\widehat{\bet})$ is close to $\mathcal{C}(\bet)$ but $\widehat{\bet}$ is not close to $\bet$, meaning the statistical recovery of $\bet$ would be poor. 

In the LASSO case, one works around this challenge by showing that the (LASSO) cost function does have nice properties when considering just the elements of the non-zero support of $\bet^t$ at any (large) iteration $t$.  In the LASSO case, the non-zero support of $\bet$ has size no larger than $\n < \p$.  

In the SLOPE problem, however, it is possible that the support set has size exceeding $\n$, and therefore the LASSO analysis is not immediately applicable.  Our proof develops novel techniques that are tailored to the characteristics of the SLOPE solution.  Specifically, when considering the SLOPE problem, one can show nice properties (similar to those in the LASSO case) by considering a support-like set, that being the \emph{unique} non-zeros in the estimate $\bet^t$ at any (large) iteration $t$. In other words, if we define an equivalence relation $x\sim y$ when $|x|=|y|$, then entries of AMP estimate at any iteration $t$ are partitioned into equivalence classes. Then we observe from \eqref{eq:blam_alph_mapping}, and the non-negativity of $\blam$, that the number of equivalence classes is no larger than $\n$. We see an analogy between SLOPE's equivalence class (or `maximal atom' as described in Appendix \ref{sec:prel-slope-amp}) and LASSO's support set. This approach allows us to deal with the lack of a strongly convex cost.

Theorem \ref{thm:main_result2} ensures that the AMP algorithm solves the SLOPE problem in an asymptotic sense. To better appreciate the convergence guarantee, it calls for elaboration on \eqref{eq:sim1}. First, it implies that $\|\widehat{\bet} -  \bet^{t}\|^2/\p$ converges in probability to a constant, say $c_t$. Next, \eqref{eq:sim1} says $c_t \goto 0$ as $t \goto \infty$.

\subsection{Exact Asymptotic Characterization of the SLOPE Estimate}

A consequence of Theorem ~\ref{thm:main_result1}, is that the SLOPE estimator $\widehat{\bet}$ inherits performance guarantees provided by the AMP state evolution, in the sense of Theorem \ref{thm:main_result3} below.
Theorem \ref{thm:main_result3} provides as asymptotic characterization 
of pseudo-Lipschitz loss between $\widehat{\bet}$  and the truth $\bet$.   
\begin{definition}
\label{def:unifPLfunc}
{\textbf{\emph{Uniformly pseudo-Lipschitz functions}}}~\cite{nonseparable}: For $k\in \mathbb{N}_{>0}$, a function $\phi : \R^{d} \to \R$ is \emph{pseudo-Lipschitz of order $k$} if there exists a constant $L$, 
such that for $\mathbf{a}, \mathbf{b} \in \R^{d}$,
\begin{equation}\label{eq:pseudo-Lipschitz}
\|\phi(\bm a)-\phi(\bm b)\| \leq L\Big(1+ ({\|\bm a\|}/{\sqrt{d}})^{k-1}+ ({\|\bm b\|}/{\sqrt{d}})^{k-1}\Big) \Big({\|\bm a - \bm b\|}/{\sqrt{d}}\Big).
\end{equation}
A sequence (in $\p$) of pseudo-Lipschitz functions $\{\phi_{\p}\}_{\p\in \mathbb{N}_{>0}}$
is \emph{uniformly pseudo-Lipschitz}
of order $k$ if, denoting by $L_{\p}$ the pseudo-Lipschitz constant of $\phi_{\p}$, $L_{\p} < \infty$ for each $\p$ and $\lim\sup_{\p\to\infty}L_{\p} < \infty$.
\end{definition}

\begin{theorem}\label{thm:main_result3}
Under assumptions  \textbf{(A1)} - \textbf{(A5)}, for any uniformly pseudo-Lipschitz sequence of functions $\psi_{\p}: \R^{\p} \times \R^{\p} \rightarrow \R$ and for 
$\mathbf{Z}  \sim \mathcal{N}(0, \mathbb{I}_{\p})$, 
\begin{align*}
\plim_{\p} \psi_{\p}(  \widehat{ \bet},  \bet) =  \lim_t \plim_{\p} \E_{\mathbf{Z}}[ \psi_{\p}(  \prox_{J_{\bfalph(p) \tau_t}}(\bet + \tau_t \mathbf{Z}), \bet)],
\end{align*}
where $\tau_t$ is defined in \eqref{eq:SE2} and the expectation is taken with respect to $\mathbf{Z}.$
\end{theorem}

Theorem \ref{thm:main_result3} tells us that under uniformly pseudo-Lipschitz loss, in the large system limit, distributionally the SLOPE optimizer acts as a `denoised' version of the truth corrupted by additive Gaussian noise where the denoising function is given by the proximal operator, i.e.\ within uniformly pseudo-Lipschitz loss $\widehat{\beta}$ can be replaced with $\prox_{J_{\bfalph(p) \tau_t}}(\bet + \tau_t \mathbf{Z})$ for large $\p, t$.  

The proof of Theorem \ref{thm:main_result3} can be found in Section \ref{sec_mainproof}.  We show that Theorem \ref{thm:main_result3} follows from Theorem \ref{thm:main_result2} and  recent AMP theory dealing with the state evolution analysis in the case of non-separable denoisers \cite{nonseparable}, which can be used to demonstrate that the state evolution given in \eqref{eq:SE2} characterizes the performance of the SLOPE AMP \eqref{eq:amp_slope} via pseudo-Lipschitz loss functions.

We note that \cite[Theorem 1]{SLOPEasymptotic} follows by Theorem \ref{thm:main_result3} and their separability result \cite[Proposition 1]{SLOPEasymptotic}.  To see this, we use the following lemma that is a simple application of the Law of Large Numbers.
\begin{lemma}
\label{lem:beta_expectation}
For any function $f:\R^{\p} \rightarrow \R$ that is asymptotically separable, in the sense that there exists some function $\widetilde{f}: \R \rightarrow \R$, such that
\[\Big \lvert f(\bet) - \frac{1}{\p} \sum_{i=1}^n \widetilde{f}(\beta_i)\Big \lvert \rightarrow 0, \quad \text{ as } \quad p \rightarrow \infty,\]
where $ \widetilde{f}(B)$ is Lebesgue integrable then $\plim_p\Big(f(\bet) - \E_{\bm B} [\widetilde{f}(\bm B)]\Big) = 0,$ where $\bm B \sim $ i.i.d.\ $B$.
\end{lemma}
Now to show the result \cite[Theorem 1]{SLOPEasymptotic},
consider a special case of Theorem \ref{thm:main_result3} where $\psi_p(\bm x, \bm y) = \frac1p \sum \psi(x_i, y_i)$ for function $\psi: \R \times \R \rightarrow \R$ that is pseudo-Lipschitz of order $k=2$. It is easy to show that $\psi_p(\cdot, \cdot)$ is uniformly pseudo-Lipschitz of order $k=2$.  The result of Theorem \ref{thm:main_result3} then says that
\[\plim_{\p} \frac{1}{\p} \sum_{i=1}^{\p} \psi(  \widehat{ \beta}_i,  \beta_{i}) = \lim_t \plim_{\p}\frac{1}{\p} \sum_{i=1}^{\p}   \E_{\mathbf{Z}}[\psi(  [\prox_{J_{\bfalph(p) \tau_t}}(\bet + \tau_t \mathbf{Z})]_i , \beta_{i})].\]
Then \cite[Theorem 1]{SLOPEasymptotic} follows by \cite[Proposition 1]{SLOPEasymptotic}, restated below in Lemma \ref{lem:yue_sep}, the Law of Large Numbers, and Theorem \ref{thm:SE1}. 
Now we restate in Lemma \ref{lem:yue_sep}, the result given in \cite[Proposition 1]{SLOPEasymptotic}, which says that $\prox_{J_{\bfalph \tau_t}}( \cdot)$ becomes asymptotically separable as $p\to\infty$, for convenience.
	\begin{lemma}[Proposition 1, \cite{SLOPEasymptotic}]
		For an input sequence $\{\bm v{(p)}\}$, and a sequence of thresholds $\{\bm \lambda{(p)}\}$, both  having empirical distributions that weakly converge to distributions $V$ and $\Lambda$, respectively, then there exists a limiting scalar function $h$ (determined by $V$ and $\Lambda$) such that as $\p \rightarrow \infty,$
		\be
		\norm{\prox_{J_{\blam{(p)}}}(\bm v{(p)}) - h (\bm v{(p)}; V, \Lambda)}^2/\p \rightarrow 0,
		\label{eq:yue_limit}
		\ee
		where $h$ applies $h(\cdot; V, \Lambda)$ coordinate-wise to $\bm v{(\p)}$ (hence it is separable) and $h$ is Lipschitz(1).
		\label{lem:yue_sep}
\end{lemma}
Then \cite[Theorem 1]{SLOPEasymptotic} follows from Theorem \ref{thm:main_result3}  by using the asymptotic separability of the prox operator. 
Namely, 
the result of Lemma~\ref{lem:yue_sep} (using that $\bfalph(p) \tau_t$ has an empirical distribution that converges weakly to $A \tau_t$ for $A$ defined by \eqref{eq:lambda_func}), along with Cauchy-Schwarz and the fact that $\psi$ is pseudo-Lipschitz, allow us to apply a dominated convergence argument (see Lemma~\ref{lem:expectations}), from which it follows for some limiting scalar function $h^t$ as specified by Lemma~\ref{lem:yue_sep},
\ben
 \frac{1}{p} \Big \lvert\sum_{i=1}^{\p}\E_{\bm Z} [  \psi(  [\prox_{J_{\bfalph(p) \tau_t}}(\bet + \tau_t \mathbf{Z})]_i , \beta_{i}) ] -  \sum_{i=1}^{\p} \E_{\bm Z}[   \psi(  [h^t(\bet + \tau_t \mathbf{Z})]_i , \beta_{i})] \Big \lvert  \rightarrow 0.
 \een
Then the above allows us to apply Lemma~\ref{lem:beta_expectation} and the Law of Large Numbers to show
\ben
\begin{split}
\plim_{\p}\frac{1}{\p} \sum_{i=1}^{\p}   \E_{\mathbf{Z}}[\psi(  [\prox_{J_{\bfalph(p) \tau_t}}(\bet + \tau_t \mathbf{Z})]_i , \beta_{i})] &= \lim_{\p}\frac{1}{\p} \sum_{i=1}^{\p}   \E_{\mathbf{Z}, \bm B}[\psi( h^t([\bm B + \tau_t \mathbf{Z}]_i), B_{i})] \\
&=  \E_{Z, B}[\psi( h^t(B + \tau_t Z), B)],
%
\end{split}
\een
Finally we note that the result of \cite[Theorem 1]{SLOPEasymptotic} follows since
\ben
\begin{split}
\lim_t    \E_{Z, B}[\psi( h^t(B + \tau_t Z), B)] =   \E_{Z, B}[\psi( h^*(B + \tau_* Z), B)].
%
\end{split}
\een

We highlight that our Theorem \ref{thm:main_result3} allows the consideration of a non-asymptotic case in $t$. While Theorem \ref{thm:SE1} motivates an algorithmic way to find a value $\tau_t(\p)$ which approximates $\tau_*(\p)$ well, Theorem \ref{thm:main_result3} guarantees the accuracy of such approximation for use in practice.   One particular use of Theorem~\ref{thm:main_result3} is to design the optimal sequence $\blam$ that achieves the minimum $\tau_*$ and equivalently minimum error \cite{SLOPEasymptotic}, though a concrete algorithm for doing so is still under investigation.

Finally we show how we use Theorem~\ref{thm:main_result3} to study the asymptotic mean-square error between the SLOPE estimator and the truth \cite{celentano2019fundamental}.  

\begin{corollary}
	\label{corol1}
 Under assumptions $\textbf{(A1)} - \textbf{(A5)}$,
$\plim_{\p} \norm{ \widehat{ \bet}- \bet}^2/\p  = \delta(\tau_{*}^2 - \sigma_w^2).$
\end{corollary}

\begin{proof}
Applying Theorem~\ref{thm:main_result3} to the pseudo-Lipschitz loss function $\psi^1: \R^{\p} \times \R^{\p} \rightarrow \R$, defined as $\psi^1(\mathbf{x}, \mathbf{y}) = ||\mathbf{x}-\mathbf{y}||^2/\p$, we find
$\plim_{\p} \frac{1}{\p}\norm{ \widehat{ \bet}- \bet}^2  = \lim_t \plim_{\p} \frac{1}{\p} \E_{\mathbf{Z}}[\norm{  \prox_{J_{\bfalph \tau_t}}(\bet + \tau_t \mathbf{Z})- \bet}^2].$  
The desired result follows since $\lim_t \plim_{\p} \frac{1}{\p} \E_{\mathbf{Z}}[\norm{  \prox_{J_{\bfalph \tau_t}}(\bet + \tau_t \mathbf{Z})- \bet}^2] = \delta(\tau_{*}^2 - \sigma_w^2)$.  To see this, note that  $\lim_t \delta(\tau_{t+1}^2 - \sigma_w^2)  = \delta(\tau_{*}^2 - \sigma_w^2) $ and %
\[\plim_{\p} \frac{1}{\p} \E_{\mathbf{Z}}[\norm{  \prox_{J_{\bfalph \tau_t}}(\bet + \tau_t \mathbf{Z})- \bet}^2] = \lim_{\p}  \frac{1}{\p}\E_{\mathbf{Z}, \bm B}[  \norm{  \prox_{J_{\bfalph \tau_t}}(\mathbf{B} + \tau_t \mathbf{Z})- \mathbf{B}}^2] = \delta(\tau_{t+1}^2 - \sigma_w^2),\]
for $\mathbf{B}$ elementwise i.i.d.\ $B$ independent of  $\mathbf{Z}  \sim \mathcal{N}(0, \mathbb{I}_{\p})$.  A rigorous argument for the above requires showing that the assumptions of Lemma~\ref{lem:beta_expectation} are satisfied and follows similarly to that used to prove property \textbf{(P2)} stated in Section \ref{sec_mainproof} and proved in Appendix~\ref{app_props}.
\end{proof}


\section{Proof for Asymptotic Characterization of the SLOPE Estimate}\label{sec_mainproof}

In this section we prove Theorem \ref{thm:main_result3}. To do this, we use a result guaranteeing that the state evolution given in \eqref{eq:SE2} characterizes the performance of the SLOPE AMP algorithm \eqref{eq:AMP1},  given in Lemma \ref{thm:main_result1} below.
Specifically, Lemma \ref{thm:main_result1} relates the state evolution \eqref{eq:SE2} to the output of the AMP iteration \eqref{eq:AMP1} for pseudo-Lipschitz loss functions. This result follows from \cite[Theorem 14]{nonseparable}, which is a general result relating state evolutions to AMP algorithm with non-separable denoisers. In order to apply \cite[Theorem 14]{nonseparable},  we need to demonstrate that our denoiser, i.e.\ the proximal operator $\prox_{J_{\bfalph \tau_t}}(\cdot)$ defined in \eqref{eq:prox}, satisfies two additional properties labeled \textbf{(P1)} and \textbf{(P2)} below.

Define a sequence of denoisers $\{\eta_{\p}^t\}_{\p \in  \mathbb{N}_{>0}}$ where $\eta_{\p}^t: \R^{\p} \rightarrow \R^{\p}$ to be those that apply the proximal operator $\prox_{J_{\bfalph \tau_t}}(\cdot)$ defined in \eqref{eq:prox}, i.e.\ for a vector $\mathbf{v} \in \R^{\p}$, define 
\be
\eta_{\p}^t(\mathbf{v}) := \prox_{J_{\bfalph \tau_t}}( \mathbf{v}).
\label{eq:eta_def}
\ee
\begin{itemize}
\item[\textbf{(P1)}] For each $t$, denoisers $\eta_{\p}^t(\cdot)$ defined in \eqref{eq:eta_def} are uniformly Lipschitz (i.e.\ uniformly pseudo-Lipschitz of order $k=1$) per Definition \ref{def:unifPLfunc}.
\item[\textbf{(P2)}] For any $s, t$ with $(\bm Z, \bm Z')$ a pair of length-$\p$ vectors, where  for $i \in \{1, 2, \ldots, \p\}$, the pair $(Z_i, Z'_i)$ i.i.d.\ $\sim \mathcal{N}(0, \mathbf{\Sigma})$ with $\mathbf{\Sigma}$ any $2 \times 2$ covariance matrix, the following limits exist and are finite.
\begin{align*}
&\plim_{\p\to\infty}\frac{1}{\p} \norm{\bet},\quad \plim_{\p\to\infty}\frac{1}{\p} \E_{\bm Z}[\bet^\top \eta_{\p}^t(\bet + \bm Z)],\quad \text{ and } \quad  \plim_{\p\to\infty}\frac{1}{\p} \E_{\bm Z, \bm Z'}[\eta_{\p}^s(\bet + \bm Z')^\top \eta_{\p}^t(\bet + \bm Z)].
\end{align*}
\end{itemize}

We will show that properties \textbf{(P1)} and \textbf{(P2)} are satisfied for our problem in Appendix \ref{app_props}. 

\begin{lemma}{\cite[Theorem 14]{nonseparable}}\label{thm:main_result1}
Under assumptions  \textbf{(A1)} - \textbf{(A4)}, given that \textbf{(P1)} and \textbf{(P2)} are satisfied, for the AMP algorithm in \eqref{eq:AMP1} and for any uniformly pseudo-Lipschitz sequence of functions $\phi_{\n}: \R^{\n} \times \R^{\n} \rightarrow \R$ and $\psi_{\p}: \R^{\p} \times \R^{\p} \rightarrow \R$, let $\bm Z  \sim \mathcal{N}(0, \mathbb{I}_{\n})$ and $\bm Z'  \sim \mathcal{N}(0, \mathbb{I}_{\p})$, then
\begin{align*}
\plim_{\n} \Big(\phi_{\n}(\z^t, \w) - \E_{\bm Z}[ \phi_{\n}( \w + \sqrt{\tau_t^2 - \sigma_{w}^2} \bm Z, \w)] \Big) &= 0,\\
\plim_{\p} \Big(\psi_{\p} (  \bet^{t} + \X^\top \z^t, \bet) - \E_{\bm Z'}[ \psi_{\p}(  \bet + \tau_t \bm Z', \bet)]  \Big) &= 0,
\end{align*}
where $\tau_t$ is defined in \eqref{eq:SE2}.
\end{lemma}

We now show that Theorem \ref{thm:main_result3} follows from Lemma \ref{thm:main_result1} and Theorem \ref{thm:main_result2}. 

\begin{proof}[Proof of Theorem \ref{thm:main_result3}]
First, for any fixed $n$ and $t$, the following bound uses that $\psi_n$ is uniformly pseudo-Lipschitz of order $k$ and the Triangle Inequality,
\begin{align*}
\Big \lvert  \psi_{\p}(  \bet^{t}, \bet)- \psi_{\p}(  \widehat{\bet}, \bet)\Big \lvert & \leq L \Big(1 + \Big(\frac{\norm{ (\bet^{t}, \bet)}}{\sqrt{2\p}}\Big)^{k-1} +  \Big(\frac{\norm{(  \widehat{\bet}, \bet)}}{\sqrt{2\p}}\Big)^{k-1}\Big)  \frac{1}{\sqrt{2\p}}\norm{\bet^{t} - \widehat{\bet}} \\
& \leq L \Big(1 + \Big(\frac{\norm{ \bet^{t}}}{\sqrt{2\p}}\Big)^{k-1} +  \Big(\frac{\norm{  \widehat{\bet}}}{\sqrt{2\p}}\Big)^{k-1} + \Big(\frac{\norm{ \bet}}{\sqrt{2\p}}\Big)^{k-1}\Big)  \frac{1}{\sqrt{2\p}}\norm{\bet^{t} - \widehat{\bet}}.
\end{align*}
Now we take limits on either side of the above, first with respect to ${\p}$ and then with respect to $t$.  We note that the term $\frac{1}{\sqrt{n}}\norm{\bet^{t} - \widehat{\bet}}$ vanishes by Theorem \ref{thm:main_result2}.  Then as long as 
\be
\lim_t \plim_{\p} \Big({\norm{ \bet^{t}}}/{\sqrt{{\p}}}\Big)^{k-1}, \qquad  \plim_{\p} \Big({\norm{  \widehat{\bet}}}/{\sqrt{{\p}}}\Big)^{k-1}, \qquad \text{ and } \qquad \plim_{\p} \Big({\norm{ \bet}}/{\sqrt{{\p}}}\Big)^{k-1},
\label{eq:bounded_limits}
\ee
are all finite, we have
$\plim_{\p} \psi_{\p}(  \widehat{\bet}, \bet) = \lim_t \plim_{\p}  \psi_{\p}(  \bet^{t}, \bet).$
But by Theorem \ref{thm:main_result1} we also know that 
$$ \lim_t \plim_{\p}  \psi_{\p}(  \bet^{t}, \bet) =  \lim_t \plim_{\p} \E[ \psi_{\p}(  \eta^{t}( \bet + \tau_t \bm Z),  \bet)],$$
giving the desired result.

Finally we convince ourself that the limits in \eqref{eq:bounded_limits} are finite.  Since $k$ finite, that the third term in \eqref{eq:bounded_limits} is finite follows by property \textbf{(P2)}.  Bounds for the first and second term are demonstrated in Lemma \ref{lem:bounded_vals} found in Appendix \ref{app_AMP}.

\end{proof}


\section{Proof AMP Finds the SLOPE Solutions}\label{sec_lemmaproof}
In this section we aim to prove Theorem \ref{thm:main_result2}.  
Define the SLOPE cost function as follows,
\be
\label{eq:SLOPEcost}
\mathcal{C}(\mathbf{b}) := \frac{1}{2}\| \y - \X \mathbf{b}\|^2+J_{\blam}(\mathbf{b}),
\ee
where $J_{\blam}(\mathbf{b})$ is the sorted $\ell_1$-norm.
The proof of Theorem \ref{thm:main_result2} relies on a technical lemma, Lemma \ref{lem:opt_lem}, stated in Section~\ref{sec:main_tech} below, that deals carefully with the fact that the SLOPE cost function given in \eqref{eq:SLOPEcost} is \emph{not} necessarily strongly convex. 

In the LASSO case, one works around this challenge by showing that the (LASSO) cost function does have nice properties when considering just the elements of the non-zero support of $\bet^t$ at any (large) iteration $t$, using that the non-zero support of $\bet$ has size no larger than $\n < \p$.

In the SLOPE problem, however, it is possible that the support set has size exceeding $\n$, and therefore the LASSO analysis is not immediately applicable.  Our proof develops novel techniques that are tailored to the characteristics of the SLOPE solution.  Specifically, when considering the SLOPE problem, one can show nice properties (similar to those in the LASSO case) by considering a support-like set, that being the \emph{unique} non-zeros in the estimate $\bet^t$ at any (large) iteration $t$. 

In other words, our strategy is to define an equivalence relation $x\sim y$ when $|x|=|y|$ and partition the entries of the AMP estimate at any iteration $t$ into equivalence classes. This allows us to observe, using \eqref{eq:blam_alph_mapping} and the non-negativity of $\blam$, that the number of equivalence classes is no larger than $\n$. (Recall that $\|\cdot\|_0^*$ counts the unique non-zero magnitudes in a vector.) We see an analogy between SLOPE's equivalence class (or `maximal atom' as described  in Section \ref{sec:prel-slope-amp}) and LASSO's support set. This approach, taken in Lemma   \ref{lem:opt_lem} below, allows us to deal with the fact that we are not guaranteed to have a strongly convex cost.  Then Lemma   \ref{lem:opt_lem}  is used to prove Theorem \ref{thm:main_result3}.

Before we state Lemma   \ref{lem:opt_lem}, we include some useful preliminary information on SLOPE that will be needed for the upcoming work.  In particular, we introduce in more details the idea of equivalence classes of elements having the same magnitude, a mapping of vector ranking denoted as $\hat{\Pi}$, and a polytope-related mapping whose image is the set of subgradients denoted as $\mathcal{P}$. These definitions are all given in more detail in Section \ref{sec:prel-slope-amp}.

\subsection{Preliminaries on SLOPE}
\label{sec:prel-slope-amp}

In general, we refer to the function $ \mathcal{C}(\cdot)$ stated in \eqref{eq:SLOPEcost} as the SLOPE cost function and the SLOPE estimator $\hat{\bet}$ is the one that minimizes the SLOPE cost.  We note that the SLOPE cost function  $ \mathcal{C}(\cdot)$ depends on both $\y$ and $\blam$, so technically a notation like $\mathcal{C}_{(\y, \blam)}(\cdot)$ would be more rigorous, however, we don't think that dropping the explicit dependence on $(\y, \blam)$ will cause any confusion.

For a convex function $f: \R^{\p} \rightarrow \R$, we denote the subgradient of $f$ at a point $\mathbf{x} \in \R^{\p}$ as $\partial f(\mathbf{x})$.    We will be interested, particularly, in the subgradient of the SLOPE cost $\partial \mathcal{C}( \mathbf{b})$ which forces us to study the subgradient of the SLOPE norm $\partial J_{\blam}(\mathbf{b})$.  In particular,
\begin{fact}
\label{Csub}
$\partial \mathcal{C}( \mathbf{b}) = - \mathbf{X}^\top(\mathbf{y} - \mathbf{X}  \mathbf{b})  + \partial J_{\boldsymbol \lambda}(\mathbf{b}).$
\end{fact}
We will now describe explicitly the relevant subgradient, $\partial J_{\blam}\subset\mathbb{R}^p$.  We note that the proximal operator given in \eqref{eq:prox} is linked to the subgradient of the SLOPE norm in the following way.
\begin{fact}\label{fact:sub_prox}
If $\prox_{J_{\blam}}(\mathbf{v}_1) = \mathbf{v}_2$, then $\mathbf{v}_1 - \mathbf{v}_2 \in \partial J_{\blam}(\mathbf{v}_2).$
\end{fact}

Define a function $\Pi_{\x}:\mathbb{R}^{\p}\goto\mathbb{R}^{\p}$ to be a mapping (not necessarily unique) that sorts its input by magnitude in descending order according to absolute values of entries in $\x$. For example, if $\x=(5,2,-3,-5)$, then there are two possible such mappings $\Pi_{\x}(\bm b)=(|b_1|,|b_4|,|b_3|,|b_2|)$ or $\Pi_{\x}(\bm b) = (|b_4|,|b_1|,|b_3|,|b_2|)$. Using this notation, we can rewrite the SLOPE norm as $J_{\blam}(\bm b)=\blam\cdot\Pi_{\bm b}(\bm b)$. Since such mapping may not be unique, the inverse may not exist and we therefore define a pseudo-inverse mapping, $\hat{\Pi}^{-1}_{\x}$, that is based on the function $\hat{\Pi}_{\x}: \mathbb{R}^{\p}\to \{\text{maximal atoms}\}$. In words, $\hat{\Pi}_{\x}$ finds the maximal atoms of ranking of the absolute values of $\x$. Then $\hat{\Pi}_{\x}$ corresponds to the mapping
$$
\begin{pmatrix}
1&2&3&4	
\\
\{1,2\}&4&3&\{1,2\}
\end{pmatrix}
$$
with $\hat{\Pi}_{\x}(\x)=(\{5,-5\},\{5,-5\},-3,2)$ and $\hat{\Pi}^{-1}_{\x}(\blam)=(\{\lambda_1,\lambda_2\},\lambda_4,\lambda_3,\{\lambda_1,\lambda_2\})$. Then it is not hard to see that there exists $\hat{\blam}\in\hat{\Pi}^{-1}_x(\blam)$ such that 
$J_{\blam}(\bm b)=\blam\cdot\Pi_{\bm b}(\bm b)=\hat{\blam}\cdot |\bm b|.$
In words, this says there are two equivalent ways to consider the calculation of $J_{\blam}(\bm b)$ when $\lambda_1\geq \ldots \geq \lambda_p\geq 0$.  First $\blam\cdot\Pi_{\bm b}(\bm b)$ computes the inner product between $\blam$ and the \emph{sorted} magnitudes of $\bm b$, and in the second case, $\hat{\blam}^{\top} |\bm b|$ computes the inner product between the magnitudes of $\bm b$ (unsorted), with a rearrangement of the $\blam$ vector (based on $\bm b$) that pairs the  values in $\blam$ with the values of $|\bm b|$ by magnitude.

Now we define an equivalence relation $x\sim y$ if $|x|=|y|$. Then $\hat{\Pi}_{\x}$ partitions elements in $\mathbf{x}$ into different equivalence classes $I$. The motivation of using equivalence classes roots from AMP. In calibrating the AMP to the SLOPE problem, we need to calculate $\nabla\prox$, which equals the number of non-zero equivalence classes.  For example, $\frac{\partial\prox}{\partial \bm v}|_{\bm v=(1,0,-1,3)}=(\frac{1}{2},0,\frac{1}{2},1)$ has a sum of $2$.

Now we note that the subgradient of the SLOPE norm can be represented using the idea of the equivalence classes.  For a vector $\bm v \in \R^{\p}$, we use the notation $\bm v_I$ to be the elements of the vector $\bm v$ belonging to equivalence class $I$.  Then,
\begin{fact}
\label{fact:subgradient}
\ben
\partial J_{\blam}(\mathbf{s})= \left\{\mathbf{v} \in \R^{\p}: \textup{ for each equivalent class $I$,}
\begin{cases}
	\text{if $\bm s_I\neq 0$}\implies \bm v_I \in \mathcal{P}([\, \hat{\Pi}^{-1}_{\bm s}(\mathbf{\lambda}) ]_I\,) \,\sgn(\bm s_I);
	\\
	\text{if $\bm s_I= 0$}\implies |\bm v_I| \in \mathcal{P}_0([\, \hat{\Pi}^{-1}_{\bm s}(\mathbf{\lambda})\,]_I)
\end{cases}
\right\}.
\een
\end{fact}
In the above, $\mathcal{P}, \mathcal{P}_0$ are polytope-related mappings,
\begin{align*}
\mathcal{P}(\bm u)&:=
\left\{
\y: \y=\bm A\bm u \text{ for some doubly stochastic matrix $\bm A$} 
\right\}
\\
\mathcal{P}_0(\bm u)&:=
\left\{
\y: \y=\bm A\bm u  \text{ for some doubly sub-stochastic matrix $\bm A$} 
\right\}
\end{align*}
By definition, the doubly stochastic matrix, a.k.a.\ a Birkhoff polytope, is a square matrix of non-negative real numbers, whose row and column sums equal $1$. 
 For example,
\be
\bm A = \begin{pmatrix}
1/3&2/3&0	\\
1/6&1/3&1/2	\\
1/2&0&1/2	
\end{pmatrix}
\label{eq:A_examp}
\ee
is a doubly stochastic matrix.  Similarly, a doubly sub-stochastic matrix is defined as a square matrix of non-negative real numbers, whose row and column sums are at most $1$.  Note that if all entries of $\blam$ take the same value, the subgradient in Fact~\ref{fact:subgradient} gives the usual subgradient of the $\ell_1$ norm.

Using the subgradient definition in Fact~\ref{fact:subgradient}, consider $\mathcal{P}((\lambda_1, \lambda_2, \lambda_3))$, relating to a non-zero equivalence class having three entries.  Then $\bm A$ in \eqref{eq:A_examp} is one possible matrix considered in defining the set $\mathcal{P}((\lambda_1, \lambda_2, \lambda_3))$ and it has the following interpretation. The rows of $\bm A$ determine how the subgradient $\bm v_I$ values are calculated by averaging the corresponding threshold values $\blam$, for example, the first entry of $\bm v_I$ is a weighted average with $1/3$ its weight in $\lambda_1$ and $2/3$ in $\lambda_2$; the second entry of $\bm v_I$ is a weighted average with $1/6$ its weight in $\lambda_1$, $1/3$ in $\lambda_2$, and $1/2$ in $\lambda_2$, etc.  You can think of this as determining the threshold each input value $\bm s_I$ receives, as some weighted combination of all the possible threshold values $\blam$ corresponding to this equivalence class. Similarly, the columns of the doubly-stochastic matrix considered in the mapping $\mathcal{P}$ define how the thresholds $\blam$ are spread out amongst each element of the subgradient, for example, $1/3$ of $\lambda_1$'s value goes to the first element of $\bm v_I$, $1/6$ to the second value, and $1/2$ to the third value, etc.

To see why $\partial J_{\blam}(\bm s)$ takes the form given in Fact~\ref{fact:subgradient}, let's consider again the $\mathcal{P}$ used in the case that $\bm s_I\neq 0$. Recall the $\bm s_I$ looks at only the indices of $\bm s$ appearing in the equivalence class $I$, so all elements of $\bm s_{I}$ have the same absolute value.  This means that there are many ways to share the corresponding $\blam$ threshold values among them. We can think of this as an assignment problem: assign jobs (thresholds $\blam$) to workers ($s_i$) where as assignment according to a doubly stochastic matrix is a natural one (all workers take on the same load, and all jobs must be completed).  On the other hand, $\mathcal{P}_0$ does not require that the sharing of the threshold values $\blam$ amongst the entries of $\bm s_I$ be strict: row and/or column sums can be smaller than one. This difference is rooted in the subgradient of $\ell_1$ norm: i.e.\ 
$\partial|x|=\sgn(x)$ when $x\neq 0$ and $\partial|x| \in [-1,1]$ when $x= 0$.

For a rigorous proof of Fact~\ref{fact:subgradient}, we refer the reader to \cite[Exercise 8.31]{rockafellar2009variational}, but we give a quick sketch here in the case of $\bm s_I \neq 0$. The proof uses that $\mathcal{P}(\bm u)$ is a permutohedron, meaning a convex hull with vertices corresponding to permuted entries of $\bm u$.  Notice that we can rewrite $J_{\blam}(\bm s)$ as a finite max function $J_{\blam}(\bm s): \max\{\blam^{\top}f_1(\bm s),...,\blam^{\top}f_m(\bm s)\}$, where $\{f_i(\bm s)\}_{1 \leq i \leq m}$ is the collection of all possible permutations for the entries of $|\bm s|$.  Notice that the permutation that sorts the magnitudes will be chosen by the maximum function.  For such a function (see \cite[Exercise 8.31]{rockafellar2009variational}) the subgradient takes the form of a convex hull of the partial derivatives of the maximizing elements:
\begin{align}
\partial J_{\blam}(\bm s) \in \text{conv}\{ \nabla_{\bm s} (\blam^\top f_i(\bm s)): i\in A(\bm s)\}\equiv \text{conv}\{ f_i^{-1}(\blam): i\in A(\bm s)\},
\end{align}
where $A(\bm s)=\{i\in\{1,2,\ldots,m\}: \blam^{\top}f_i(\bm s)=J_{\blam}(\bm s)\}$ and in our case, the partial derivatives correspond to permutations of the thresholds.
Now, without loss of generality, let's consider an input that has only one non-zero equivalence class, i.e.\ $\bm s=(s,s,...,s)\in \mathbb{R}^d$. Then clearly there are $m=d!$ possible permutations. 
Therefore, 
$$\partial J_{\blam}(\bm s) \in \text{conv}\{f_i^{-1}(\blam): i\in \{1,2,...,d!\}\}\equiv\text{conv}\{f_i(\blam): i\in \{1,2,...,d!\}\}.$$
In other words, the partial derivative lies in the set that is the convex combination of all possible permutations of the threshold $\blam$. By definition, this is a permutohedron. So, in our case, the subgradient is a convex hull whose vertices are the permutated thresholds, i.e.\ an image of Birkhoff polytope under the thresholds, which can be characterized by doubly stochastic matrices.

\subsection{Main Technical Lemma} \label{sec:main_tech}

Now we state and prove the main technical lemma  that will be used to prove Theorem \ref{thm:main_result2}.  
Before we state Lemma \ref{lem:opt_lem}, let us introduce a very important definition:
\begin{definition}
	Given a vector $\bm v \in \R^p$, a set $I \subset \{1, \ldots, p\}$ is said to be a maximal atom of indices of $\bm v$ if 
	$
	|v_i| = |v_j|
	$
	for all $i, j \in I$ and 
	$
	|v_i| \ne |v_k|
	$
	for $i \in I$ and all $k \notin I$. With this definition in place, we define the star support of the vector $\bm v$ as
	\[
	\supps(\bm v) := \{I: I \subset \{1, \ldots, p\} \text{ is a maximal atom of indices of } \bm v \text{ and } \bm v_I\neq 0\}.
	\]
\end{definition}

For example, if $\bm v = (1, 1, -1, 0, 2, -1)$, then
$
\supps(\bm v) = \left\{ \{1, 2, 3, 6\}, \{5\} \right\}.
$
Now we state and prove Lemma \ref{lem:opt_lem}.

\begin{lemma} \label{lem:opt_lem}
	For constants $c_1,..., c_5 > 0$, if the following conditions are satisfied,
	\begin{itemize}
		\item[(1)] $\frac{1}{\sqrt{\p}}\|\boldsymbol  \beta^{t} -\hat{\boldsymbol  \beta}\| \leq c_1,$
		\item[(2)] There exists a subgradient $sg(\mathcal{C}, \boldsymbol  \beta^{t}) \in \partial \mathcal{C}(\boldsymbol  \beta^{t})$ such that $\frac{1}{\sqrt{\p}}\|sg(\mathcal{C}, \boldsymbol  \beta^{t})\| \leq \epsilon,$
		\item[(3)] Let $\boldsymbol \nu^t := \mathbf{X}^\top(\mathbf{y} - \mathbf{X} \bet^{t}) + sg(\mathcal{C}, \boldsymbol  \beta^{t}) \in \partial J_{\blam}(\bet^{t}) $ (where $ sg(\mathcal{C},\bet^{t})$ is the subgradient from Condition (2)). Denote $s_t(c_2) := \{I \subset [\p]: |\bm \nu^t_I| \succeq  [\mathcal{P}(\hat\Pi^{-1}_{\bet^{t}}(\blam))]_I(1 - c_2)\}$ and $S_t(c_2) := \{i \in I: I\in s(c_2)\}$, where the equivalence classes, $I$, for both sets are defined via the AMP estimation $\bet^t$, and for a vector $\x\in \mathbb{R}^d$ and a set $\mathbf{A}\subset \mathbb{R}^d$, the notation $\x\succeq \mathbf{A}$ means there exists some $\y\in \mathbf{A}$ such that $\x\geq \y$ elementwise. Then for $s'$ being \emph{any} set of maximal atoms in $[\p]$ with $|s'| \leq c_3 \p$ and $S' := \{i \in I: I\in s'\}$, we have $\sigma_{min}(\mathbf{X}_{S_t(c_2) \cup S'}) \geq c_4$. 
		\item[(4)] The minimum non-zero and maximum singular value of $\mathbf{X}$, denoted as $\hat{\sigma}^2_{min}(\mathbf{X})$ and $\sigma^2_{max}(\mathbf{X})$, are bounded: i.e.\
		$\hat{\sigma}^2_{min}(\mathbf{X}) \geq \frac{1}{c_5}$ and $\sigma^2_{max}(\mathbf{X}) \leq c_5.$
		\item[(5)] Define $\mathcal{C}_{\x}(\mathbf{b})= \frac{1}{2}\| \y - \X \mathbf{b}\|^2+\sum_{i=1}^p \hat{\lambda}_i |b_i|$ for some $\hat{\blam}\in\mathcal{P}(\hat{\Pi}^{-1}_{\x}(\blam))$. Then $\mathcal{C}(\bet^t)\geq \mathcal{C}_{\bet^t}(\hat{\bet})$.
	\end{itemize}
	then for some function $f(\epsilon) := f(\epsilon, c_1, c_2, c_3, c_4, c_5)$ such that $f(\epsilon) \rightarrow 0$ as $\epsilon \rightarrow 0$,
	\[\frac{1}{\sqrt{\p}}\|\boldsymbol  \beta^{t} - \hat{\boldsymbol  \beta}\| < f(\epsilon).\]
\end{lemma}

We wrap up this section by proving Lemma  \ref{lem:opt_lem}.  Once we have proved Lemma  \ref{lem:opt_lem}, we will be able to prove Theorem \ref{thm:main_result2}. The major piece of work in proving Theorem \ref{thm:main_result2} is in showing that the five assumptions of Lemma  \ref{lem:opt_lem} are satisfied.  Then the result of Theorem \ref{thm:main_result2} is immediate. We show the five assumptions are met in Sections \ref{sec:assumption4} - \ref{sec:assumption3}.  Now we prove the Lemma.

\begin{proof}[Proof of Lemma \ref{lem:opt_lem}]
Throughout the proof, we denote $\xi_1, \xi_2,\dots$ as
functions of the constants $c_1,\dots,c_5>0$ and of $\epsilon$ such that
$\xi_i(\epsilon)\to 0$ as $\epsilon\to 0$ (we omit the dependence of
$\xi_i$ on $\epsilon$).  We will think of $t$ as a fixed iteration and we denote the residual we are interested in studying as $\rr=\hat{\bet}-\bet^t$.   

The proof strategy is to show that  $\frac{1}{\p}\|\X \rr\|^2 \leq \xi(\epsilon)$ from which a similar result for  $\frac{1}{\p} \|\rr\|^2$ follows when we have control of the singular values of $\bm X$ as we do with Condition (4). 
Structurally, the proof is similar to that in the LASSO case (cf.\ \cite[Lemma 3.1]{lassorisk}), with the main difference coming through Condition (3), where we need to use star support instead of the support when bounding the minimum singular value of a selection of columns of $\X$.

For a fixed iteration $t$, let $S = \{i \in [p]: i \in I \text{ and } I \in \supp^*(\bet^t)\}$, i.e.\ $S$ is the collection of (unique) indices belonging to the star support of the AMP estimate at iteration $t$. Then for a vector $\bm v \in \R^{\p}$ we denote $\bm v_{S}$ to mean the vector indexed only over the indices in the set $S$ and we let $\bar{S}$ denote the complement of $S$.  
In what follows, we drop the $t$-dependence on $\bm\nu^t$, writing $\bm\nu=\bm\nu^t$  and for $p$-length vectors $\bm u$ and $\bm v$, define $\langle \bm u,\bm v \rangle:=\frac{1}{p}\sum_i u_i v_i$. 

First,
\begin{align*}
	&0  \stackrel{(a)}{\geq} \frac{1}{\p}  (\mathcal{C}_{\bet^t}(\widehat{\bet})-\mathcal{C}(\bet^t))  \stackrel{(b)}{=} \frac{1}{2\p}(\|\y-\X \widehat{\bet}\|^2-\|\y-\X\bet^t\|^2) + \langle \hat{\blam}, |\widehat{\bet}|  - |\bet^t| \rangle \\
	&\stackrel{(c)}{=} \langle\hat{\blam}_S,|\bet^t_S+\rr_S|-|\bet^t_S|\rangle+\langle\hat{\blam}_{\bar{S}}, |\rr_{\bar{S}}|\rangle+ \frac{1}{2\p}(\|\y-\X \bet^t-\X \rr\|^2-\|\y-\X\bet^t\|^2)\\
	& \stackrel{(d)}{=}  \Big[\langle\hat{\blam}_S,|\bet^t_S+\rr_S|-|\bet^t_S|\rangle-\langle \bm \nu_S, \rr_S\rangle \Big]
	+\Big[\langle\hat{\blam}_{\bar{S}}, |\rr_{\bar{S}}|\rangle-\langle \bm \nu_{\bar{S}}, \rr_{\bar{S}}\rangle \Big]
	+\langle \bm \nu, \rr\rangle
	-\langle \y-\X\bet^t, \X\rr \rangle+\frac{\|\X\rr\|^2}{2\p} 
	\\
	& \stackrel{(e)}{=}  \Big[\langle\hat{\blam}_S,|\bet^t_S+\rr_S|-|\bet^t_S|\rangle-\langle \bm \nu_S, \rr_S\rangle \Big]
	+\Big[\langle\hat{\blam}_{\bar{S}}, |\rr_{\bar{S}}|\rangle-\langle \bm \nu_{\bar{S}}, \rr_{\bar{S}}\rangle \Big]
	+\langle sg(\mathcal{C},\bet^t),\rr\rangle+\frac{\|\X\rr\|^2}{2p}.
\end{align*}
In the above, step $(a)$ follows immediately from Condition (5) and step $(b)$ holds \emph{for any}  $\hat\blam \in \mathcal{P}(\hat\Pi_{\bet^t}^{-1}(\blam))$ by the definition of $\mathcal{C}_{\bet^t}(\widehat{\bet})$, noticing that $J_{\blam}(\bet^t) = \hat\blam^{\top}|\bet^t|$ in the SLOPE cost \eqref{eq:SLOPEcost} \emph{since} $\hat\blam \in \mathcal{P}(\hat\Pi_{\bet^t}^{-1}(\blam))$.  Below we will select a specific $\hat\blam \in \mathcal{P}(\hat\Pi_{\bet^t}^{-1}(\blam))$ based on the definition of $\bm \nu$.
Step $(c)$ follows by replacing $ \widehat{\bet} $ with $\bet^t + \rr$ and noticing that $\bet^t_{\bar{S}} = \bm 0$.  Step $(d)$ follows since $\langle \bm \nu, \rr\rangle=\langle \bm \nu_S, \rr_S\rangle+\langle \bm \nu_{\bar{S}}, \rr_{\bar{S}}\rangle$  and step $(e)$ from the definition of $\bm\nu$.

Using Conditions (1) and  (2), we get by Cauchy-Schwarz
\begin{eqnarray}
\Big[\langle\hat{\blam}_S,|\bet^t_S+\rr_S|-|\bet^t_S|\rangle-\langle \bm\nu_S, \rr_S\rangle\Big]
+\Big[\langle\hat{\blam}_{\bar{S}}, |\rr_{\bar{S}}|\rangle-\langle \bm\nu_{\bar{S}}, \rr_{\bar{S}}\rangle\Big]
+\frac{\|\X\rr\|^2}{2p} \leq \frac{\norm{ sg(\mathcal{C},\bet^t)} \norm{\rr}}{\p} \leq c_1\epsilon.
\label{eq:three-non-neg}
\end{eqnarray}
We now show all three terms on the left side of \eqref{eq:three-non-neg} are non-negative.  The idea is then: if all three terms are non-negative and their sum tends to $0$ as $\epsilon\goto 0$, it must be true that each term tends to $0$ too.  
The third term in \eqref{eq:three-non-neg}, $\frac{1}{2p} \|\X\rr\|^2$, is trivially non-negative, so we focus on the first two.

To show that the other terms are non-negative, we consider choosing a specific vector $\hat{\blam} \in \mathcal{P}(\hat\Pi_{\bet^t}^{-1}(\blam))$ such that on the support, $\hat{\blam}_S=|\bm\nu_S|$, and off the support $\hat{\blam}_{\bar{S}}\geq|\bm\nu_{\bar{S}}|$, meaning $\hat{\blam}_I$ is parallel to $|\bm\nu_I|$ for each equivalence class $I$ of $\bet^t$.
That such a $\hat{\blam}$ exists in the set $\mathcal{P}(\hat\Pi_{\bet^t}^{-1}(\blam))$ follows since $\bm \nu$ is a valid subgradient of $J_{\blam}(\bet^{t})$ (see Fact ~\ref{fact:subgradient}). 

Using this $\hat{\blam}$, notice that the sets defined in Condition (3) are equivalent to the following:
$s_t(c_2) := \{I \subset [\p]: |\bm \nu_I| \geq  (1 - c_2)\hat{\blam}_I\}$ and $S_t(c_2) := \{i : |\nu_i| \geq  (1 - c_2)\hat{\lambda}_i\}$, where both use equivalence classes, $I$, defined for $\bet^t$.  To see that this is the case, note that if $I$ is a non-zero equivalence class, by Fact~\ref{fact:subgradient}, since $|\bm \nu_I| \in  [\mathcal{P}(\hat\Pi^{-1}_{\bet^{t}}(\blam))]_I$, we know that $|\bm \nu_I| \succeq  [\mathcal{P}(\hat\Pi^{-1}_{\bet^{t}}(\blam))]_I(1 - c_2)$ and similarly, since $\hat{\blam}_S=|\bm\nu_S|$ we know that $|\bm \nu_I| \geq  (1 - c_2)\hat{\blam}_I$, so $I$ clearly belongs to $s_t(c_2)$ for both definitions.  If $I$ is the zero equivalence class, 
 if $|\bm \nu_I| \geq  (1 - c_2)\hat{\blam}_I$ then obviously $|\bm \nu_I| \succeq  [\mathcal{P}(\hat\Pi^{-1}_{\bet^{t}}(\blam))]_I(1 - c_2)$ since $\hat{\blam} \in \mathcal{P}(\hat\Pi_{\bet^t}^{-1}(\blam))$. In the other direction, if the non-zero equivalence class $I$ is such that $|\bm \nu_I| \succeq  [\mathcal{P}(\hat\Pi^{-1}_{\bet^{t}}(\blam))]_I(1 - c_2)$ then there exists a vector $\widetilde{\bm \nu}_I \in  [\mathcal{P}(\hat\Pi^{-1}_{\bet^{t}}(\blam))]_I$ such that $|\bm \nu_I| \geq \widetilde{\bm \nu}_I(1-c_2)$ elementwise.  However since $\widetilde{\bm \nu}_I \in  [\mathcal{P}(\hat\Pi^{-1}_{\bet^{t}}(\blam))]_I$, this implies that $|\bm \nu_I| \geq  (1 - c_2)\hat{\blam}_I$ is also true since $\hat{\blam}_I \in  [\mathcal{P}(\hat\Pi^{-1}_{\bet^{t}}(\blam))]_I$ in the same direction as $|\bm \nu_I|$.

To visualize the choice of $\hat{\blam}$, we consider an example where $\bm\nu_I=(-1,2)$ for equivalence class $I=\{1,2\}$ with $\blam_I=(4,1)$ in Figure~\ref{fig:subgradient}.  In the figure, the blue shaded region indicates possible subgradient values for zero elements and the black line are possible subgradients for zero elements.  In this example, the equivalence class is that for zero elements, so we notice that $\bm\nu_I$ lies in the blue region.  Then $\blam_I$ is in the same direction as $|\bm\nu_I|$ but lies on the black line (since $\hat{\blam} \in \mathcal{P}(\hat\Pi_{\bet^t}^{-1}(\blam))$).
\begin{figure}[!h]
	\centering
	\includegraphics[width=6cm]{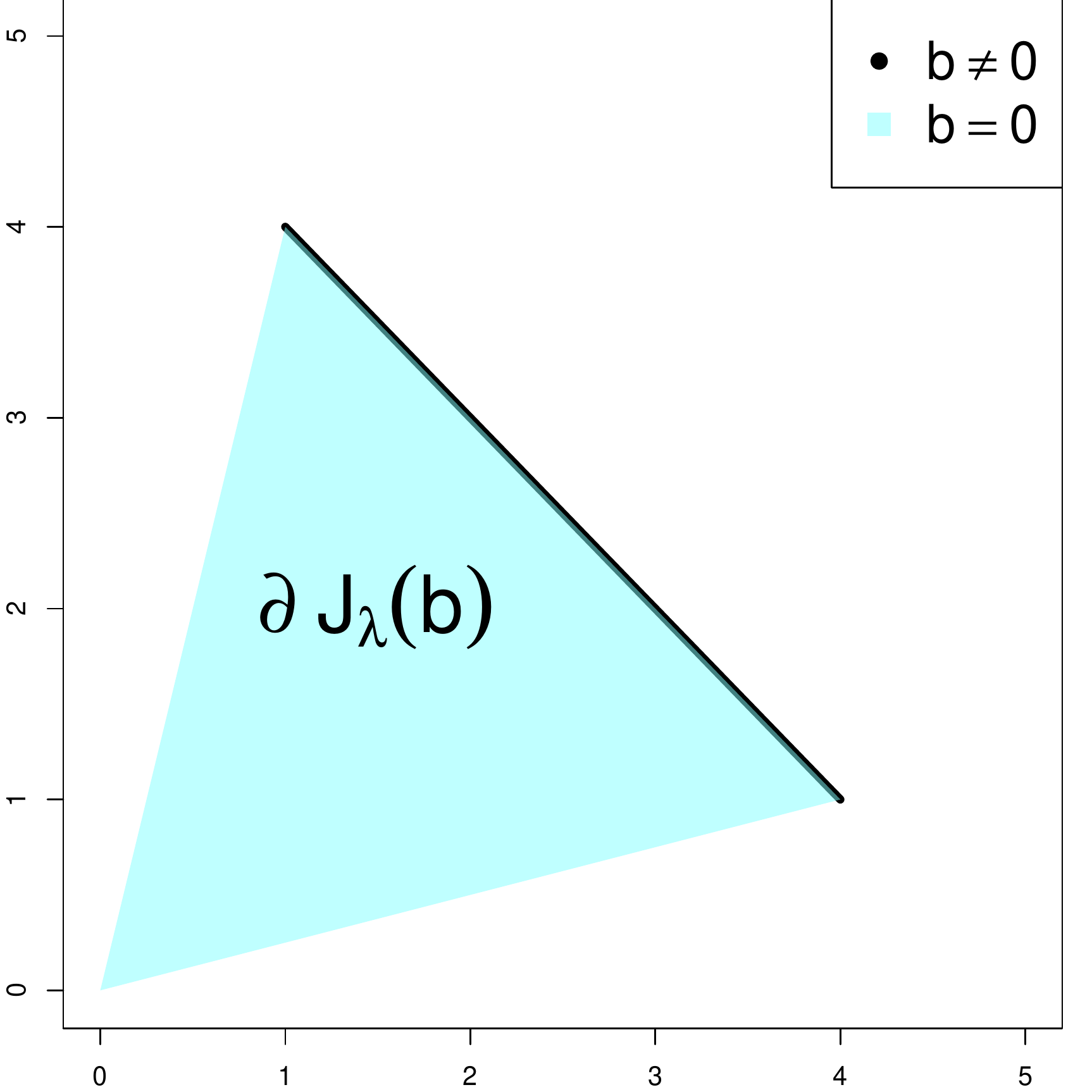}
	\includegraphics[width=6cm]{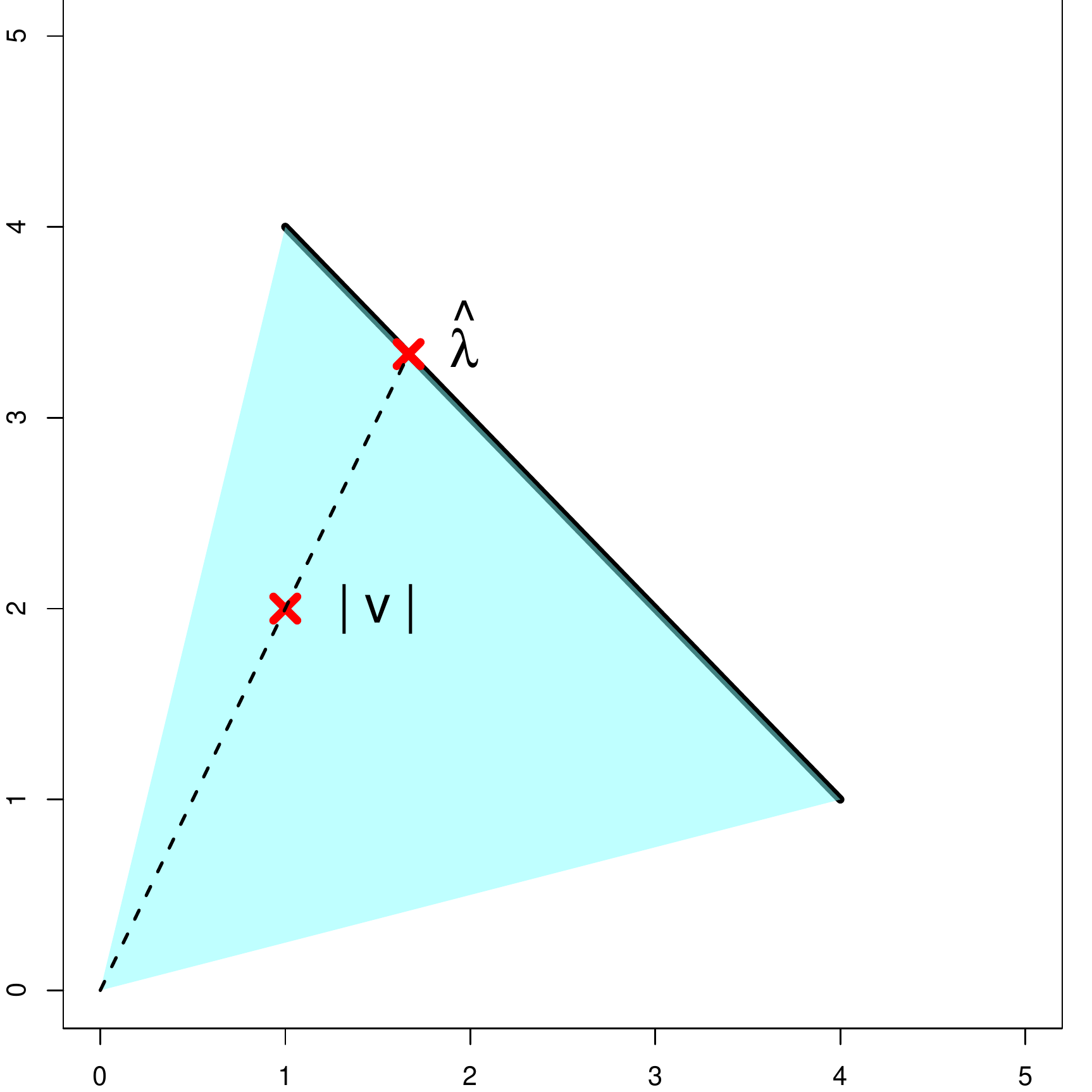}
	\caption{The blue area contained by the black line segment is the set of subgradients; Red crosses are examples of $\bm\nu_I$ and $\hat{\blam}_I$ correspondingly when $\bm b_I=\bm 0$.}
	\label{fig:subgradient}
\end{figure}

Now we would like to show that the first term in \eqref{eq:three-non-neg} is non-negative.  
Specifically, our choice of $\hat{\blam}$ gives $\nu_i=\sgn(\beta^t_i)\hat{\lambda}_i$, for each $i \in S$, and then it suffices, in order to prove the non-negativity of $\langle\hat{\blam}_S,|\bet^t_S+\rr_S|-|\bet^t_S|\rangle-\langle \bm\nu_S, \rr_S\rangle$, to show
\begin{align*}
0 &\leq (|\beta^t_i+r_i|-|\beta^t_i|)-\sgn(\beta^t_i)r_i \\
&=(\beta^t_i+r_i)\sgn(\beta^t_i+r_i)-\beta^t_i\sgn(\beta^t_i)-r_i\sgn(\beta^t_i)
 =(\beta^t_i+r_i)\big[\sgn(\beta^t_i+r_i)-\sgn(\beta^t_i)\big],
\end{align*}
which follows since each $(\beta^t_i+r_i)\left[\sgn(\beta^t_i+r_i)-\sgn(\beta^t_i)\right]$ is either equal to $0$ (when $\sgn(\beta^t_i)=\sgn(\beta^t_i+r_i)$) or equal to $2|\beta^t_i+r_i|$ otherwise. 

Finally, the second term in \eqref{eq:three-non-neg} is also non-negative. It suffices to show for each $i \in \bar{S}$, we have $0\leq \hat{\lambda}_i|r_i|-\nu_i r_i$, or equivalently $ 0\leq \hat{\lambda}_i-\nu_i\sgn( r_i) = \hat{\lambda}_i(1-\sgn(\beta^t_i)\sgn( r_i))$ which is clearly true. Since all three terms in \eqref{eq:three-non-neg} are non-negative and their sum tends to 0 as $\epsilon\goto 0$, it must be true that each term tends to 0,
\begin{eqnarray}
\langle\hat\blam_{\bar{S}}, |\rr_{\bar{S}}|\rangle-\langle {\bm\nu}_{\bar{S}}, \rr_{\bar{S}}\rangle &\leq &\xi_1(\epsilon),\label{eq:Sco}\\
\|\X\rr\|^2 &\leq & p\xi_1(\epsilon).\label{eq:AzBound}
\end{eqnarray}
We now make use of these inequalities to construct the bound for $\frac{1}{p}\|\rr\|^2$. 

Decompose $\rr$ as $\rr = \rr^{\perp}+\rr^\parallel$, with $\rr^\parallel \in \text{ker}(\X)$ and  $\rr^{\perp}\in \text{ker}^{\perp}(\X)$ so that $\X \rr = \X \rr^{\perp}$.  We will now use \eqref{eq:Sco} and \eqref{eq:AzBound} to obtain bounds for $\|\rr^{\perp}\|^{2}$ and $\|\rr^{\parallel}\|^{2}$.  First notice that by \eqref{eq:AzBound} and Condition (4) we have
$\frac{1}{c_5}\|\rr^{\perp}\|^{2} \leq \hat{\sigma}^2_{min}(\X) \|\rr^{\perp}\|^{2} \leq \|\X \rr^{\perp}\|^{2} =  \|\X \rr\|^{2} \leq  p\xi_1(\epsilon).$

In the case $\text{ker}(\X) = \{0\}$, the proof is concluded. Otherwise, we prove a similar bound for $\|\rr^\parallel\|^{2}$. To bound $\|\rr^\parallel\|^{2}$, we use the fact that that this can be done if there exists sets $Q \in [\p]$ and $\bar{Q} \in [\p]/Q$ such that we can bound $\|\rr^\parallel_{\bar{Q}}\|^{2}$ and show a high probability lower bound for $\sigma^2_{min}(\X_Q)$.

In \eqref{eq:Sco}, decompose $\rr_{\bar{S}}=\rr_{\bar{S}}^\perp+\rr_{\bar{S}}^\parallel$ and observe that by Cauchy Schwarz inequality and the bound just obtained,
\be
\langle\hat{\blam}_{\bar{S}}, |\rr^{\perp}_{\bar{S}}|\rangle
\leq \frac{1}{\p}\|\hat{\blam}_{\bar{S}}\| \|\rr^{\perp}_{\bar{S}}\|
\leq \frac{1}{\p} \|\hat{\blam}\| \|\rr^{\perp}\|
\leq \frac{1}{\sqrt{\p}} \|\hat{\blam}\| \sqrt{c_5 \xi_1(\epsilon)} .
\label{eq:new_CS}
\ee 
Then we use the fact that
\begin{align*}
& \langle\hat{\blam}_{\bar{S}}, |\rr_{\bar{S}}^\parallel|\rangle-\langle {\bm\nu}_{\bar{S}}, \rr_{\bar{S}}^\parallel\rangle =  \langle\hat{\blam}_{\bar{S}}, |\rr_{\bar{S}} - \rr_{\bar{S}}^\perp|\rangle-\langle {\bm\nu}_{\bar{S}}, \rr_{\bar{S}} - \rr_{\bar{S}}^\perp\rangle\leq \langle\hat{\blam}_{\bar{S}}, |\rr_{\bar{S}}|\rangle + \langle\hat{\blam}_{\bar{S}}, | \rr_{\bar{S}}^\perp|\rangle-\langle {\bm\nu}_{\bar{S}}, \rr_{\bar{S}}\rangle +\langle {\bm\nu}_{\bar{S}},  \rr_{\bar{S}}^\perp\rangle\\
&= \langle\hat{\blam}_{\bar{S}}, |\rr_{\bar{S}}|\rangle + \langle\hat{\blam}_{\bar{S}}, | \rr_{\bar{S}}^\perp|\rangle-\langle {\bm\nu}_{\bar{S}}, \rr_{\bar{S}}\rangle +\langle \hat{\blam}_{\bar{S}} \sgn(\bet^t_{\bar{S}}),  \rr_{\bar{S}}^\perp\rangle\leq \langle\hat{\blam}_{\bar{S}}, |\rr_{\bar{S}}^\perp|\rangle-\langle {\bm\nu}_{\bar{S}}, \rr_{\bar{S}}\rangle + 2 \langle\hat{\blam}_{\bar{S}}, |\rr_{\bar{S}}^\perp|\rangle,
\end{align*}
to get from \eqref{eq:Sco} and \eqref{eq:new_CS} that
\begin{align}
\langle\hat{\blam}_{\bar{S}}, |\rr_{\bar{S}}^\parallel|\rangle-\langle {\bm\nu}_{\bar{S}}, \rr_{\bar{S}}^\parallel\rangle &\leq \xi_2(\epsilon).\label{eq:Sco2}
\end{align}
Next we would like to show 
\be
\langle\hat{\blam}_{\bar{S}(c_2)}, |\rr_{\bar{S}(c_2)}^\parallel|\rangle-\langle {\bm\nu}_{\bar{S}(c_2)}, \rr_{\bar{S}(c_2)}^\parallel\rangle(1-c_2)^{-1}\geq 0.
\label{eq:new_bound1}
\ee
Note that it suffices again to prove this elementwise for each $i \in \bar{S}(c_2)$.  Specifically, note that $(1-c_2)^{-1}|\nu_i|<\hat{\lambda}_i$ for each $i \in \bar{S}(c_2)$ by the set's definition and therefore $\hat{\lambda}_i|r^{\parallel}_i|-\nu_i r^{\parallel}_i(1-c_2)^{-1} \geq |\nu_i| |r^{\parallel}_i| (1-c_2)^{-1} -\nu_i r^{\parallel}_i(1-c_2)^{-1} \geq 0$.
Therefore,
\be
\begin{split}
&\langle\hat{\blam}_{\bar{S}(c_2)}, |\rr_{\bar{S}(c_2)}^\parallel|\rangle
\overset{(a)}{\leq} \frac{1}{c_2} \langle\blam_{\bar{S}(c_2)}, |\rr_{\bar{S}(c_2)}^\parallel|\rangle- \frac{1}{c_2}\langle {\bm\nu}_{\bar{S}(c_2)},\rr_{\bar{S}(c_2)}^\parallel\rangle
=\frac{1}{c_2}\langle\hat{\blam}_{\bar{S}(c_2)}- {\bm\nu}_{\bar{S}(c_2)}\sgn(\rr_{\bar{S}(c_2)}^\parallel),|\rr_{\bar{S}(c_2)}^\parallel|\rangle\\
&\overset{(b)}{\leq} \frac{1}{c_2} \langle\hat{\blam}_{\bar{S}}- {\bm\nu}_{\bar{S}}\sgn(\rr_{\bar{S}}^\parallel),|\rr_{\bar{S}}^\parallel|\rangle=\frac{1}{c_2} \langle\hat{\blam}_{\bar{S}}, |\rr_{\bar{S}}^\parallel|\rangle- \frac{1}{c_2}\langle {\bm\nu}_{\bar{S}}, \rr_{\bar{S}}^\parallel\rangle \overset{(c)}{\leq} c_2^{-1}\xi_2(\epsilon) \label{eq:Sbar2Sbarc2}.
\end{split}
\ee
In particular, step $(a)$ follows by \eqref{eq:new_bound1}, step $(b)$ since $S \subseteq S_t(c_2)$ implies $\bar{S_t}(c_2)\subseteq \bar{S}$ along with the fact that $\hat{\blam}_{\bar{S}}- {\bm\nu}_{\bar{S}}\sgn(\rr_{\bar{S}}^\parallel) \geq 0$ elementwise (for each $i\in \bar{S}$, we have $\hat{\lambda}_i-\nu_i\sgn(r_i^\parallel)>0$ by $\hat{\lambda}_i\geq |\nu_i|$). Finally step $(c)$ holds by \eqref{eq:Sco2}.  We now use the bound in \eqref{eq:Sbar2Sbarc2} to bound components of $\rr^\parallel$. 

In order to bound $\|\rr^\parallel\|^{2}$, we would like to exploit a relationship between the $\ell_1$ and $\ell_2$ norms.  To do this, we consider an ordering of the elements of the vector $\rr^\parallel$ by magnitude. Recall that $\bar{S_t}(c_2)\subseteq \bar{S}$ and we first assume $|\bar{S_t}(c_2)|\geq \p c_3/2$. Now we partition $\bar{S_t}(c_2) = \cup_{\ell=1}^K S_\ell$, where $(\p c_3/2)\leq |S_{\ell}|\leq \p c_3$, and such that for each $i\in S_{\ell}$ and $j\in S_{\ell+1}$, it follows that  $|r^\parallel_i|\geq |r^\parallel_j|$. Finally, define $\bar{S}_+ := \cup_{\ell=2}^K S_{\ell}\subseteq \bar{S_t}(c_2)$, i.e.\ the set union of all the partitions except the first one corresponding to the indices containing the largest elements in $\rr^\parallel$. Now we note for any $i\in S_{\ell}$, we have $|\rr^{\parallel}_i| \leq \norm{\rr^{\parallel}_{S_{\ell-1}}}/|S_{\ell-1}|$, that is, in terms of absolute value, for any $i$ in group $\ell$, it should be smaller than the average of all the elements in the previous group $\ell-1$.

Then,
\be
\begin{split}
\|\rr^\parallel_{\bar{S}_+}\|^{2} \overset{(a)}{=} \sum_{\ell=2}^K \|\rr^\parallel_{S_{\ell}}\|^{2} \overset{(b)}{\leq}
\sum_{\ell=2}^K |S_{\ell}|
\frac{\|\rr^\parallel_{S_{\ell-1}}\|^2_1}{|S_{\ell-1}|^2} 
&\overset{(c)}{\leq} \frac{4}{p c_3}\sum_{\ell=2}^{K} \|\rr^\parallel_{S_{\ell-1}}\|_1^2 \le
\frac{4}{p c_3} \Big[\sum_{\ell=2}^{K}\|\rr^\parallel_{S_{\ell-1}}\|_1\Big]^2\\
&\overset{(d)}{\leq}
\frac{4}{p c_3}\|\rr^\parallel_{\bar{S}(c_2)}\|_1^2\overset{(e)}{\leq}  \frac{4\xi_2(\epsilon)^2 \p}{c_2^{2}c_3(\min\hat{\blam}_{\bar{S}(c_2)})^2}
=: p\xi_3(\epsilon).
\label{eq:par_res1}
\end{split}
\ee
In the above, step $(a)$ follows from the definition of $\bar{S}_+$, step $(b)$ from the fact that for $i\in S_{\ell}$, we have $|\rr^{\parallel}_i| \leq \norm{\rr^{\parallel}_{S_{\ell-1}}}/|S_{\ell-1}|$, step (c) since $(\p c_3/2)\leq |S_{\ell}|\leq \p c_3$, and step (d) since $\sum_{\ell=2}^{K} S_{\ell} \subset \sum_{\ell=1}^{K} S_{\ell} = \bar{S_t}(c_2)$.  Finally step $(e)$ follows using that $\frac{1}{\p}\min\{\hat{\blam}_{\bar{S}(c_2)}\} \|\rr^\parallel_{\bar{S}(c_2)}\|_1 \leq \langle\hat{\blam}_{\bar{S}(c_2)}, |\rr_{\bar{S}(c_2)}^\parallel|\rangle$.

Now, recalling  $S_+ = S_t(c_2)\cup S_1$ and $|S_1|\le p c_3$, by Condition (3), $\sigma_{min}(\mathbf{X}_{S_+}) \geq c_4$ and therefore,
\begin{align}
c_4^2\|\rr^\parallel_{S_+}\|^{2} \le \sigma^2_{min}(\mathbf{X}_{S_+}) \norm{\rr^{\parallel}_{S_+}}^2 \le \|\X_{S_+}\rr^\parallel_{S_+}\|^2 &\overset{(a)}{=} \|\X_{\bar{S}_+}\rr^\parallel_{\bar{S}_+}\|^2 \overset{(b)}{\leq} 2c_5\|\rr^\parallel_{\bar{S}_+}\|^{2}.
\label{eq:par_res2}
\end{align}
In the above, in step $(a)$ we use that $\bm0 = \X \rr^{\parallel}  =\X_{S_+}\rr^\parallel_{S_+}+\X_{\bar{S}_+}\rr^\parallel_{\bar{S}_+}$.  In step $(b)$ we use Condition (4) and the fact that $\|\X_{\bar{S}_+}\rr^\parallel_{\bar{S}_+}\|^2 \leq \sigma^2_{\max}(\X) \|\rr^\parallel_{\bar{S}_+}\|^2$.  Therefore, to conclude the proof, it is sufficient to prove a  bound for $\|\rr^\parallel_{S_+}\|^{2}$.  

Decomposing $\|\rr^\parallel\|^2 =
\|\rr^\parallel_{S_+}\|^{2}+\|\rr^\parallel_{\bar{S}_+}\|^{2}$, we find from \eqref{eq:par_res1} and \eqref{eq:par_res1} the desired bound:
\begin{align*}
\|\rr^\parallel\|^{2} \leq \|\rr^\parallel_{S_+}\|^{2} + \|\rr^\parallel_{\bar{S}_+}\|^{2}\le \Big(\frac{2c_5}{c_4^2} +1\Big) \|\rr^\parallel_{\bar{S}_+}\|^{2} \le
\Big(\frac{2c_5}{c_4^2} +1\Big) \p \xi_3(\epsilon).
\end{align*}
This finishes the proof when $|\bar{S_t}(c_2)|\ge p c_3/2$.
When $|\bar{S_t}(c_2)| < p c_3/2$, we can take $\bar{S}_+=\emptyset$ and $S_+=[p]$. Hence, the result holds as a special case of the above inequality.
\end{proof}


\section{Expansion of the AMP State Evolution Ideas}\label{app_AMP}

In this section, we develop ideas and notation specifically for the SLOPE AMP algorithm given in \eqref{eq:amp_slope}.  Most are adapted from the  work in \cite{nonseparable} that studies general non-separable AMP algorithms.  These results relate to the performance analysis of the AMP algorithm and will  be useful in proving Lemma \ref{lem:opt_lem}.
Throughout this section, we use the $\{\eta_{\p}^t\}_{\p \in  \mathbb{N}_{>0}}$ notation introduced in Section~\ref{sec_mainproof} and defined in \eqref{eq:eta_def}. Namely, we consider a sequence of denoisers $\eta_{\p}^t: \R^{\p} \rightarrow \R^{\p}$ to be those that apply the proximal operator $\prox_{J_{\bfalph \tau_t}}(\cdot)$ defined in~\eqref{eq:prox}, i.e.\ $\eta_{\p}^t(\mathbf{v}) := \prox_{J_{\bfalph \tau_t}}( \mathbf{v})$ for a vector $\mathbf{v} \in \R^{\p}$.

Given $\w \in \mathbb{R}^{\n}$ and $\bet \in \mathbb{R}^{\p}$, define sequences of column vectors $\mathbf{h}^{t+1} \in \mathbb{R}^{\p}$ and $\mathbf{m}^{t} \in \mathbb{R}^{\n}$ for $t \geq 0$. At each iteration $t$, the sequence $\mathbf{h}^{t+1}$ measures the difference between the truth $\bet$ and the pseudo-data $\mathbf{X}^\top \z^t + \bet^t$, that is the input to the denoiser, and the sequence $\mathbf{m}^{t}$ measures the difference between the noise $\w$ and the AMP residual $\z^t$.  Namely, define $\mathbf{m}^{t}, \mathbf{h}^{t+1}$: for $t \geq 0$, 
\begin{equation}
\begin{split}
\mathbf{h}^{t+1} = \bet - (\mathbf{X}^\top \z^t + \bet^t) \quad \text{ and } \quad \mathbf{m}^t = \w - \z^t.
\end{split}
\label{eq:hqbm_def_AMP}
\end{equation}

We next introduce a generalization to the state evolution given in \eqref{eq:SE2}, that will be useful in studying the limiting properties of functions of the AMP estimates $\bet^s$ and $\bet^t$ at different iterations $s$ and $t$.  To do this, we will recursively define covariances $\{\Sigma_{s,t}\}_{s,t \geq 0}$: for  $\mathbf{B}$ elementwise i.i.d.\ $\sim B$,
set $\Sigma_{0,0} =  \sigma_w^2 + \frac{1}{\delta} \mathbb{E}[B^2]$ and
\be
\Sigma_{0, t+1} = \sigma_w^2 + \lim_{\p} \frac{1}{\delta \p} \mathbb{E}\{-\mathbf{B}^\top[\eta^t_{\p}(\mathbf{B} + \tau_t \mathbf{Z}_t) - \mathbf{B} ]\},
\label{eq:E0_def}
\ee
for $\mathbf{Z}_t \sim \mathcal{N}(0, \mathbb{I})$ independent of $\mathbf{B}$.  
Then for each $t \geq 0$, given $(\Sigma_{s,r})_{0 \leq s,r \leq t}$, define
\be
\Sigma_{s+1, t+1} =  \sigma_w^2 + \lim_{\p} \frac{1}{\delta \p}  \mathbb{E}\Big\{[\eta_{\p}^s(\mathbf{B} + \tau_s \mathbf{Z}_s) - \mathbf{B}]^\top [\eta_{\p}^t(\mathbf{B} + \tau_t \mathbf{Z}_t) - \mathbf{B}]\Big\},
\label{eq:Sigma_def}
\ee
where $ \mathbf{Z}_s$ and $\mathbf{Z}_r$ are length$-\p$ jointly Gaussian vectors, independent of $\mathbf{B} \sim B$ i.i.d.\ elementwise, with $\mathbb{E}[\mathbf{Z}_s] = \mathbb{E}[\mathbf{Z}_r] = \mathbf{0}$, $\mathbb{E}\{([\mathbf{Z}_s]_i)^2\} = \mathbb{E}\{([\mathbf{Z}_r]_i)^2\} = 1$ for any element $i \in [\p]$, and $\mathbb{E}\{[\mathbf{Z}_s]_i [\mathbf{Z}_r]_j\} = \frac{\Sigma_{s, r}}{\tau_r \tau_s} \mathbb{I}\{i=j\}$.  Note that $\Sigma_{t, t} = \tau_t^2$ defined in \eqref{eq:SE2}. 

Using the above covariances, we have the following result that characterizes the asymptotic empirical distributions of the difference vectors defined in \eqref{eq:hqbm_def_AMP} and generalizes Lemma \eqref{thm:main_result1}.  This result follows by \cite[Theorem 1]{nonseparable}.

\begin{lemma}{\cite[Theorem 1]{nonseparable}}
	\label{eq:lem_covariances}
	Assuming that $\Sigma_{0,0}, \ldots, \Sigma_{t+1, t+1} > \sigma_w^2$, then for any deterministic sequence $\phi_{\p}: (\mathbb{R}^{\p} \times \mathbb{R}^{\n})^t \times \mathbb{R}^{\p} \rightarrow \mathbb{R}$ of uniformly pseudo-Lipschitz functions of order $k$,
	\[\plim_{\p} \Big(\phi_{\p}(\bet, \mathbf{m}^0, \mathbf{h}^1, \ldots, \mathbf{m}^t, \mathbf{h}^{t+1}) - \mathbb{E}[\phi_{\p}(\bet, \sqrt{\tau_0^2 - \sigma^2_w} \mathbf{Z}'_0, \tau_0 \mathbf{Z}_0, \ldots, \sqrt{\tau_t^2 - \sigma^2_w} \mathbf{Z}'_t, \tau_t \mathbf{Z}_t)] \Big) = 0,\]
	for $(\mathbf{Z}_0, \mathbf{Z}_1, \ldots, \mathbf{Z}_t)$ defined in \eqref{eq:Sigma_def} in dependent of $(\mathbf{Z}'_0, \mathbf{Z}'_1, \ldots, \mathbf{Z}'_t)$ and the expectation is taken with respect to the collection $(\mathbf{Z}_0, \mathbf{Z}'_0, \mathbf{Z}_1, \mathbf{Z}'_1, \ldots, \mathbf{Z}'_t, \mathbf{Z}_t)$. We note that $ \mathbf{Z}'_s$ and $\mathbf{Z}'_r$ are length$-\n$ jointly Gaussian vectors, with $\mathbb{E}[\mathbf{Z}'_s] = \mathbb{E}[\mathbf{Z}'_r] = \mathbf{0}$, $\mathbb{E}\{([\mathbf{Z}'_s]_i)^2\} = \mathbb{E}\{([\mathbf{Z}'_r]_i)^2\} = 1$ for any element $i \in [\n]$, and $\mathbb{E}\{[\mathbf{Z}'_s]_i [\mathbf{Z}'_r]_j\} = (\Sigma_{s, r} - \sigma_w^2)((\tau^2_r - \sigma_w^2)(\tau^2_s - \sigma_w^2))^{-1/2} \mathbb{I}\{i=j\}$.
\end{lemma}

We use Lemma~\ref{eq:lem_covariances} to explicitly state asymptotic characterizations of AMP quantities that will be useful in our analysis.
\begin{lemma}\label{Thm42}
Under the condition of Theorem \ref{thm:main_result3}, for $\z^t$ and $\bet^{t+1}$ defined in \eqref{eq:amp_slope} and the generalized state evolution sequence defined in \eqref{eq:Sigma_def},
\begin{align}
\plim_{\n} \Big(\frac{1}{\n} \norm{\z^t - \z^{t-1}}^2 - (\tau_t^2 - 2\Sigma_{t, t-1}+ \tau_{t-1}^2)\Big) = 0, \label{eq:corresult1} \\
\plim_{\p} \Big(\frac{1}{\delta \p} \norm{\bet^{t+1} - \bet^{t}}^2 - (\tau_t^2 - 2\Sigma_{t, t-1}+ \tau_{t-1}^2)\Big) = 0.\label{eq:corresult2}
\end{align}
\end{lemma}

\begin{proof}
The major tools in proving \eqref{eq:corresult1}-\eqref{eq:corresult2} are first recognizing that we can write the differences $\z^t - \z^{t-1}$ and $\bet^{t+1} - \bet^t$ as a function of the values $(\bet, \mathbf{m}^0, \mathbf{h}^1, \ldots, \mathbf{m}^t, \mathbf{h}^{t+1})$ defined in \eqref{eq:hqbm_def_AMP} and finally making an appeal to the Law of Large Numbers.  We prove \eqref{eq:corresult2} and \eqref{eq:corresult1} follows similarly.

By \eqref{eq:AMP0}, $\bet^{t+1} - \bet^{t} = \eta_{\p}^t(\bet^{t} + \mathbf{X}^\top \z^t) - \eta_{\p}^{t-1}(\bet^{t-1} + \mathbf{X}^\top \z^{t-1}) =  \eta^{t}_{\p}(\bet - \mathbf{h}^{t+1}) - \eta^{t-1}_{\p}(\bet - \mathbf{h}^{t})$.  Therefore, we will appeal to Lemma~\ref{eq:lem_covariances} for the uniformly pseudo-Lipschitz function
\begin{align*}
\phi_{\p}(\bet, \mathbf{m}^0, \mathbf{h}^1, \ldots, \mathbf{m}^t, \mathbf{h}^{t+1}) &= \frac{1}{\delta  \p}\norm{\bet^{t+1} - \bet^{t}}^2 = \frac{1}{\delta  \p}\norm{\eta^{t}_{\p}(\bet - \mathbf{h}^{t+1}) - \eta^{t-1}_{\p}(\bet - \mathbf{h}^{t})}^2.
\end{align*}
We note that it easy to show that 
the above function is uniformly pseudo-Lipschitz, though we don't do this here.  Then by Lemma~\ref{eq:lem_covariances},
\begin{align}
\plim_{\p} \Big(\frac{1}{\delta  \p}\norm{\bet^{t+1} - \bet^{t}}^2 - \frac{1}{\delta  \p}\mathbb{E}\norm{\eta^{t}_{\p}(\bet -  \tau_{t} \mathbf{Z}_{t}) - \eta^{t-1}_{\p}(\bet -  \tau_{t-1} \mathbf{Z}_{t-1})}^2 \Big) = 0. \label{eq:thmoutput2}
\end{align}
Now to prove result \eqref{eq:corresult1}, we note that by Lemma~\ref{lem:beta_expectation},
\[\plim_{\delta  \p}\frac{1}{\p}\mathbb{E}\norm{\eta^{t}_{\p}(\bet -  \tau_{t} \mathbf{Z}_{t}) - \eta^{t-1}_{\p}(\bet -  \tau_{t-1} \mathbf{Z}_{t-1})}^2 = \lim_{\p} \frac{1}{\delta  \p}\mathbb{E}\norm{\eta^{t}_{\p}(\mathbf{B} -  \tau_{t} \mathbf{Z}_{t}) - \eta^{t-1}_{\p}(\mathbf{B} -  \tau_{t-1} \mathbf{Z}_{t-1})}^2,\]
where $\mathbf{B} \sim B$ i.i.d.\ elementwise independent of $\mathbf{Z}_{t}$ and $\mathbf{Z}_{t-1}$.  The argument for showing that the assumptions of Lemma~\ref{lem:beta_expectation} are met follows like that used in Appendix~\ref{app_props} in the proof of Proposition \textbf{(P2)} introduced in Section~\ref{sec_mainproof}.  Then,
$
\lim_{\p} \frac{1}{\delta  \p}\mathbb{E}\norm{\eta^{t}_{\p}(\mathbf{B} -  \tau_{t} \mathbf{Z}_{t}) - \eta^{t-1}_{\p}(\mathbf{B} -  \tau_{t-1} \mathbf{Z}_{t-1})}^2 
= \Sigma_{t, t} -2\Sigma_{t, t-1}+ \Sigma_{t-1, t-1}.
$

\end{proof}

We finally state a lemma that characterizes the asymptotic value of the normalized $\ell_2$ norm of the residuals in AMP algorithm \eqref{eq:AMP1} following from Lemma~\ref{thm:main_result1}.
\begin{lemma}\label{lem:4.1}
For $\z^t$ defined in \eqref{eq:AMP1} and $\tau_t^2$ given in \eqref{eq:SE2},
\begin{align}
\label{eq:lem41result}
\plim_{\n} \big( \norm{\z^t}^2/\n - \tau_t^2\big) = 0.
\end{align}
\end{lemma}

\begin{proof}
This follows from Lemma~\ref{thm:main_result1},
using the uniformly pseudo-Lipschitz (of order $2$) sequence of functions $\phi_{\n}(\mathbf{a}, \mathbf{b}) = \frac{1}{\n} \norm{\mathbf{a}}^2$ to get, $\plim_{\n}  \norm{\z^t}^2/\n = \plim_{\n} \mathbb{E}_{\mathbf{Z}}[ \norm{ \w + \sqrt{\tau_t^2 - \sigma_{w}^2} \mathbf{Z}}^2]/\n$ for $\bm Z \sim \mathcal{N}(0, \mathbb{I})$.
Then the final result follows by noticing that
$ \mathbb{E}_{\mathbf{Z}} \norm{ \w + \sqrt{\tau_t^2 - \sigma_{w}^2} \mathbf{Z}}^2 =    \norm{ \w}^2 + (\tau_t^2 - \sigma_{w}^2)  \mathbb{E}_{\mathbf{Z}}\norm{\mathbf{Z}}^2 =   \norm{ \w}^2 + \n(\tau_t^2 - \sigma_{w}^2),$
and therefore, using that $ \plim_{\n}  \norm{ \w}^2 /\n = \sigma_{w}^2$ by the Law of Large Numbers,
\[\plim_{\n}  \frac{1}{\n} \mathbb{E}_{\mathbf{Z}} \norm{ \w + \sqrt{\tau_t^2 - \sigma_{w}^2} \mathbf{Z}}^2  =(\tau_t^2 - \sigma_{w}^2)  + \plim_{\n} \frac{1}{\n}  \norm{ \w}^2  = \tau_t^2.\]
\end{proof}


\section{Verification of Main Technical Lemma Conditions} \label{sec:conds}

We now verify that the Lemma \ref{lem:opt_lem} conditions 1-5 are met for the SLOPE cost function and the associated AMP algorithm.  We note that conditions 1, 4, and 5 are straightforward, so their proof is presented first.  On the other hand, condition 2 and condition 3 are quite technical.  Their proofs are given in Section \ref{sec:assumption2} and Section \ref{sec:assumption3} below.

\subsection{Condition (4)} \label{sec:assumption4}
This follows by standard limit theorems about the singular values of Wishart matrices (see Appendix \ref{app_F}, Theorem H.\ref{prop:marchenko-pastur}).

\subsection{Condition (5)} \label{sec:assumption5}
Recall, $\mathcal{C}_{\x}(\mathbf{b})= \frac{1}{2}\| \y - \X \mathbf{b}\|^2+\sum_{i=1}^p \hat{\blam}_i |b_i|$ for some $\hat{\blam}\in\mathcal{P}(\hat{\Pi}^{-1}_{\x}(\blam))$, and 
by definition, $\mathcal{C}_{\x}(\x)=\mathcal{C}(\x)$ for all $\x$. Since $\widehat{\bet}$ is the minimizer of $\mathcal{C}(\cdot)$ we have $\mathcal{C}(\bet^t) \geq \mathcal{C}(\hat{\bet})$ and by the rearrangement inequality, $\mathcal{C}_{\hat{\bet}}(\hat{\bet}) \geq \mathcal{C}_{\bet^t}(\hat{\bet})$.
Therefore, $\mathcal{C}(\bet^t) \geq \mathcal{C}(\hat{\bet}) = \mathcal{C}_{\hat{\bet}}(\hat{\bet}) \geq \mathcal{C}_{\bet^t}(\hat{\bet})$.

\subsection{Condition (1)} \label{sec:assumption1}
Condition (1) follows, for large enough $p$, from Lemma \ref{lem:bounded_vals}, stated below, which proves the asymptotic boundedness of the norms of the AMP estimates $\bet^t$ and the SLOPE estimate $\widehat{\bet}$. 

\begin{lemma} For any parameter vector $\blam \in \R^{\p}$ defining a SLOPE cost as in \eqref{eq:SLOPE_est}, let $\bfalph = \bfalph(\blam)$, then for $t \geq 0$,
\be
\plim_{\p} \frac{1}{\p}\norm{ \bet^{t}}^2 = \plim_{\p} \frac{1}{\p}\mathbb{E}_{\mathbf{Z}}[  \norm{\eta^t_{\p} (  \bet + \tau_t \mathbf{Z})}^2] \leq 2\sigma_{\bet}^2 + 2 \tau_t^2,
\label{eq:res1}
\ee
for $\eta^t_{\p}(\cdot)$ defined in \eqref{eq:eta_def} with $\sigma_{\bet}^2 := \mathbb{E}[B^2] < \infty$ and $\sigma_{\bet}^2 +  \tau_*^2 < \infty$ and
\be
 \plim_{\p} \frac{1}{\p} \norm{  \widehat{\bet}}^2 \leq \mathsf{C},
\label{eq:res2}
\ee
where $\mathsf{C} := \mathsf{C}(\delta, \sigma_{\bet}^2, \sigma_{w}^2, \textsf{B}_{max} , \textsf{B}_{min},  \lambda_{min})$ is a positive constant depending on $\delta, \sigma_{\bet}^2, \sigma_{w}^2,$ along with the singular values of $\X$ through
$
\textsf{B}_{max} \geq  \lim_{\p} \sigma^2_{max}(\X),$ and $\textsf{B}_{min} \leq \lim_{\p} \hat{\sigma}^2_{min}(\X),
$
and a lower bound on the parameter values $ \lambda_{min} := \lim_p \min(\blam)$.

\label{lem:bounded_vals}
\end{lemma}


\begin{proof}
The proof is included in Appendix~\ref{app:bounded_vals}.
\end{proof}


\subsection{Condition (2)} \label{sec:assumption2}
Condition (2)  follows from Lemma \ref{lem:3.3} stated below, for $\epsilon$ arbitrarily small when $t$ is large enough.


\begin{lemma}\label{lem:3.3}
	Under the conditions of Theorem \ref{thm:main_result3}, for every iteration $t$, there exists a subgradient $sg(C, \bet^t)$ of $C$ defined in \eqref{eq:SLOPEcost} at point $\bet^t$ such that almost surely,
$$ \lim_{t}\plim_{\p}\frac{1}{\p}\|sg(C,\bet^t)\|^2=0.$$
\end{lemma}
	The proof is an adaption of \cite[Lemma 3.3]{lassorisk}, though, the subgradient for the SLOPE cost function (studied extensively in Section~\ref{sec:prel-slope-amp}) is quite different than that of the LASSO cost and our analysis requires handling this carefully.
Before we prove Lemma~\ref{lem:3.3}, we state and prove a result which tells us that the asymptotic difference between the AMP output at any two iterations $t$ and $t-1$ goes to zero in $\ell_2$ norm as the algorithm runs.  This result is crucial to the proof of Lemma \ref{lem:3.3}.
\begin{lemma}\label{lem:4.3}
		Under the condition of Theorem \ref{thm:main_result3}, the estimates $\{\bet^t\}_{t\geq 0}$ and residuals $\{\z^t\}_{t\geq 0}$ of
		AMP almost surely satisfy
	$$ \lim_{t}\plim_{\p}\frac{1}{\delta \p}\|\bet^t-\bet^{t-1}\|^2 = 0, \qquad \text{ and } \qquad   \lim_{t}\plim_{\p}\frac{1}{\n}\|\z^t-\z^{t-1}\|^2 = 0$$
	\end{lemma}
\begin{proof}[Proof of Lemma~\ref{lem:4.3}]
This result uses Lemma~\ref{Thm42}, which characterizes the large system limit of $\frac{1}{\n} \norm{\z^t - \z^{t-1}}^2$ and $\frac{1}{\delta \p} \norm{\bet^{t+1} - \bet^{t}}^2$ as both being equal to $\tau_t^2 - 2\Sigma_{t, t-1}+ \tau_{t-1}^2$ where $\Sigma_{t, t-1}$ is the generalized state evolution sequence defined in \eqref{eq:Sigma_def}.  Then Lemma \ref{lem:5.7} (which is stated and proved in Appendix \ref{app_5743}) shows that $\lim_t \, (\tau_t^2 - 2\Sigma_{t, t-1}+ \tau_{t-1}^2) = 0$.
\end{proof}

	\begin{proof}[Proof of Lemma~\ref{lem:3.3}]
	For any vector $\bm \nu^t \in \partial J_{\blam}( \bet^t)$, note that $\bm \nu^t - \mathbf{X}^{\top}(\mathbf{y} - \mathbf{X}\bet^t)$ is a valid subgradient belonging to the set $\partial \mathcal{C}(\bet^t)$ as defined in Fact \ref{Csub}.  Moreover, by AMP \eqref{eq:AMP1}, $\mathbf{y} - \mathbf{X} \bet^{t}  = \mathbf{z}^t - \omega^t\mathbf{z}^{t-1}$ with $ \omega^t :=\frac{1}{\delta \p} [\nabla \eta^{t-1}(\bet^{t-1} + \mathbf{X}^\top \mathbf{z}^{t-1})]$. Therefore we can write,
	\be
	\begin{split}
		&\bm \nu^t - \mathbf{X}^{\top}(\mathbf{y} - \mathbf{X} \bet^t) =\bm \nu^t - \mathbf{X}^{\top}( \mathbf{z}^t - \omega^t\mathbf{z}^{t-1}) =\bm \nu^t - \mathbf{X}^{\top}( \mathbf{z}^t - \mathbf{z}^{t-1}) -( 1- \omega^t)   \mathbf{X}^{\top}\mathbf{z}^{t-1} \\
		& =(\bm \nu^t  -\mu_t \mathbf{X}^{\top} \mathbf{z}^{t-1})- \mathbf{X}^{\top}( \mathbf{z}^t - \mathbf{z}^{t-1})+(\mu_t  -( 1- \omega^t) )  \mathbf{X}^{\top}\mathbf{z}^{t-1},
		\label{sub_grad_array}
	\end{split}
	\ee
	where we define $\mu_t : = {\langle  \blam, \thet_{t-1} \rangle}/{\| \thet_{t-1}\|^2}$ as the ratio of $\blam$ to $\thet_{t-1}$ so that $\blam=\mu_t\thet_{t-1}$ (here $\thet_{t-1}:=\bfalph\tau_{t-1}$ and recall that $\bfalph$ is calibrated to be parallel to $\blam$). It follows that
	$\partial J_{\blam}(\bm x)=\mu_t \, \partial J_{\thet_{t-1}}(\bm x).$

	Now, by the definition of the proximal operator used in \eqref{eq:AMP0} and by Fact \ref{fact:sub_prox}, we have that
$(\mathbf{X}^{\top} \mathbf{z}^{t-1} + \bet^{t-1}) - \bet^t \in \partial J_{\thet^{t-1}}( \bet^t).$
	Hence we choose $\bm \nu^t$ to be the specific subgradient defined by 
	\begin{align}
	\bm \nu^t=\mu_t(\X^\top \z^{t-1}+\bet^{t-1}-\bet^t)\in\partial J_{\blam}(\bet^t),
	\label{v_definition}
	\end{align}
	which leads to
	$\bm \nu^t-\mu_t \X^{\top} \z^{t-1}=\mu_t ( \bet^{t-1} - \bet^t ).$
	Plugging into \eqref{sub_grad_array},
	\be
		\bm \nu^t - \X^{\top}(\y - \X \bet^t) = \mu_t (\bet^{t-1} - \bet^t ) - \X^{\top}( \z^t - \z^{t-1})+(\mu_t  -( 1- \omega^t) )  \X^{\top}\z^{t-1}.
		\label{sub_grad_array2}
	\ee
	Then taking the norm, dividing by $\sqrt{\p}$, and using the triangular inequality, we have
	\ben
		\frac{1}{\sqrt{\p}}\norm{\bm \nu^t - \X^{\top}(\y - \X \bet^t)} \leq \frac{\mu_t}{\sqrt{\p}}\norm{\bet^{t-1} - \bet^t} + \frac{1}{\sqrt{\p}}\norm{\X^{\top}( \z^t - \z^{t-1})}+\frac{(\mu_t  -( 1- \omega^t) )}{\sqrt{\p}}\norm{\X^{\top}\z^{t-1}}.
	\een	
Using Lemma \ref{Thm42}, that $\sigma_{\max}(\X)$ is almost surely bounded as $\p\to\infty$ (cf.\ Theorem \ref{prop:marchenko-pastur}), and that $\lim_{t}\lim_{p}\mu_t=1 - \lim_{\p}\frac{1}{\delta \p} \E || \prox_{J_{\bm A(\p)\tau_*}}(\mathbf{B}+ \tau_* \mathbf{Z}) ||_0^*$ as in \eqref{eq:lambda_func} is finite, the first two terms on the right side of the above $\goto 0$. Finally, for the third term, Lemma \ref{lem:4.1} gives $\lim_t \plim_p\|z^t\|/\sqrt{\p} = \tau_*$, and together with the calibration formula \eqref{eq:lambda_func}, that $\sigma_{\max}(\X)$ is almost surely bounded as $\p\to\infty$, and the definition of $\omega$ in the proof of Lemma \ref{lm:slope_stationary}, we find $\lim_{t}\lim_{\p}(\mu_t  -( 1- \omega^t))=0,$ and thus the third term $\goto 0$.  As $\bm \nu^t - \X^{\top}(\y - \X \bet^t) \in \partial \mathcal{C}(\bet^t)$, the proof is complete.
\end{proof}


\subsection{Condition (3)} \label{sec:assumption3}
We take $\bm\nu^t$ to be the subgradient defined in \eqref{v_definition} and since $t$ is fixed, we drop the superscript $t$ writing $\bm \nu := \bm \nu^t$. 
Recall the sets $s_t(c_2)$ and $S_t(c_2)$ defined in Condition (3).  
Then for $s'$ being \emph{any} set of maximal atoms in $[\p]$ with $|s'| \leq c_3 \p$ and $S' := \{i \in I: I\in s'\}$, we would like to show $\sigma_{min}(\mathbf{X}_{S_t(c_2) \cup S'}) \geq c_4$.
This holds by Proposition \ref{propo:PSD}, stated below, whose proof is the main challenge. We state the proposition and then we identify two auxiliary lemmas, Lemma \ref{lemma:MinS} and \ref{lemma:ConvergenceSupport}, that will be used to ultimately prove Proposition \ref{propo:PSD}.
\begin{proposition}
	\label{propo:PSD}
	There exist constants $c_2\in(0,1)$, $c_3$, $c_4>0$ and $t_{\rm min}<\infty$ such that,
	for any $t\ge t_{\rm min}$, and set $S_t$ defined in Condition (3)
	\begin{eqnarray*}
	\min_{s'}\big\{\sigma_{\rm min}(\X_{S_{t}(c_2)\cup S'})\, :
	S'\subseteq [p]\, ,
	\;|s'|\le c_3 p\, ,
	\; S' = \{i \in I: I \in s'\}\big\}
	\ge c_4\,
	\end{eqnarray*}
	eventually almost surely as $p\to\infty$.
\end{proposition}

The proof of Proposition \ref{propo:PSD} will use two auxiliary lemmas,  Lemma \ref{lemma:MinS} and \ref{lemma:ConvergenceSupport}, stated below.

\begin{lemma}
	\label{lemma:MinS}
	Let the set $s_t$ be measurable on the $\sigma$-algebra $\mathfrak{S}_t$ generated by $\{\z^0,\dots, \z^{t-1}\}$ and	$\{\bet^0+\X^*\z^0,\dots,\bet^{t-1}+\X^*\z^{t-1}\}$ and assume $|s_t|\le p(\delta-c)$ for some $c>0$. Define $S_t\subseteq [\p]$ as $\{i\in I \text{ for some } I\in s_t\}$. Then there exists $a_1=a_1(c)>0$ (independent of $t$) and $a_2=a_2(c,t)>0$ (depending on $t$ and $c$) such that
	\begin{eqnarray*}
	\min_{s'}\big\{\sigma_{\rm min}(\X_{S_t\cup S'})\, :S'\subseteq [\p]\, ,
	\;|s'|\le a_1 p\, ,
	\; S' = \{i \in I: I \in s'\}\big\}
	\ge a_2\,,
	\end{eqnarray*}
	eventually almost surely as $\p\to\infty$.
\end{lemma}

\begin{proof}
The proof of Lemma \ref{lemma:MinS} is given in Appendix \ref{app_34}. The key difference in SLOPE case (Lemma \ref{lemma:MinS}) and LASSO case (cf.\ \cite[Lemma 3.4]{lassorisk}) is the concept of equivalence classes of indices. On a high level, the set $s$ describes some structure in the support space $S$ and such structure restricts the dimension of some linear spaces in the proof of Lemma \ref{lemma:MinS}.
\end{proof}

\begin{lemma}{\cite[Lemma 3.5]{lassorisk}}
	\label{lemma:ConvergenceSupport}
	Fix $\gamma\in(0,1)$ and let the sequence $\{S_t(\gamma)\}_{t\ge 0}$
	be defined as before.
	For any $\xi>0$ there exists $t_*=t_*(\xi,\gamma)<\infty$ such that,
	for all $t_2\ge t_1\ge t_*$ fixed, we have
	\begin{eqnarray}
	\frac{1}{\p} |S_{t_2}(\gamma)\setminus S_{t_1}(\gamma)|  <\xi\,,
	\end{eqnarray}
	eventually almost surely as $p\to\infty$.
\end{lemma}

\begin{proof}
For LASSO, this result was given in \cite[Lemma 3.5]{lassorisk}, and for SLOPE, the proof stays largely the same so we don't repeat it here.  The major difference is that where the work in \cite{lassorisk} can appeal to AMP analysis in \cite{amp2}, for SLOPE, we appeal to similar results given in \cite{nonseparable} (e.g.\ Lemma~\ref{eq:lem_covariances}).
\end{proof}

\begin{proof}[Proof of Proposition \ref{propo:PSD}]
The subgradient in Condition (2) is given by $sg(\mathcal{C}, \bet^{t}) := \bm\nu^t - \X^\top(\y - \X \bet^{t}) $ where $\bm\nu^t  \in \partial J_{\blam}(\bet^{t})$ is the subgradient defined in the Condition (2) proof at Eq.\ \eqref{v_definition}.
Recall,
$S_t(c_2) = \{i \in I: |\bm\nu^{t}_I| \succeq  \mathcal{P}([\hat{\Pi}^{-1}_{\bet^{t} }(\blam)]_I) (1 - c_2)\}$. We include a  simple visualization for the set $S_t(c_2)$ in Figure~\ref{eq:fig_viz}.  We have plotted the subgradient $\bm\nu^{t}_I=(-1, 2)$ for (zero) equivalence class $I=\{1, 2\}$ when $\blam=(4,1)$ and $\bet^t=(0,0)$.  Then indices of $|\bm\nu^{t}_I|$, namely $(1,2)$ are in $S_{t}(c_2)$ unless $c_2< 0.4$.
\begin{figure}[!h]
	\centering
	\includegraphics[width=6cm]{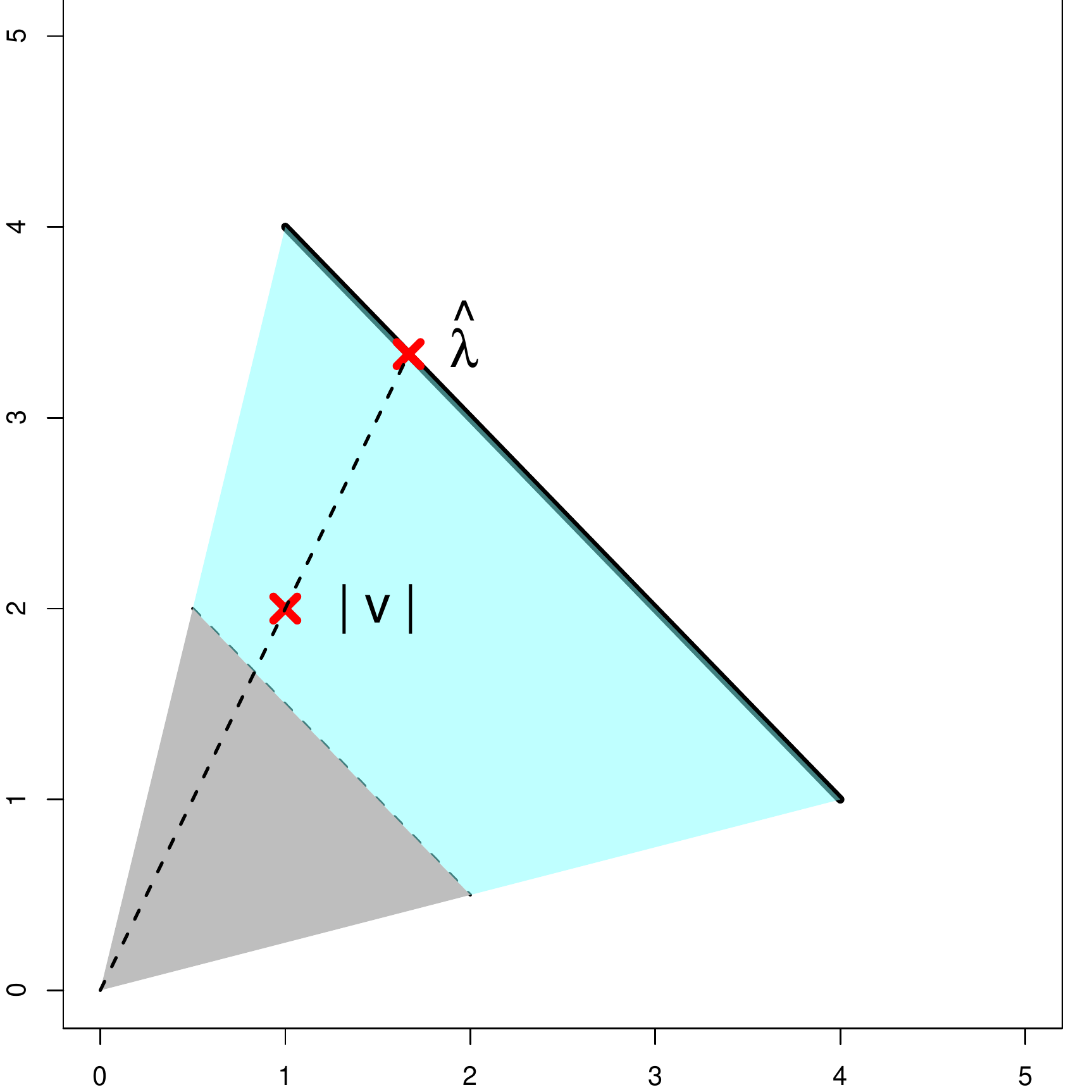}
	\includegraphics[width=6cm]{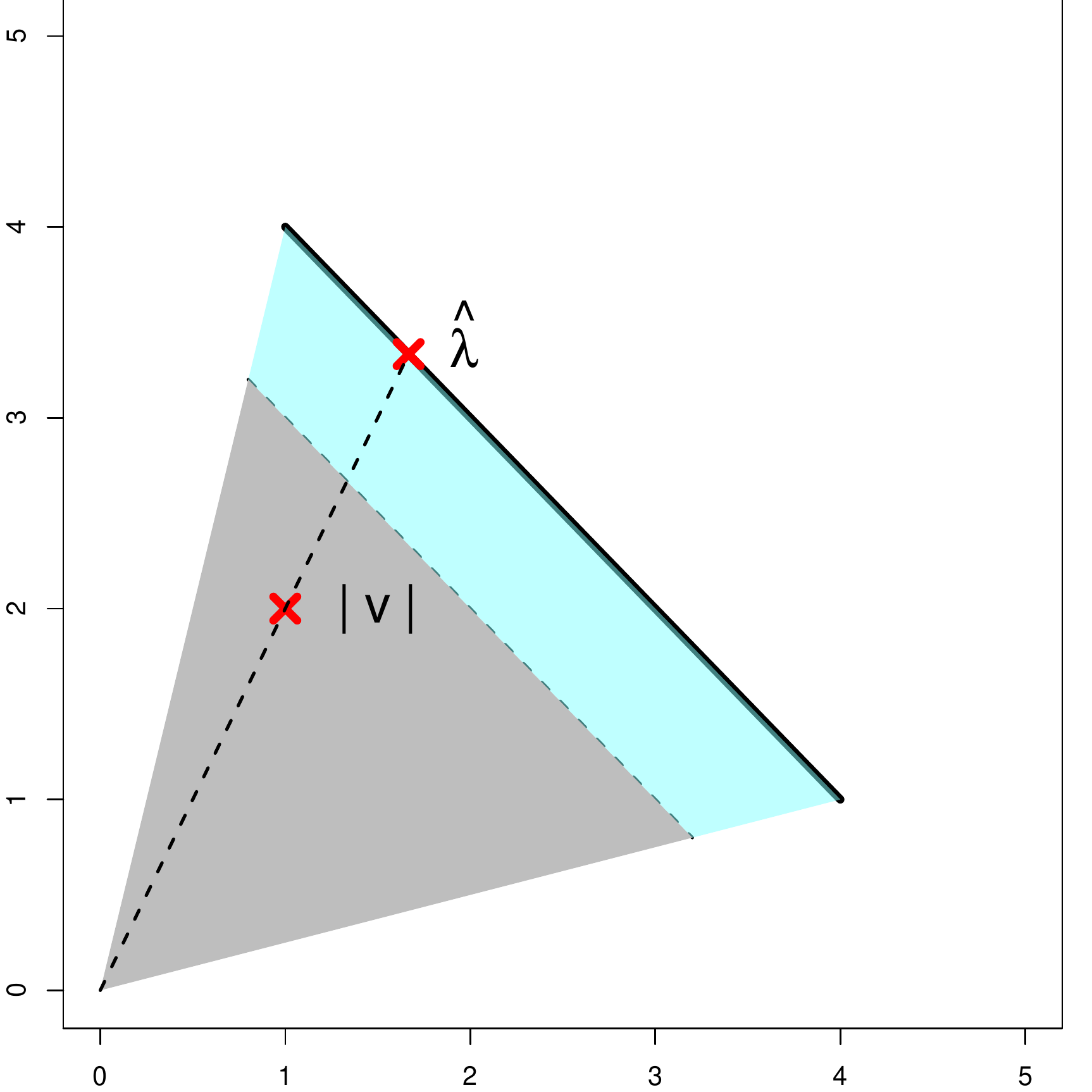}
	\caption{Left: $c_2=0.5$; Right: $c_2=0.2$; Blue area is $\{\bm\nu\in\partial J_{\blam}(0,0):|\bm\nu|\succeq (1 - c_2)\mathcal{P}(\lambda_1,\lambda_2)\}$ and grey area is complement of blue area in $\partial J_{\blam}(0,0)$.}
\label{eq:fig_viz}
\end{figure}

We know from the proof of Lemma \ref{lem:3.3} Eq.\ \eqref{v_definition} that $\bm \nu^{t}=\mu_{t}  (\X^{\top} \z^{t-1} + \bet^{t-1} - \bet^{t} ) \in \mu_{t}   J_{\thet^{t} }( \bet^{t} )$ where $\mu_t : = {\langle  \blam, \thet_{t-1} \rangle}/{\| \thet_{t-1}\|^2}$ and $\blam=\mu_t\thet^{t-1}$.  
Therefore, summing over all equivalence classes $I$,
\be
\begin{split}
\label{eq:set_equiv}
|s_t(c_2)| &= \sum_{I}\mathbb{I}\{|\boldsymbol \nu^{t}_I| \succeq   \mathcal{P}([\hat{\Pi}^{-1}_{ \bet^{t} }( \blam)]_I) (1 - c_2)\} \\
&= \sum_{I}\mathbb{I}\Big\{|\bet^{t} - [\mathbf{X}^{\top} \mathbf{z}^{t-1}] - \bet^{t-1}|_I \succeq   \mathcal{P}([\hat{\Pi}^{-1}_{\bet^{t}}(\thet^{t-1})]_I) (1 - c_2) \Big\}.
\end{split}
\ee
As detailed in the proof of Lemma~\ref{lem:opt_lem}, for non-zero equivalence classes, let $\hat{\blam}_I=|\bm\nu_I|$, and for the zero equivalence class, let $\hat{\blam}_{I}\geq|\bm\nu_{I}|$, meaning $\hat{\blam}_I$ is parallel to $|\bm\nu_I|$ for each equivalence class $I$ of $\bet^t$.
That such a $\hat{\blam}$ exists in the set $\mathcal{P}(\hat\Pi_{\bet^t}^{-1}(\blam))$ follows since $\bm \nu$ is a valid subgradient of $J_{\blam}(\bet^{t})$ (see Fact ~\ref{fact:subgradient}). We can then simplify the set definitions of $s_t(c_2)$ and $S_t(c_2)$ to be $s_t(c_2) := \{I \subset [\p]: |\bm \nu_I| \geq  (1 - c_2)\hat{\blam}_I\}$ and $S_t(c_2) := \{i : |\nu_i| \geq  (1 - c_2)\hat{\lambda}_i\}$, where both use equivalence classes, $I$, defined for $\bet^t$.  Then since $\blam=\mu_t\thet^{t-1}$, we also let $\hat{\thet}^{t-1}$ be defined such that $\hat{\blam}=\mu_t \hat{\thet}^{t-1}$.

Therefore, by \eqref{eq:set_equiv},
$
|s_t(c_2)| = \sum_{I}\mathbb{I}\{|\bet^{t} - [\mathbf{X}^{\top} \mathbf{z}^{t-1}] - \bet^{t-1}|_I \geq \hat{\thet}^{t-1}_I  (1 - c_2) \}.
$
In the notation of \eqref{eq:hqbm_def_AMP},
$\bet^{t} - [\X^{\top} \z^{t-1}] - \bet^{t-1} =  \bm h^{t} + \eta^{t-1}(\bet - \bm h^{t}) - \bet$ and
$\bet^{t} =  \eta^{t-1}(\bet - \bm h^{t})$
and therefore by \eqref{eq:set_equiv},
\begin{align*}
%
 |s_t(c_2)| &= \sum_{I}\mathbb{I}\Big\{| \bm h^{t} + \eta^{t-1}(\bet - \bm h^{t}) - \bet|_I\geq \hat{\thet}^{t-1}_I (1 - c_2) \Big\}.
\end{align*}
Now, we note that Lemma \ref{eq:lem_covariances} implies weak convergence of the empirical distribution of $\bm h^{t}$ to $\tau_{t-1} \mathbf{Z}_{t-1}$  for $ \mathbf{Z}_{t-1}$ a vector of i.i.d.\ standard Gaussian and $\tau_{t-1}$ given by the state evolution~\eqref{eq:SE2}.  Therefore a careful argument using continuous approximations to indicators gives,
\be
\begin{split}
\label{eq:port}
& \plim_\p \frac{1}{\p}  \sum_{I} \mathbb{I}\Big\{| \bm h^{t} + \eta^{t-1}(\bet - \bm h^{t}) - \bet|_I \geq  \hat{\thet}^{t-1}_I (1 - c_2) \Big\} \\
&= \lim_\p \E_{\mathbf{Z}_{t-1}}\Big\{\frac{1}{\p} \sum_{I}   \mathbb{I}\Big\{| \tau_{t-1} \mathbf{Z}_{t-1} + \eta^{t-1}(\bet -  \tau_{t-1} \mathbf{Z}_{t-1}) - \bet|_I \geq  \hat{\thet}^{t-1}_I (1 - c_2) \Big\}\Big\}, 
\end{split}
\ee
where in the right side of the above, the equivalence classes $I$ are taken with respect to $\eta^{t-1}(\bet -  \tau_{t-1} \mathbf{Z}_{t-1})$ and $\hat{\thet}^{t-1}_I $ as equal to or larger than $| \tau_{t-1} \mathbf{Z}_{t-1} + \eta^{t-1}(\bet -  \tau_{t-1} \mathbf{Z}_{t-1}) - \bet|_I$ depending on whether $I$ is the zero equivalence class or not. We justify the substitution of $\tau_{t-1}\mathbf{Z}_{t-1}$ for $\bm h^t$ by approximating the sum of indicators with a function that counts the number of elements in $\eta^{t-1}(\bet-\bm h^t)$ that are strictly greater than its neighbour. Then this function converges to a continuous and bounded function, the function that measures the proportion of $\eta^{t-1}$ that is non-flat, to which we apply the Portmanteau Theorem (cf. \cite{SLOPEasymptotic}, Lemma 1(b) in \cite{amp2} and Lemma F.3(b) in \cite{lassorisk}).

Now, using \eqref{eq:port}, we can simplify:
\be
\begin{split}
& \plim_\p \frac{1}{\p} |s_{t}(c_2)| = \lim_\p \frac{1}{\p} \sum_{I}  \PP_{\mathbf{Z}_{t-1}}\Big(| \tau_{t-1} \mathbf{Z}_{t-1} - \eta^{t-1}(\bet -  \tau_{t-1} \mathbf{Z}_{t-1}) - \bet|_I \geq \hat{\thet}_I^{t-1}(1 - c_2) \Big),
\label{eq:prob_v1}
\end{split}
\ee
and we study the probability on the right side of the above, for a fixed equivalence class $I$, writing $\eta^{t-1}(\bet -  \tau_{t-1} \mathbf{Z}_{t-1})$ to be $\eta^{t-1}$, dropping the input.
\be
\begin{split}
	& \PP\Big(| \tau_{t-1} \mathbf{Z}_{t-1} + \eta^{t-1} - \bet|_I \geq \hat{\thet}^{t-1}_I (1 - c_2)\Big) \\
	&= \PP\Big(|\tau_{t-1} \mathbf{Z}_{t-1} + \eta^{t-1} -\bet|_I \geq \hat{\thet}^{t-1}_I (1 - c_2), \eta^{t-1}_1= \bm 0 \Big) \\
	&\qquad +\PP\Big(| \tau_{t-1} \mathbf{Z}_{t-1} + \eta^{t-1}- \bet|_I \geq \hat{\thet}^{t-1}_I (1 - c_2), \eta^{t-1}_I \neq \bm 0 \Big)
	\\
	&\overset{(a)}{=} \PP\Big(\hat{\thet}^{t-1}_I \geq|\bet -  \tau_{t-1} \mathbf{Z}_{t-1}|_I  \geq \hat{\thet}^{t-1}_I (1 - c_2) \Big)+\PP\Big( \hat{\thet}^{t-1}_I   \geq \hat{\thet}^{t-1}_I (1 - c_2)\Big) \PP(\eta^{t-1}_I \neq \bm 0).
	\\
	&=\PP\Big(\hat{\thet}^{t-1}_I \geq |\bet -  \tau_{t-1} \mathbf{Z}_{t-1} |_I  \geq \hat{\thet}^{t-1}_I (1 - c_2) \Big) +\PP(\eta^{t-1}_I \neq 0).\label{eq:STsetsize1}
\end{split}
\ee
In the above, step $(a)$ follows when $ \eta^{t-1}_I= [\prox_{J_{\thet^{t-1}}}(\bet -  \tau_{t-1} \mathbf{Z}_{t-1})]_I = \bm 0$, since we must have $| \bet -  \tau_{t-1} \mathbf{Z}_{t-1} |_I\leq \hat{\thet}^{t-1}_I$, and when $ \eta^{t-1}_I \neq \bm 0$, by Fact~\ref{fact:sub_prox} and Fact \ref{fact:subgradient}, we know that $|\eta^{t-1}(\bet - \tau_{t-1} \mathbf{Z}_{t-1}) - (\bet - \tau_{t-1} \mathbf{Z}_{t-1}) |_I \in \mathcal{P}([\hat{\Pi}^{-1}_{ \eta^{t-1}}(\thet^{t-1})]_I)$.

It obvious that one can make the first probability arbitrarily small by bringing $c_2$ to $0$. To see this, say $1\in I$ and notice that $\mathcal{P}([\hat{\Pi}^{-1}_{\eta^{t-1}}(\thet^{t-1})]_I)$ always has Lebesgue measure 0 because it is a subset of the hyperplane $\{\bm x\in\R^\p: \sum_{j\in I}x_j=\sum_{j\in I}\theta^{t-1}_j\}$.

On the other hand, notice that 
\[\sum_{I}\PP([\eta^{t-1}(\bet -  \tau_{t-1} \mathbf{Z}_{t-1})]_I\neq \bm 0)=\sum_{I}\E\{\mathbb{I}([\eta^{t-1}(\bet -  \tau_{t-1} \mathbf{Z}_{t-1})]_I\neq \bm 0)\}=\E_{ \mathbf{Z}_{t-1}}\|\eta^{t-1}(\bet -  \tau_{t-1} \mathbf{Z}_{t-1})\|_0^*,\]
and that $\eta^{t-1}$ is asymptotically separable by Lemma \ref{lem:yue_sep}.  Define $h^{t-1}(x)=h(x; B + \tau_{t-1} Z, \Theta^{t-1})$ with $\Theta^{t-1}$ being the distribution to which the empirical distribution of $\thet^{t-1}$ converges, and also define
$$\mathbf{W}_{t-1}:=\left\{x \,\,  \Big| \,\, h^{t-1}(x)\neq 0 \text{ and } m\{z \,\, | \,\, |h^{t-1}(z)|=|h^{t-1}(x)|\}=0\right\}$$ 
similarly to \eqref{eqn:non flat}, where $m$ is the Lebesgue measure.  Then,
\begin{align*}
&\lim_p  \frac{1}{\p} \E_{ \mathbf{Z}_{t-1}}\|\eta^{t-1}(\bet -  \tau_{t-1} \mathbf{Z}_{t-1})\|_0^* = \lim_p \frac{1}{\p} \E_{ \mathbf{Z}_{t-1}}\|h^{t-1}(\bet -  \tau_{t-1} \mathbf{Z}_{t-1})\|_0^*
\\
&=\lim_p \frac{1}{\p} \E_{ \mathbf{Z}_{t-1}}\sum_{i=1}^{\p} \mathbb{I}\left\{(\beta_i -  \tau_{t-1} Z_{t-1,i})\in \mathbf{W}_{t-1}\right\} =\lim_\p \frac{1}{\p} \E_{ \mathbf{Z}_{t-1}, \bm B}\|\eta^{t-1}(\bm B -  \tau_{t-1} \mathbf{Z}_{t-1})\|_0^*,
\end{align*}
where the last equality holds by Lemma \ref{lem:beta_expectation}.

Then \eqref{eq:blam_alph_mapping}
gives this term is smaller than $\delta$ for large $t$. Hence, by \eqref{eq:prob_v1} and \eqref{eq:STsetsize1},
\ben
\begin{split}
& \plim_\p \frac{1}{\p} |s_{t}(c_2)| \\
&= \lim_\p \frac{1}{\p} \sum_{I}\PP\Big(\hat{\thet}^{t-1}_I \geq |\bet -  \tau_{t-1} \mathbf{Z}_{t-1} |_I  \geq \hat{\thet}^{t-1}_I (1 - c_2) \Big) +  \lim_\p \frac{1}{\p} \E_{ \mathbf{Z}_{t-1}, \bm B}\|\eta^{t-1}(\bm B -  \tau_{t-1} \mathbf{Z}_{t-1})\|_0^*,
\end{split}
\een
Therefore, for some $c>0$, choose $c_2\in (0, 1)$ such that the first term on the right side of the above is arbitrarily small along with $t_{\min,1}(c)$ such that the second term is arbitrarily close to $\delta$, meaning
$$\lim_\p \PP\left(\frac{1}{\p}|s_t(c_2)|< \delta-c\right)=1, $$
for all fixed $t$ larger than some $t_{\min,1}(c)$.

For any $t\ge t_{{\rm min},1}(c)$ we can apply Lemma \ref{lemma:MinS}
for some $a_1(c)$, $a_2(c,t)$.  Note this doesn't immediately give the result we use since the lower bound, $a_2$, depends on $t$.  To get around this we additionally appeal to Lemma~\ref{lemma:ConvergenceSupport} that tells us after some time $t_{*}$, the supports of the AMP estimates don't change appreciably. Now we fix $c>0$ and consequently $a_1=a_1(c)$ is fixed. Define $t_{\rm min} =\max(t_{{\rm min},1},
t_{*}(a_1/2,c_2))$ with $t_*(\,\cdot\,)$ defined as in Lemma \ref{lemma:ConvergenceSupport} and let $a_2=a_2(c,t_{\rm min})$. 
Then, by Lemma \ref{lemma:MinS} and the fact that $a_2(c,t)$ is non-increasing in $t$,
\begin{eqnarray*}
\min\big\{\sigma_{\rm min}(\X_{S_{t_{\rm min}}(c_2)\cup S'})\, :\;\;
S'\subseteq [\p]\, ,\;|s'|\le a_1\p\big\}
\ge a_2.
\end{eqnarray*}
In addition, by Lemma \ref{lemma:ConvergenceSupport}, 
$
|S_{t}(c_2)\setminus S_{t_{\rm min}}(c_2)|\leq \p a_1/2.
$
Both events hold eventually almost surely as $\p\to\infty$. 
The proof completes with $c_3 = a_1(c)/2$ and $c_4=a_2(c,t_{\rm min})$, fixed with respect to $t$.
\end{proof}


\section{Discussion and Future Work}\label{sec_discussion}

This work develops and analyzes the dynamics of an approximate message passing (AMP) algorithm with the purpose of solving the SLOPE convex optimization procedure for high-dimensional linear regression.  By employing recent theoretical analysis of AMP when the non-linearities used in the algorithm are non-separable \cite{nonseparable}, as is the case for the SLOPE problem, we provide rigorous proof that the proposed AMP algorithm finds the SLOPE solution asymptotically.  Moreover empirical evidence suggests that the AMP estimate is already very close to the SLOPE solution even in few iterations.  By leveraging our analysis showing AMP provably solves SLOPE, we provide an exact asymptotic characterization of the $\ell_2$ risk of the SLOPE estimator from the underlying truth and insight into other statistical properties of the SLOPE estimator.  Though this asymptotic analysis of the SLOPE solution has been demonstrated in other recent work \cite{SLOPEasymptotic} using a different proof strategy, we believe that our AMP-based approach offers a more concrete and algorithmic understanding of the finite-sample behavior of the SLOPE estimator.

A limitation of this approach is that the theory assumes an i.i.d.\ Gaussian measurement matrix, and moreover, the AMP algorithm can become unstable when the measurement matrix is far from i.i.d., creating the need for heuristic techniques to provide convergence in applications where the measurement matrix is generated by nature (i.e., a real-world experiment or observational study).   Additionally, the asymptotical regime studied here, $n/\p \rightarrow \delta \in (0,\infty)$, requires that the number of columns of the measurement matrix $\p$ grow at the same rate as the number of rows $\n$.  It is of practical interest to extend the results to high-dimensional settings where $\p$ grows faster than $\n$.


\bibliographystyle{ieeetr}
{\small{
\bibliography{ref}
}}

\clearpage


\appendix
\section{State Evolution Analysis}\label{app_SE}

We first prove Theorem \ref{thm:SE1} and then provide a proof of Proposition \ref{inverse_possible}.

\subsection{Proving Theorem \ref{thm:SE1}}
\begin{proof}[Proof of Theorem \ref{thm:SE1}]	
To begin with, we prove that $\F(\tau^2,\bfalph \tau)$ defined in \eqref{eq:SE_F}
is concave with respect to $\tau^2$.  The proof follows along the same lines as the proof of \cite[Proposition 1.3]{lassorisk}, however, whereas the proof of \cite[Proposition 1.3]{lassorisk} proceeds by explicitly expressing the first derivative of the corresponding function $\F$,  and then differentiating on the explicit form to get the second derivative, in SLOPE case, because of the averaging that occurs within the proximal operation, it is extremely difficult to similarly  derive an explicit form.  To work around this, we keep all differentiation implicit. 
First,
\begin{align}
&\frac{\partial \F}{\partial \tau^2}(\tau^2, \bfalph \tau) =\frac{\partial }{\partial \tau^2}\big[\sigma_w^2+  \frac{1}{\delta \p}  \mathbb{E} \norm{  \prox_{J_{\bfalph \tau}}(\B + \tau \mathbf{Z}) - \B}^2\big] \overset{(a)}{=}  \frac{1}{\delta}  \mathbb{E}\big\{  \frac{\partial }{\partial \tau^2} \frac{1}{ \p} \norm{  \prox_{J_{\bfalph \tau}}(\B + \tau \mathbf{Z}) - \B}^2\big\} \nonumber \\
%
%
&= \frac{2}{\delta \p}   \sum_{i=1}^{\p} \mathbb{E}\big\{ \big( [\prox_{J_{\bfalph \tau}}(\B + \tau \mathbf{Z})]_i - B_i\big) \, \frac{\partial }{\partial \tau^2}  [\prox_{J_{\bfalph \tau}}(\B + \tau \mathbf{Z})]_i \big\}.
\label{eq:total_deriv_F_v1}
\end{align}
We note that the interchange between the derivative (a limit) and the expectation in step $(a)$ of the above holds due to a dominated convergence argument that relies on the following lemma.  First we introduce a bit of notation that will be used throughout the proof.  Define an equivalence classes $I_i$ for each index $i = \{1, 2, \ldots, \p\}$, defined as $$I_i:=\{j:|[   \prox_{J_{\bfalph \tau}}(\B + \tau \mathbf{Z}) ]_j|=|[  \prox_{J_{\bfalph \tau}}(\B + \tau \mathbf{Z})]_i|\}.$$  For any $j \in I_i$, with the above definition, $I_j = I_i$. In general, we use $I$, without any specific index, to represent an entire equivalence class and let $\mathsf{I}$ indicate the collection of unique equivalence classes.  
\begin{lemma}
%
\be
\Big \lvert \frac{\partial }{\partial \tau^2} \frac{1}{\p}\norm{  \prox_{J_{\bfalph \tau}}(\B + \tau \mathbf{Z}) - \B}^2 \Big \lvert \leq
\frac{1}{\p} \sum_{I \in \mathsf{I}} \frac{1}{|I|} \Big(\sum_{i\in I}  \lvert\sgn(B_i + \tau Z_i)Z_i - \alpha_i \lvert  \Big)^2.
\label{eq:dom_conv}
\ee
\label{lem:DC_res1}
\end{lemma}
Lemma \ref{lem:DC_res1} will be proved below, after we solve $\frac{\partial }{\partial \tau^2}  [\prox_{J_{\bfalph \tau}}(\B + \tau \mathbf{Z})]_i $. 	
	
Now we describe how the bound in Lemma \ref{lem:DC_res1} can be used to produce the dominated convergence result needed in step $(a)$ of \eqref{eq:total_deriv_F_v1}.  First note,
%
\begin{align*}
&\frac{1}{\p} \E \Big\{ \sum_{I \in \mathsf{I}} \frac{1}{|I|} \Big(\sum_{i\in I}  \lvert\sgn(B_i + \tau Z_i)Z_i - \alpha_i \lvert  \Big)^2\Big\} \leq \frac{1}{\p} \E \Big\{ \sum_{I \in \mathsf{I}}\sum_{i\in I}   \Big(  \lvert\sgn(B_i + \tau Z_i)Z_i - \alpha_i \lvert  \Big)^2\Big\} 
\\ 
&\leq\frac{2}{\p}\E \Big\{ \sum_{I \in \mathsf{I}} \sum_{i\in I} (Z_i^2 +\alpha_i^2)\Big\}  
=\frac{2}{\p}\E \Big\{ \sum_{i\in[\p]}(Z_i^2 +\alpha_i^2)\Big\}
=2+2\|\bfalph\|^2/p<\infty
\end{align*}
The first and second inequalities follow from $(\sum_{i=1}^n x_i)^2\leq n\sum_i x_i^2$. The last inequality comes from entries of $\bfalph$ being finite and then $\|\bfalph\|^2/p\leq \max_i \alpha_i^2<\infty$.
Therefore we can invoke the dominated convergence theorem that allows the exchange of the derivative and expectation in step $(a)$ of \eqref{eq:total_deriv_F_v1}.

Now we want to further simplify \eqref{eq:total_deriv_F_v1}.  For each $1 \leq i \leq \p$, we would like to study $\frac{\partial }{\partial \tau^2}  [\prox_{J_{\bfalph \tau}}(\B + \tau \mathbf{Z})]_i $.  We first note  that the mapping $\tau^2 \mapsto  [\prox_{J_{\bfalph \tau}}(\B + \tau \mathbf{Z})]_i $ can be considered as $f(g(\tau^2))$, where $g:\mathbb{R}\to\mathbb{R}^{2 \p}$ is defined as $y \mapsto g(y):=(\B + \mathbf{Z} \sqrt{y}, \bfalph \sqrt{y})$ and $f:\mathbb{R}^{2 \p} \to \mathbb{R}$ is defined as $(\mathbf{a}, \mathbf{b}) \mapsto f(\mathbf{a}, \mathbf{b}):=  [\prox_{J_{\mathbf{b}}}(\mathbf{a})]_i $.
Hence,
\begin{equation}
\begin{split}
\frac{\partial}{\partial \tau^2} [\prox_{J_{\bfalph \tau}}(\B + \tau \mathbf{Z})]_i  =\mathbf{J}_{f\circ g}(\tau^2) &\overset{(a)}{=}\mathbf{J}_f(g(\tau^2)) \mathbf{J}_g(\tau^2) 
= \Big[ \nabla_{\mathbf{a}} f(g(\tau^2)), \nabla_{\mathbf{b}} f(g(\tau^2))\Big] \Big[ \frac{\mathbf{Z}}{2\tau}, \frac{\bfalph}{2\tau}\Big]^\top,
\label{eq:total_deriv_F}
\end{split}
\end{equation}
where $\mathbf{J}_h\in\mathbb{R}^{m\times n}$ is the Jacobian matrix of a function $h:\mathbb{R}^n\to\mathbb{R}^m$ and step $(a)$ follows by the chain rule.    We denote the proximal operator using a function $\eta: \R^{2\p} \rightarrow \R^{\p}$ as $\eta(\mathbf{a}, \mathbf{b}) := \prox_{J_{\mathbf{b}}}( \mathbf{a})$ and consider the partial derivatives of $\eta$ with respect to its first and second arguments. Denote 
\be
\partial_1\eta(\mathbf{a}, \mathbf{b}) := \text{diag}\Big[\frac{\partial}{\partial a_1}, \frac{\partial}{\partial a_2}, \ldots, \frac{\partial}{\partial a_{\p}}\Big]\eta(\mathbf{a}, \mathbf{b}),
\,\,\, \text{ and } \,\,\, 
%
%
\partial_2\eta(\mathbf{a}, \mathbf{b}) := \text{diag}\Big[\frac{\partial}{\partial b_1}, \frac{\partial}{\partial b_2}, \ldots, \frac{\partial}{\partial b_{\p}}\Big]\eta(\mathbf{a}, \mathbf{b}).
\label{eq:eta_partials2}
\ee

Recall that the derivatives computed in $\partial_1\eta(\mathbf{a}, \mathbf{b})$ are defined in \eqref{eq:dif_prox},  
and by anti-symmetry between two arguments, $\frac{d}{d b_j}  [\eta(\mathbf{a}, \mathbf{b})]_i  = -\sgn([\eta(\mathbf{a}, \mathbf{b})]_j)\frac{d}{d a_j}  [\eta(\mathbf{a}, \mathbf{b})]_i$.  
Then using the result of  \eqref{eq:dif_prox}:
\begin{align*}
\frac{\partial [ \prox_{J_{\blam}}(\mathbf{v})]_i }{\partial v_j}
= \frac{\partial [ \eta(\mathbf{v}, \blam)]_i }{\partial v_j} 
&= \frac{\mathbb{I}\{ |[ \eta(\mathbf{v}, \blam)]_i| = |[ \eta(\mathbf{v}, \blam)]_j|\}\sgn([ \eta(\mathbf{v}, \blam)]_i[ \eta(\mathbf{v}, \blam)]_j) }{\text{\#\{$1\leq k\leq \p: | [\eta(\mathbf{v}, \blam)]_k  |=| [\eta(\mathbf{v}, \blam)]_i  |$\}}}
\end{align*}
we have 
\begin{align}
\frac{d}{d a_j} f(\mathbf{a}, \mathbf{b}) = \frac{d}{da_j}  [\eta(\mathbf{a}, \mathbf{b})]_i = \mathbb{I} \{ |[ \eta(\mathbf{a}, \mathbf{b})]_i| = |[ \eta(\mathbf{a}, \mathbf{b})]_j|\}\sgn([ \eta(\mathbf{a}, \mathbf{b})]_i [ \eta(\mathbf{a}, \mathbf{b})]_j)  [\partial_1 \eta(\mathbf{a}, \mathbf{b})]_i,
\label{eq:a_partial}	
\end{align}
and similarly,
\[ \frac{d}{d b_j} f(\mathbf{a}, \mathbf{b}) = \frac{d}{d b_j}  [\eta(\mathbf{a}, \mathbf{b})]_i  = - \mathbb{I}\big\{ |[ \eta(\mathbf{a}, \mathbf{b})]_i| =|[ \eta(\mathbf{a}, \mathbf{b})]_j|\big\}\sgn \big([ \eta(\mathbf{a}, \mathbf{b})]_i \big)    \big[\partial_1 \eta(\mathbf{a}, \mathbf{b})\big]_i.\]

Now plugging the above into \eqref{eq:total_deriv_F}, we have
\begin{equation}
\begin{split}
&\frac{\partial  }{\partial \tau^2}  [\prox_{J_{\bfalph \tau}}(\B + \tau \mathbf{Z})]_i  \\
&= \frac{1}{2\tau} \big[\partial_1 \eta(\B + \tau \mathbf{Z} , \bfalph \tau )\big]_i\sgn \big([ \eta(\B + \tau \mathbf{Z} , \bfalph \tau )]_i \big)\sum_{j\in I_i}\big(\sgn([ \eta(\B + \tau \mathbf{Z} , \bfalph \tau )]_j)Z_j-\alpha_j\big)
\label{eq:total_deriv_F_v2}
\end{split}
\end{equation}
In what follows, we drop the explicit statement of the $\eta(\cdot, \cdot)$ input to save space, writing $\eta_i$ to mean $[ \eta(\B + \tau \mathbf{Z} , \bfalph \tau )]_i$ or $[\partial_1 \eta]_i$ to mean $[\partial_1 \eta(\B + \tau \mathbf{Z} , \bfalph \tau )]_i$ for example. Using \eqref{eq:total_deriv_F_v2} in \eqref{eq:total_deriv_F_v1},
\begin{equation}
\begin{split}
&\frac{\partial \F}{\partial \tau^2}(\tau^2, \bfalph \tau) = \frac{1}{\delta \p \tau}   \sum_{i=1}^{\p}  \sum_{j\in I_i} \mathbb{E}\Big\{ ( \eta_i - B_i) \, [\partial_1 \eta]_i\sgn(\eta_i)(\sgn(\eta_j)Z_j-\alpha_j)\Big\} \\
&= \frac{1}{\delta \p }   \sum_{i=1}^{\p}  \sum_{j\in I_i}\E\Big\{([\partial_1\eta]_i)^2+(\eta_i-B_i)[\partial_1^2\eta]_i\Big\} -  \frac{1}{\delta \p \tau}   \sum_{i=1}^{\p}  \sum_{j\in I_i} \mathbb{E}\Big\{ ( \eta_i - B_i) \, [\partial_1 \eta]_i\sgn(\eta_i) \alpha_j \Big\}.
\label{eq:next_to_last}
\end{split}
\end{equation}
where the second equality follows by  Stein's lemma for a fixed $i$ and $j\in I_i$, namely, for standard Gaussian $Z$ we have $\E\{f(Z)Z\}=\E\{f'(Z)\}$ and therefore,
\begin{align*}
\frac{1}{\tau}\mathbb{E}\big\{[\partial_1 \eta]_i\sgn(\eta_i)( \eta_i - B_i ) \sgn(\eta_j)Z_j\big\} 
&= \mathbb{E}\big\{\sgn(\eta_i) \sgn(\eta_j) \big[( \eta_i - B_i )  \frac{d}{d a_j} [\partial_1 \eta]_i  + [\partial_1 \eta]_i  \frac{d}{d a_j} [\eta]_i  \big]\big \} \\
&= \mathbb{E}\big\{  ( \eta_i - B_i ) [\partial^2 _1 \eta]_i+ ([\partial_1 \eta]_i)^2\big \}.
\end{align*}
where the last step uses the definition of $\frac{d}{d a_j} [\eta(\bm a, \bm b)]_i$ given in \eqref{eq:a_partial} and the fact that $ \frac{d}{d a_j} [\partial_1 \eta(\bm a, \bm b)]_i = \sgn(\eta_i) \sgn(\eta_j) [\partial^2 _1 \eta(\bm a, \bm b)]_i$.

Therefore, simplifying \eqref{eq:next_to_last}, we have shown
\be
\begin{split}
\label{eq:first_partial}
(\delta \p \tau) \times \frac{\partial \F}{\partial \tau^2}(\tau^2, \bfalph \tau) 
&=\sum_{i=1}^{\p} \mathbb{E}\Big\{\tau|I_i|\Big([\partial_1 \eta]_i^2+(\eta_i-B_i)[\partial_1^2\eta]_i\Big)-[\partial_1 \eta]_i\sgn(\eta_i)( \eta_i - B_i )\sum_{j\in I_i}\alpha_j\Big\}.
\end{split}
\ee
%


We now have the tools to prove Lemma \ref{lem:DC_res1}.
\begin{proof}[Proof of Lemma \ref{lem:DC_res1}]
First,
\begin{equation*}
\begin{split}
\frac{\partial }{\partial \tau^2}  \frac{1}{\p} \norm{  \prox_{J_{\bfalph \tau}}(\B + \tau \mathbf{Z}) - \B}^2 
&= \frac{2}{ \p}   \sum_{i=1}^{\p} \Big( [\prox_{J_{\bfalph \tau}}(\B + \tau \mathbf{Z})]_i - B_i\Big) \, \frac{\partial }{\partial \tau^2}  [\prox_{J_{\bfalph \tau}}(\B + \tau \mathbf{Z})]_i .
\end{split}
\end{equation*}
As in the work above, we denote the proximal operator using a function $\eta: \R^{2\p} \rightarrow \R^{\p}$ as $\eta(\mathbf{a}, \mathbf{b}) := \prox_{J_{\mathbf{b}}}( \mathbf{a})$.  Now from \eqref{eq:total_deriv_F_v2}, denoting $I_i:=\{j:|[ \eta(\mathbf{a}, \mathbf{b})]_j|=|[ \eta(\mathbf{a}, \mathbf{b})]_i|\}$, again dropping the explicit statement of the $\eta(\cdot, \cdot)$ input to save space,
\begin{equation*}
\begin{split}
&\frac{\partial   }{\partial \tau^2} [\prox_{J_{\bfalph \tau}}(\B + \tau \mathbf{Z})]_i  = \frac{1}{2\tau}[\partial_1 \eta]_i\sgn(\eta_i)\sum_{j\in I_i}(\sgn(\eta_j)Z_j-\alpha_j).
\end{split}
\end{equation*}
Therefore,
\begin{equation*}
\begin{split}
&\Big \lvert\frac{\partial }{\partial \tau^2}  \frac{1}{\p} \norm{  \prox_{J_{\bfalph \tau}}(\B + \tau \mathbf{Z}) - \B}^2\Big \lvert
=  \frac{1}{\tau \p} \Big \lvert   \sum_{i=1}^{\p} ( \eta_i - B_i) \,[\partial_1 \eta]_i \sgn(\eta_i) \sum_{j\in I_i}(\sgn(\eta_j)Z_j-\alpha_j)\Big \lvert.
\end{split}
\end{equation*}
Since the averaging operation reduces the dot product (meaning informally that for a vector $\bm v\in\mathbb{R}^p$, 
$(\text{mean}(\bm v),...,\text{mean}(\bm v))\cdot \bm v\leq \norm{\bm v}^2),$
we have for any $i \in \{1, 2, \ldots, \p\}$ that $[\eta(\B + \tau \mathbf{Z} , \bfalph \tau )]_i - B_i $ can be replaced with $ B_i + \tau Z_i -\sgn(\eta_i) \alpha_i \tau - B_i$.  Using this in the above,
\begin{equation}
\begin{split}
\Big \lvert \frac{\partial }{\partial \tau^2}  \frac{1}{\p} \norm{  \prox_{J_{\bfalph \tau}}(\B + \tau \mathbf{Z}) - \B}^2 \Big \lvert
&\leq  \frac{1}{\p} \Big \lvert    \sum_{i=1}^{\p}  \sum_{j\in I_i}( Z_i -\sgn(\eta_i)\alpha_i) \,[\partial_1 \eta]_i \sgn(\eta_i)(\sgn(\eta_j)Z_j-\alpha_j)\Big \lvert\\
&= \frac{1}{\p} \Big \lvert  \sum_{i=1}^{\p}  \sum_{j\in I_i}( \sgn(\eta_i)Z_i - \alpha_i) (\sgn(\eta_j)Z_j-\alpha_j)\,[\partial_1 \eta]_i \Big \lvert.
\label{eq:lemma_eq1}
\end{split}
\end{equation}
Next, using that $0 \leq |[\partial_1\eta]_i|\leq {1}/{|I_i|}$,
\begin{align*}
&\Big \lvert  \sum_{i=1}^{\p}  \sum_{j\in I_i}( \sgn(\eta_i)Z_i - \alpha_i) (\sgn(\eta_j)Z_j-\alpha_j)\,[\partial_1 \eta]_i \Big \lvert 
 \leq  \sum_{i=1}^{\p} \frac{1}{|I_i|}  \sum_{j\in I_i} \Big \lvert  ( \sgn(\eta_i)Z_i - \alpha_i) (\sgn(\eta_j)Z_j-\alpha_j) \Big \lvert.
\end{align*}
Finally we make the following observation.  Any equivalence class $I_i$ is a collection of indices $j \in \{1, 2, \ldots, \p\}$ such that $|[   \prox_{J_{\bfalph \tau}}(\B + \tau \mathbf{Z}) ]_j|=|[  \prox_{J_{\bfalph \tau}}(\B + \tau \mathbf{Z})]_i|$, so for any $j \in I_i$, it follows $I_j = I_i$.  Recall, $\mathsf{I}$ indicates the collection of unique equivalence classes, and we have
\begin{align*}
&\sum_{i=1}^{\p} \frac{1}{|I_i|}  \sum_{j\in I_i} \Big \lvert  ( \sgn(\eta_i)Z_i - \alpha_i) (\sgn(\eta_j)Z_j-\alpha_j) \Big \lvert 
 = \sum_{I \in \mathsf{I}}  \frac{1}{|I|}  \sum_{i, j\in I}  \Big \lvert  ( \sgn(\eta_i)Z_i - \alpha_i) (\sgn(\eta_j)Z_j-\alpha_j) \Big \lvert.
\end{align*}
Now plugging back into \eqref{eq:lemma_eq1},
\begin{equation*}
\begin{split}
 \Big \lvert \frac{\partial }{\partial \tau^2}  \frac{1}{\p} \norm{  \prox_{J_{\bfalph \tau}}(\B + \tau \mathbf{Z}) - \B}^2 \Big \lvert 
&\leq  \frac{1}{\p} \sum_{I \in \mathsf{I}}  \frac{1}{|I|}  \sum_{i ,j\in I}  \Big \lvert  ( \sgn(\eta_i)Z_i - \alpha_i) (\sgn(\eta_j)Z_j-\alpha_j) \Big \lvert\\
&=  \frac{1}{\p} \sum_{I \in \mathsf{I}}   \frac{1}{|I|} \Big(\sum_{j\in I} \lvert  \sgn(\eta_j)Z_j - \alpha_j \lvert  \Big)^2.
\end{split}
\end{equation*}
\end{proof}


Now considering \eqref{eq:first_partial}, for simplicity in our future calculations, we suppress $|I_i|$ to 1 without loss of generality. To see this, recall that $I_i:=\{j:|[ \eta(\B + \tau \mathbf{Z} , \bfalph \tau )]_j|=|[ \eta(\B + \tau \mathbf{Z} , \bfalph \tau )]_i|\}$ and note that when $|[ \eta(\B + \tau \mathbf{Z} , \bfalph \tau )]_j$ equals $|[ \eta(\B + \tau \mathbf{Z} , \bfalph \tau )]_i$, the terms will remain equal after small changes in $\tau$.  Therefore $|I_i|$ is treated as a constant in the derivative and since all operations below preserves linearity, it can safely be assumed to be equal to $1$.   Note that similarly, $\sum_{j\in I_i}\alpha_j$, will pass through future calculations as a constant. 
Therefore \eqref{eq:first_partial} becomes
\begin{equation}
\begin{split}
&(\delta \p \tau) \times \frac{\partial \F}{\partial \tau^2}(\tau^2, \bfalph \tau)
=   \sum_{i=1}^{\p} \Big[\mathbb{E}\big\{ \tau  ( [\partial_1 \eta]_i)^2 +  \tau  ( \eta_i  - B_i )  [\partial_1^2 \eta]_i  -   \alpha_i \sgn( \eta_i)( \eta_i  - B_i )   [\partial_1 \eta]_i \big\}\Big].
	\end{split}
	\label{eq:after_suppress}
	\end{equation}
In what follows we will need to take care with the points $(\mathbf{x}, \mathbf{y})$ such that $[\partial_1^2\eta(\mathbf{x}, \mathbf{y})]_i$ is not equal to $0$. We refer to such points as `kink' points, since these are points where the partial derivative jumps (and the second partial gradient acts like Dirac delta function $\delta(x)$), or in other words the points where the two (sorted, averaged) arguments in $\eta$ are equal to each other.  Informally, define a `kink' point as an index where the sorted vector $\mathbf{x}$ matches the corresponding threshold $\mathbf{y}$ exactly.  In LASSO, for example, the correspond to the `kinks' of the soft-thresholding function. We have
\begin{align}
[\partial_1^2 \eta(\B + \tau \mathbf{Z} , \bfalph \tau )]_i=\delta(B_i+\tau Z_i-\alpha_i\tau)-\delta(B_i+\tau Z_i+\alpha_i\tau)
\label{eta'' Dirac}
\end{align}
and
\begin{equation}
\begin{split}
&\E_{\mathbf{Z},\B}\Big\{( [\eta(\B + \tau \mathbf{Z} , \bfalph \tau)]_i  - B_i )  [\partial_1^2 \eta(\B + \tau \mathbf{Z} , \bfalph \tau )]_i\Big\} \\
&=- \E_{\B}\E_{\mathbf{Z|\B}}\big\{B_i\big[\delta(B_i+\tau Z_i-\alpha_i\tau)-\delta(B_i+\tau Z_i+\alpha_i\tau)\big]\big\}
\\
&=-\frac{1}{\tau}\E_{\B}\big\{B_i\big[\phi(\alpha_i-\frac{1}{\tau}B_i)-\phi(-\alpha_i-\frac{1}{\tau}B_i)\big]\big\}.
\end{split}
\label{(eta-B)eta''}
\end{equation}
Therefore, denoting $\odot$ as elementwise multiplication of vectors, by \eqref{eq:after_suppress} and \eqref{(eta-B)eta''},
	\begin{equation}
	\begin{split}
	&(\delta \p \tau) \times \frac{\partial \F}{\partial \tau^2}(\tau^2, \bfalph \tau) 
	\\
	=& \tau  \mathbb{E}|| \partial_1 \eta||^2 -\E_{\B}\big\{\B^{\top}\big[\phi(\bfalph-\frac{1}{\tau}\B)-\phi(-\bfalph-\frac{1}{\tau}\B)\big]\big\} -  \mathbb{E}\big\{ \big[ \bfalph \odot \sgn( \eta) \odot  (\eta - \B)\big]^\top   \partial_1 \eta  \big\} .
	\label{eq:final_deriv}
	\end{split}
	\end{equation}
	Now we have shown the first derivative, so we consider the second derivative to prove concavity.  
	
	Notice, however, that in order to prove concavity of $\F(\tau^2, \bfalph \tau)$ it suffices to show $\frac{\partial }{\partial  \tau} [\frac{\partial  \F}{\partial  \tau^2}(\tau^2, \bfalph \tau)] \leq 0$ because $\frac{\partial }{\partial  \tau^2}(\frac{\partial  \F}{\partial  \tau^2}) = \frac{\partial  \tau }{\partial  \tau^2} [\frac{\partial }{\partial  \tau} (\frac{\partial \F}{\partial \tau^2})] =\frac{1}{2\tau} [\frac{\partial }{\partial  \tau} (\frac{\partial  \F}{\partial \tau^2})]$.

	
	We now show $\frac{\partial }{\partial  \tau} [\frac{\partial  \F}{\partial \tau^2}(\tau^2, \bfalph \tau)] \leq 0$.  First,
	\begin{equation}
	\begin{split}
	(\delta \p) \times  \frac{\partial}{\partial \tau}\Big[ \frac{\partial \F}{\partial \tau^2}(\tau^2, \bfalph \tau) \Big] 
	&=\frac{\partial}{\partial \tau} \mathbb{E}|| \partial_1 \eta||^2
	-\frac{\partial}{\partial \tau}\frac{1}{\tau}\E_{\B}\big\{\B^{\top}\big[\phi(\bfalph-\frac{1}{\tau}\B)-\phi(-\bfalph-\frac{1}{\tau}\B)\big]\big\}
	\\
	&\quad - \frac{\partial}{\partial \tau} \frac{1}{\tau}   \mathbb{E}\big\{ \big[ \bfalph \odot \sgn( \eta) \odot  (\eta - \B)\big]^\top   \partial_1 \eta  \big\} .
	\label{eq:final_deriv_new}
	\end{split}
	\end{equation}

	To show that \eqref{eq:final_deriv_new} is $\leq 0$, we  find simplified representations of the three terms on the right side.  This requires the same techniques as were used to find the first derivative above and so aren't given in full detail.

The first term on the right side of \eqref{eq:final_deriv_new} can be simplified to the following:
	\be
	\begin{split}
		& \frac{\partial}{\partial \tau} \mathbb{E}|| \partial_1 \eta||^2= - \frac{1}{\tau^2} \mathbb{E}_{\B}\big\{\B^\top  \big[\phi( \bfalph -  \frac{1}{\tau} \B) ) -\phi( \bfalph +  \frac{1}{\tau}\B) ) \big]\big\}.
		\label{eq:final_term1}
	\end{split}
	\ee

	Doing so requires smart uses of the chain rule, a dominated convergence argument, the partials in \eqref{eq:total_deriv_F_v2}, and special care for the `kink' points as discussed above.
	%
	%
	Similarly, using \eqref{(eta-B)eta''}, one can easily show for the third term on the right side of \eqref{eq:final_deriv_new},
	\be
	\begin{split}
		&\frac{\partial}{\partial \tau} \frac{1}{\tau}   \mathbb{E}\big\{ \big[ \bfalph \odot \sgn(\eta)  \odot  (\eta - \B)\big]^\top   \partial_1 \eta  \big\}  \geq  \frac{1}{\tau^3} \mathbb{E}_{\B}\Big\{ [ \bfalph \odot \B^2]^\top [ \phi(  \bfalph + \frac{1}{\tau}\B )+ \phi( \bfalph- \frac{1}{\tau}\B )] \Big\}.
		\label{eq:final_term3}
	\end{split}
	\ee
	Finally, using $\phi'(u)=-u\phi(u)$ and a dominated convergence argument, the second term on the right side of \eqref{eq:final_deriv_new} equals
	\be
	\begin{split}
	&-\frac{\partial}{\partial\tau}\frac{1}{\tau}\E_{\B}\big\{\B^{\top}\big[\phi(\bfalph-\frac{1}{\tau}\B)-\phi(-\bfalph-\frac{1}{\tau}\B)\big]\big\}	\\
	&\qquad  =\frac{1}{\tau^2}\E_{\B}\big\{\B^{\top}\big[\phi(\bfalph-\frac{1}{\tau}\B)-\phi(-\bfalph-\frac{1}{\tau}\B)\big]\big\}
\\
&\qquad  \qquad  -\frac{1}{\tau^3}\E_{\B}\big\{{(\B^2)}^{\top}\big[(\frac{1}{\tau}\B-\bfalph)\odot\phi(\bfalph-\frac{1}{\tau}\B)-(\bfalph+\frac{1}{\tau}\B)\odot\phi(-\bfalph-\frac{1}{\tau}\B)\big]\big\}.
	\end{split}
	\label{eq:final_term2}
	\ee

	
	Now we plug \eqref{eq:final_term1},\eqref{eq:final_term3}, and \eqref{eq:final_term2} back into \eqref{eq:final_deriv_new} to show that $\frac{\partial }{\partial  \tau} [\frac{\partial  \F}{\partial \tau^2}(\tau^2, \bfalph \tau)] \leq 0$.
	\begin{equation}
	\begin{split}
	&(\delta \p) \times  \frac{\partial}{\partial \tau}\Big[ \frac{\partial \F}{\partial \tau^2}(\tau^2, \bfalph \tau) \Big]
	\\
	&\leq -\frac{1}{\tau^2} \mathbb{E}_{\B}\big\{\B^\top  \big[\phi( \bfalph -   \B/\tau) ) -\phi( \bfalph +  \B/\tau) ) \big]\big\}
	+\frac{1}{\tau^2}\E_{\B}\big\{\B^{\top}\big[\phi(\bfalph-\B/\tau)-\phi(-\bfalph-\B/\tau)\big]\big\}
	\\
	&\qquad -\frac{1}{\tau^3}\E_{\B}\big\{{(\B^2)}^{\top}\big[(\B/\tau-\bfalph)\odot\phi(\bfalph-\B/\tau)-(\bfalph+\B/\tau)\odot\phi(-\bfalph-\B/\tau)\big]\big\}
	\\&
	\qquad - \frac{1}{\tau^3} \mathbb{E}_{\B}\Big\{ [ \bfalph \odot \B^2]^\top [ \phi(  \bfalph + \B/\tau )+ \phi( \bfalph- \B/\tau )] \Big\}
	\\
	&= - \frac{1}{\tau^4} \mathbb{E}_{\B}\big\{[\B^3]^\top \big[\phi( \bfalph - \B/\tau ) -\phi( \bfalph +\B/\tau ) \big]\big\}.
	\label{eq:final_deriv_overall}
	\end{split}
	\end{equation}
	We justify non-positivity of \eqref{eq:final_deriv_overall} by showing that the  elementwise term inside the expectation is less than or equal to $0$.   First assume $B_i \geq 0$, then $\alpha_i - B_i/\tau \leq \alpha_i +B_i/\tau$ and $\phi( \alpha_i - B_i/\tau )\geq \phi( \alpha_i +B_i/\tau )$. The other case $B_i \leq 0$ follows similarly.

	Now \eqref{eq:final_deriv_overall}, implies
	$\frac{\partial}{\partial \tau}\big[ \frac{\partial \F}{\partial \tau^2}(\tau^2, \bfalph \tau) \big] \leq 0$ and therefore, we have shown that $\F(\tau^2,\bfalph \tau)$ defined in \eqref{eq:SE_F}, is concave with respect to $\tau^2$.
	
	
	Next we show that $\tau^2\mapsto \F(\tau^2,\bfalph\tau)$ is strictly increasing.   To do so, it is sufficient to show that  $\frac{\partial \F}{\partial\tau^2} (\tau^2, \bfalph \tau)$ is positive as $\tau \rightarrow \infty$ because the concavity implies that $\frac{\partial \F}{\partial\tau^2} (\tau^2, \bfalph \tau)$ is non-increasing. 
Define $f(\bfalph):=\delta \lim_{\tau\to\infty} \frac{\partial \F}{\partial\tau^2} (\tau^2, \bfalph \tau)$.   First recall that $ \frac{\partial \F}{\partial\tau^2} (\tau^2, \bfalph \tau)$ is given in \eqref{eq:first_partial}.  In particular,
\be
\begin{split}
&\delta \frac{\partial \F}{\partial \tau^2}(\tau^2, \bfalph \tau) 
=\frac{1}{\p} \sum_{i=1}^{\p} \mathbb{E}\Big\{|I_i|\Big([\partial_1 \eta]_i^2+(\eta_i-B_i)[\partial_1^2\eta]_i\Big)- \frac{1}{\tau}[\partial_1 \eta]_i\sgn(\eta_i)( \eta_i - B_i )\sum_{j\in I_i}\alpha_j\Big\},
\label{eq:new_Fpartial}
\end{split}
\ee
Then taking $\tau\rightarrow \infty$  in the above, it is easy to see that  $f(\bfalph)$ is equivalent to setting $\B=\mathbf{0}$ in $\eta(\B + \tau \mathbf{Z} , \bfalph \tau )$ and using that $\eta(\tau \mathbf{Z} , \bfalph \tau )= \tau \eta(\mathbf{Z} , \bfalph)$ (implying that $\partial_1 \eta(\tau \mathbf{Z} , \bfalph \tau )=\partial_1 \eta(\mathbf{Z} , \bfalph)$).  We note that using a simplification of $[\partial_1^2\eta]_i$ as in \eqref{eta'' Dirac}-\eqref{(eta-B)eta''}, means that this term will go to zero as $\tau \rightarrow \infty$.
Therefore, using $\sgn(\eta( \mathbf{Z} , \bfalph  )) \odot  \eta(  \mathbf{Z} , \bfalph )  = | \eta(  \mathbf{Z} , \bfalph ) |$, 
	\begin{equation*}
	\begin{split}
	&f(\bfalph) = \frac{1}{\p}  \sum_{i=1}^{\p} \mathbb{E}\Big\{[D(\eta(\mathbf{Z},\bfalph))]_i ([\partial_1 \eta( \mathbf{Z} , \bfalph  )]_i)^2- [\partial_1 \eta( \mathbf{Z} , \bfalph  )]_i \lvert [\eta( \mathbf{Z} , \bfalph  )]_i \lvert \sum_{j : |[\eta(\mathbf{Z},\bfalph)]_j|=|[\eta(\mathbf{Z},\bfalph)]_i|}\alpha_j\Big\}.
	\end{split}
	\end{equation*}
In the above we have used the following definition: for a vector $\bm v \in \R^{\p}$, define $\bm D$ elementwise as $[\bm D(\bm v)]_i:=\#\{j : |v_j|=|v_i|\}=|I_i|$ if  $v_i\neq 0$ and $\infty$ otherwise.  Using that $\partial_1 \eta(\mathbf{Z} , \bfalph)=\frac{1}{\bm D(\eta(\mathbf{Z},\bfalph))}$, 
	\begin{equation}
	\begin{split}
	f(\bfalph)
	&=   \frac{1}{\p}  \sum_{i=1}^{\p} \E \Big \{\Big(1- |[\eta(\mathbf{Z},\bfalph)]_i|\sum_{j : |[\eta(\mathbf{Z},\bfalph)]_j|=|[\eta(\mathbf{Z},\bfalph)]_i|}\alpha_j \Big)\frac{1}{\bm [D(\eta(\mathbf{Z},\bfalph))]_i}\Big\} 
	\label{eq:f_alpha}
	\end{split}
	\end{equation}
	This simplification can be efficiently computed because only $|\eta(\mathbf{Z},\bfalph)|$ and $\bfalph$ need to be memorized. 
	
	Now considering \eqref{eq:f_alpha}, let $\bfalph \rightarrow \infty$ and note that since $|\mathbf{Z}| < \bfalph$ almost surely as $\bfalph \to \infty$, it follows that $\eta(  \mathbf{Z} , \bfalph )  = \partial_1  \eta(  \mathbf{Z} , \bfalph ) = \mathbf{0}$.  Therefore $
	\lim_{\bfalph \rightarrow \infty} f(\bfalph)  = 0$.  By a very similar argument to the proof of concavity, it is easy to see $f'(\bfalph)<0$, and together these facts imply $f(\bfalph)>0$ for all $\bfalph$.  The monotonicity of $\F$ is now obvious: since $\F$ is concave (implying $ \frac{\partial \F}{\partial\tau^2} (\tau^2, \bfalph \tau)$ is non-increasing) and strictly increasing for $\tau^2$ large enough, it is increasing everywhere.  Moreover, the monotonicity of $\F$ implies the monotonicity of the sequence $\{\tau_t^2(\p)\}_{t \geq 0}$.
	
	Finally we show that there exists a unique $\tau_*$ such that $\F (\tau_*^2, \bfalph \tau_*) = \tau^2_*$, from which it follows that the monotone sequence $\{\tau_t^2(\p)\}_{t \geq 0}$ converges to $\tau_*^2(\p)$ as $t \rightarrow \infty$.  First, by \eqref{eq:f_alpha}, we know
		$f(\mathbf{0})=  \mathbb{E}\norm{   \partial_1 \eta(\tau \mathbf{Z} ,\mathbf{0}) }^2/\p =   \mathbb{E}\norm{  \bm 1}^2/\p  =1.$
	%
This, along with the fact that $f'(\bfalph)<0$, tells us that $0 < f(\bfalph) < 1$ for all $\bfalph$.  Recall the definition of the set $\bm A_{\min}$, namely $\bm A_{\min}:=\{\bfalph: f(\bfalph)=\delta\}$.  We know that this set is non-empty since the LASSO case shows $\bfalph=(\alpha_{\min},\cdots,\alpha_{\min})$ belongs to $\bm A_{\min}$ where $\alpha_{\min}$ is the unique non-negative solution of $(1+\alpha^2)\Phi(-\alpha)-\alpha\phi(\alpha)=\delta/2$.
We write $\bfalph \succeq \bm A_{\min}$ to mean $\bfalph$ is larger than at least one element in $\bm A_{\min}$, where we consider one vector $\bm v$ to be larger than another vector $\bm u$ if $v_i\geq u_i$ for all $i$ and $v_j>u_j$ for some $j$.

To complete the proof, we show that $\F(\tau^2, \bfalph \tau) > \tau^2$ for small enough $\tau^2$ and $\F(\tau^2, \bfalph \tau) < \tau^2$ for large enough $\tau^2$. Therefore, there is at least one $\tau_*$ such that $\F(\tau_*^2, \bfalph \tau_*) = \tau_*^2$ since $\F$ is continuous in $\tau$. 
It follows from the concavity of $\F$ that the solution is unique and the sequence of iterates $\tau_t^2(\p)$ converge to $\tau_*^2(\p)$.
We first show that $\F(\tau^2, \bfalph \tau) > \tau^2$ for small enough $\tau^2$.  Consider the function $G(\tau^2) := \F(\tau^2, \bfalph \tau) - \tau^2$. Recalling the definition of  $\F(\tau^2,\bfalph \tau)$ in \eqref{eq:SE_F},
 namely,
	$
	\F(\tau^2, \bfalph \tau) = \sigma_w^2+   \mathbb{E} \norm{  \prox_{J_{\bfalph \tau}}(\B + \tau \mathbf{Z}) - \B}^2/(\delta \p),
	$
clearly $\F(0,\bm 0)=\sigma_w^2\geq 0$ and therefore $G(0) =\sigma_w^2\geq 0$ (with equality only if $\sigma_w^2= 0$).
Now we show that $\F(\tau^2, \bfalph \tau) < \tau^2$ for large enough $\tau^2$. Since $f(\bfalph)$ is decreasing in $\bfalph$, for $\bfalph \succeq \bm A_{\min}$, it must be that $f(\bfalph) < \delta$.  Moreover, $\lim_{\tau\to\infty} \frac{\partial \F}{\partial\tau^2} (\tau^2, \bfalph \tau) = \frac{1}{\delta} f(\bfalph) \leq 1$ for $\bfalph \succeq \bm A_{\min}$. Therefore, $\lim_{\tau\to\infty} \frac{\partial G}{\partial\tau^2}(\tau^2) \leq 0$ meaning $G$ is eventually decreasing (as $\tau^2$ grows) for any $\bfalph \succeq \bm A_{\min}$.  Also, $G(\tau^2)$ is concave and therefore for $\tau^2$ large enough we will have $G(\tau^2) < 0$, in which case $\F(\tau^2, \bfalph \tau) < \tau^2$.

Finally, $\big \lvert \frac{\partial \F}{\partial\tau^2}(\tau^2,\bfalph\tau) \big \lvert $ evaluated at at $\tau^2 = \tau_*^2$ is upper bounded by $1$ when $\bfalph \succeq \bm A_{\min}$, as the concavity of $\F$ implies that $ \frac{\partial \F}{\partial\tau^2}(\tau^2,\bfalph\tau)$ is strictly decreasing in $\tau^2$ along with $\lim_{\tau\to\infty} \frac{\partial \F}{\partial\tau^2} (\tau^2, \bfalph \tau) = \frac{1}{\delta} f(\bfalph) \leq 1$ when $\bfalph \succeq \bm A_{\min}$.  If this were not the case then there would be multiple fixed points.

\end{proof}

\subsection{Proving Proposition \ref{inverse_possible}}

\begin{proof}[Proof of Proposition \ref{inverse_possible}]
	
This proof is a generalized result of \cite[Proposition 1.4]{lassorisk} (originally proved in \cite{donoho2011noise}) and \cite[Corollary 1.7]{lassorisk}. Here we fixed $\p$ and denote $\tau(\p)$ as $\tau$.  

Recall in the proof of Theorem \ref{thm:SE1} we have shown the following facts: \textbf{(A)} $0 < \lim_{\tau^2\goto\infty}\frac{\partial \F}{\partial \tau^2} (\tau^2,\bfalph\tau)<1$; \textbf{(B)} $\tau^2\mapsto \F(\tau^2,\bfalph\tau)$ is concave; \textbf{(C)} $\tau^2\mapsto \F(\tau^2,\bfalph\tau)$ is strictly increasing; and \textbf{(D)} $ \frac{\partial \F}{\partial \tau^2}(\tau^2,\bfalph\tau)$ evaluated at $\tau = \tau_*$, which we denote $ \frac{\partial \F}{\partial \tau^2}(\tau_*^2,\bfalph\tau_*)$ is such that $0 < \frac{\partial \F}{\partial \tau^2}(\tau_*^2,\bfalph\tau_*) < 1$.

First we claim $\bfalph\mapsto\tau_*^2(\bfalph)$ is continuously differentiable on $\mathbb{R}_{+}^p$.
This follows from the implicit function theorem on function $G(\bfalph, \tau^2):=\tau^2-\F(\tau^2,\bfalph \tau)$ and from Fact \textbf{(D)}: $G$ is continuously differentiable and $0<\frac{\partial G}{\partial \tau^2}<1$. Hence $\tau^2$ can be written as $\tau^2(\bfalph)$ which is continuously differentiable.
Defining
$
	g(\bfalph,\tau^2):= \bfalph\tau\big[1 -
	\frac{1}{\n} \E \norm{  \prox_{J_{\bfalph \tau}}(\B + \tau \mathbf{Z})}_0^*\big],
$
notice that $\blam(\bfalph)=g(\bfalph,\tau_*^2(\bfalph))$. Clearly $g$ is continuously differentiable in $\bfalph$ and so is $\bfalph\mapsto\blam(\bfalph)$.

In the next step, we consider $\bfalph \succeq \bm A_{\min}(\delta)$ such that $\bfalph \goto \bm a_{\min}$ for some $\bm a_{\min} \in \bm A_{\min}(\delta)$ (denote as $\bfalph\downarrow \bm A_{\min}(\delta)$). We claim $\tau_*^2(\bfalph)\to +\infty$ as $\bfalph\downarrow \bm A_{\min}(\delta)$. Recall, $f(\bfalph):=\delta \lim_{\tau\to\infty} \frac{\partial \F}{\partial\tau^2} (\tau^2, \bfalph \tau)$ (cf.\ Theorem \ref{thm:SE1}).  Then by concavity of $\F(\tau^2,\bfalph \tau)$ in $\tau$,
\begin{align*}
&\tau_{*}^2 = \F(\tau_{*}^2,\bfalph \tau_*)\geq \F(0,\bm 0)+\tau_*^2 \lim_{\tau^2\to\infty}\frac{\partial \F}{\partial \tau^2}(\tau^2,\bfalph \tau) =\F(0,\bm 0)+\frac{1}{\delta}\tau_*^2 f(\bfalph) \quad  \Rightarrow  \quad \tau_*^2\geq \frac{\F(0,\bm 0)}{1-f(\bfalph)/\delta}
\end{align*}
Recall $\F(0,\bm 0)=\sigma_w^2$ and $f(\bm a_{\min})=\delta$ for any $\bm a_{\min} \in \bm A_{\min}(\delta)$. Hence $\tau_*^2(\bfalph)\to +\infty$ as $\bfalph\downarrow \bm A_{\min}(\delta)$.

Define $\ell(\bfalph):= 1 - \frac{1}{n}\E\| \prox_{J_{\bfalph \tau_*}}(\mathbf{B}+\tau_* \mathbf{Z})\|_0^*$. Then when $\tau_*^2(\bfalph)\to +\infty$ as $\bfalph\downarrow \bm A_{\min}(\delta)$,
\begin{eqnarray*}
	\ell_*:=\lim_{\bfalph\to \bm a_{\min}} \ell(\bfalph)= \lim_{\bfalph\to \bm a_{\min}}  \Big(1 - \frac{1}{n}\E\| \prox_{J_{\bfalph \tau_*}}(\tau_* \mathbf{Z})\|_0^*\Big) = 1-\frac{1}{n}\E\| \prox_{J_{\bm a_{\min}}}(\mathbf{Z})\|_0^*
	\, .
\end{eqnarray*}
We claim that $\ell_*<0$. Using the definition of the vector $\bm D$ and the set $\bm A_{\min}(\delta)$ in \eqref{eq:littlef_def},
\begin{align*}
\ell_*&=1-\frac{1}{n}\E\| \prox_{J_{\bm a_{\min}}}(\mathbf{Z})\|_0^*=1-\frac{1}{\delta}\E\Big\langle\frac{1}{\bm D(\prox_{J_{\bm a_{\min}}}(\mathbf{Z}))}\Big \rangle\\
&< 1-\frac{1}{\delta \p}\sum_i\E\Big\{\frac{1}{[\bm D(\prox_{J_{\bm a_{\min}}}(\mathbf{Z}))]_i}\Big(1-\sum_{j\in I_i}[\bm a_{\min}]_j\cdot  |[\prox_{J_{\bm a_{\min}}}(\mathbf{Z})]_i|\Big)\Big\}=0,
\end{align*} 
where (writing $\bm\eta$ to mean $\prox_{J_{\bm a_{\min}}}(\mathbf{Z})$ and $\bm D$ to mean $\bm D(\eta)$) the inequality in the above uses the fact that
\begin{align*}
&\frac{1}{\bm D_i} - \frac{1}{\bm D_i}\Big(1-\sum_{j\in I_i}[\bfalph_{\min}]_j |\bm\eta_i|\Big)=\frac{1}{\bm D_i}\sum_{j\in I_i}[\bfalph_{\min}]_j |\bm\eta_i|
\geq 0.
\end{align*}
Notice in the above, the equality only holds when $\bm\eta_i = 0$ but $\bm\eta\neq \bm 0$ almost surely. Therefore, using that $\blam(\bfalph)=g(\bfalph,\tau_*^2(\bfalph)) =  \bfalph \tau_*(\bfalph)\big[1 -
	\frac{1}{\n} \E \norm{  \prox_{J_{\bfalph \tau_*(\bfalph)}}(\B + \tau_*(\bfalph) \mathbf{Z})}_0^*\big],$
\begin{align}
	\lim_{\bfalph\downarrow \bm A_{\min}(\delta)}\blam(\bfalph) = \ell_*\cdot\lim_{\bfalph\downarrow \bm A_{\min}(\delta)}	\bfalph\tau_*(\bfalph)= -\infty\, .
	\label{range1}
\end{align}
Finally we consider the case $\bfalph\goto\infty$ and observe $\tau_*^2(\bfalph)\to \sigma_w^2+\E\{B^2\}/\delta$. To see this, notice that $\F(\tau^2,\bfalph \tau)\goto\sigma_w^2 +\E\{B^2\}/\delta$ as $\bfalph\goto\infty$ since $\tau_*^2(\bfalph) = \F(\tau_*^2(\bfalph),\bfalph \tau_*(\bfalph))$ is bounded above.
Moreover, since $\tau_*(\bfalph)$ is bounded, $\bfalph \tau_*(\bfalph)$ is unbounded as $\bfalph\goto\infty$ and we have $\lim_{\bfalph\to\infty}\ell(\bfalph)=1$ whence
\begin{align}
	\lim_{\bfalph\to\infty}\blam(\bfalph) = 1\cdot\lim_{\bfalph\to\infty}
	\bfalph\tau_*(\bfalph) = \infty\, .
	\label{range2}
\end{align}

We pause here to summarize that $\bfalph\mapsto\blam(\bfalph)$ is continuously differentiable on the domain $\{\bfalph: \bfalph\succeq \bm A_{\min}(\delta)\}$ with $\blam(\bm A_{\min}(\delta))=-\infty$ and $\lim_{\bfalph\to\infty}\blam(\bfalph)=+\infty$.
\\\\
Now to prove the inverse mapping $\blam\mapsto\bfalph(\blam)$ is continuous and non-decreasing when $\p\goto\infty$, we claim that the invertibility of $\bfalph\mapsto\blam(\bfalph)$ is sufficient.  Precisely, \textbf{(1)} invertibility implies strict monotonicity; \textbf{(2)} monotonicity plus \eqref{range1} and \eqref{range2} implies both $\bfalph\mapsto\blam(\bfalph)$ and $\blam\mapsto\bfalph(\blam)$ are increasing; and \textbf{(3)} continuity of $\bfalph\mapsto\blam(\bfalph)$ implies continuity of $\blam\mapsto\bfalph(\blam)$.

Now we prove the invertibility by contradiction. Assume that there are two distinct
such values $\bfalph_1$, $\bfalph_2$ satisfying $\widetilde{\blam}=\blam(\bfalph_1)=\blam(\bfalph_2)$. Apply Theorem \ref{thm:main_result3}
to both $\bfalph(\widetilde{\blam})=\bfalph_1, \bfalph_2$ with $\psi(\bm x,\bm y) = \langle(\bm x-\bm y)^2\rangle$. Then, together with Corollary \ref{corol1},
\begin{eqnarray*}
	\plim_{\p\to\infty}\|\hat{\bet}-\bet\|^2/\p =\plim_{\p\to\infty} \E\langle\|\prox_{J_{\bfalph\tau_*}}(\bet+\tau_* \mathbf{Z}\,;\,
	\bfalph\tau_*)-\bet\|_2^2\rangle = \delta(\tau_*^2-\sigma_w^2)\, .
\end{eqnarray*}
Since $\plim_{\p\to\infty}\|\hat{\bet}-\bet\|^2/\p$ is independent of $\bfalph$, the right side gives 
$
\tau_*(\bfalph_1) = \tau_*(\bfalph_2).
$
Next apply Theorem \ref{thm:main_result3} with $\psi(\bm x,\bm y) = \langle|\bm x|\rangle$, giving
$
	\plim_{\p\to\infty}\|\hat{\bet}\|_1/\p =\plim_{\p\to\infty} \E\langle\|\prox_{J_{\bfalph\tau_*}}(\bet+\tau_* \mathbf{Z}\,;\,
	\bfalph\tau_*)\|_1\rangle \, .
$
Obviously, for $\tau_*$ and $p$ fixed,
$\thet\mapsto\E\langle\|\prox_{J_{\bfalph\tau_*}}(\bet+\tau_* \mathbf{Z}\,;\,
\thet)\|_1\rangle$
 is strictly decreasing in $\thet$. Therefore $\bfalph_1\tau_*(\bfalph_1)=\bfalph_2\tau_*(\bfalph_2)$ implying $\bfalph_1=\bfalph_2$, since $\tau_*(\bfalph_1) = \tau_*(\bfalph_2)$, which is a contradiction.
 
\end{proof}


\section{Verifying Properties \textbf{(P1)} and \textbf{(P2)}}\label{app_props}

In this appendix we demonstrate that  the properties \textbf{(P1)} and \textbf{(P2)} given in Section \ref{sec_mainproof} and relating to the denoiser $\eta_{\p}^t(\cdot)$ defined in \eqref{eq:eta_def} are true.

\begin{proof}[Verifying Properties \textbf{(P1)} and \textbf{(P2)}]
Property \textbf{(P1)} follows since $\eta_{\p}^t(\cdot) = \prox_{J_{\bfalph \tau_t}}(\cdot)$, as it is easy to show that proximal operators are Lipschitz continuous with Lipschitz constant one. Namely
\[||\eta_{\p}^t(\mathbf{v}_1) - \eta_{\p}^t(\mathbf{v}_2)|| = ||\prox_{J_{\bfalph \tau_t}}(\mathbf{v}_1) - \prox_{J_{\bfalph \tau_t}}(\mathbf{v}_2)|| \leq ||\mathbf{v}_1 - \mathbf{v}_2||.\]

Next we show that property \textbf{(P2)} is true.  We restate property \textbf{(P2)} for convenience: for any $s, t$ with $(\mathbf{Z}, \mathbf{Z}')$ a pair of length-$\p$ vectors such that $(Z_i, Z'_i)$ are i.id.\ $\sim \mathcal{N}(0, \mathbf{\Sigma})$ for $i \in [\p]$ where $\mathbf{\Sigma}$ is any $2 \times 2$ covariance matrix, the following limits exist and are finite.
\begin{align}
&\plim_{\p\to\infty}\frac{1}{\p} \norm{\bet}, \quad \qquad \plim_{\p\to\infty}\frac{1}{\p} \mathbb{E}_{\mathbf{Z}}[\bet^\top \eta_{\p}^t(\bet + \mathbf{Z})], \quad \qquad  \plim_{\p\to\infty}\frac{1}{\p} \mathbb{E}_{\mathbf{Z}, \mathbf{Z}'}[\eta_{\p}^s(\bet + \mathbf{Z}')^\top \eta_{\p}^t(\bet + \mathbf{Z})].  \label{eq:lim1}
\end{align}

We first note that the first limit in \eqref{eq:lim1} exists by Assumption \textbf{(A2)} and the strong law of large numbers.  We focus on the other two limits.  These results follow by \cite[Proposition 1]{SLOPEasymptotic} given in Lemma \ref{lem:yue_sep} and  the following lemma, which is a classic result in probability theory.

\begin{lemma}[Doob's $L^1$ maximal inequality, \cite{doob1953stochastic} Chapter VII, Theorem 3.4] \label{eq:maximal}
Let $ X_1, X_2, \dots, X_p $ be a sequence of nonnegative i.i.d.\ random variables such that $ \mathbb{E}[X_1\max\{0, \log(X_1)\}] < \infty $. Then,
$$
\mathbb{E}\Big[\sup_{p \geq 1} \Big\{\frac{1}{p}(X_1+X_2+\cdots+ X_p)\Big\}\Big] \leq \frac{e}{e-1}(1+\mathbb{E}[X_1\max\{0, \log(X_1)\}]).
$$
\end{lemma}

\begin{proof}
Let $ M_p = \frac{1}{p}(X_1+X_2+\cdots+ X_p) $. Then the sequence $ \{M_p\} $ is a submartingale and hence by Doob's maximal inequality,
$$
\mathbb{E}\Big[\sup_{p' \geq p \geq 1} M_p\Big] \leq \frac{e}{e-1}(1+\mathbb{E}[M_{p'}\max\{0, \log(M_{p'})\}]).
$$
Note the mapping $ x \mapsto x\max\{0, \log x \} $ is convex and hence $ \mathbb{E}[M_{p'}\max\{0, \log(M_{p'})\}]) \leq \mathbb{E}[X_1\max\{0, \log(X_1)\}] $.
The result follows by Fatou's lemma and by noting that $ \sup_{p' \geq p \geq 1} M_p \uparrow \sup_{p \geq 1} M_p $ as $ p' \rightarrow \infty $.
\end{proof}

Before we prove that the second and third limits in \eqref{eq:lim1} exist and are finite, we state one more result that will be helpful in the proof.  This result uses Lemma \ref{eq:maximal} along with a Dominated Convergence argument to study expectations taken with respect to $(\bm Z, \bm Z')$ like those in \eqref{eq:lim1}.
\begin{lemma}
\label{lem:expectations}
Consider a function $\psi_{\p}: \R^{\p} \times \R^{\p} \times  \R^{\p} \rightarrow \R$ such that for iterations $s, t \geq 0,$
\be
 \frac{1}{\p} \Big \lvert \psi_{\p}(\bet, \eta^s_{\p}( \bet + \bm Z), \eta^t_{\p}( \bet + \bm Z')) - \psi_{\p}(\bet, h^s( \bet + \bm Z), h^t( \bet + \bm Z')) \Big \lvert  \rightarrow 0, \quad \text{ as } \quad \p \rightarrow \infty,
 \label{eq:exp_assumption1}
 \ee
where $h^s, h^t$ are the unspecified functions of Lemma~\ref{lem:yue_sep}, and $(\bm Z, \bm Z')$ are independent Gaussian vectors having zero-mean and independent entries with finite variance.  Assume, for some constant $L > 0$ not depending on $\p$,
\be
 \frac{1}{\p} \Big \lvert  \psi_{\p}(\bet, \eta^s_{\p}( \bet + \bm Z), \eta^t_{\p}( \bet + \bm Z'))- \psi_{\p}(\bet, h^s( \bet + \bm Z), h^t( \bet + \bm Z')) \Big \lvert \leq L\Big(1 + \frac{\norm{\bet}^2}{\p} + \frac{\norm{\bm Z}^2}{\p} + \frac{\norm{\bm Z'}^2}{\p}\Big).
 \label{eq:exp_assumption2}
 \ee
Then, as $\p \rightarrow \infty$,
\be
 \frac{1}{p} \Big \lvert \E_{\bm Z, \bm Z'}\Big\{ \psi_{\p}(\bet, \eta^s_{\p}( \bet + \bm Z), \eta^t_{\p}( \bet + \bm Z)) \Big\} - \E_{\bm Z, \bm Z'} \Big\{\psi_{\p}(\bet, h^s( \bet + \bm Z), h^t( \bet + \bm Z'))\Big\} \Big \lvert  \rightarrow 0.
 \label{eq:exp_result1}
 \ee
\end{lemma}

\begin{proof}
We begin by showing that 
$\E_{\bm Z, \bm Z'}  \big\{\sup_{p\geq 1} \frac{1}{p} \big \lvert \psi_{\p}(\bet, \eta^s_{\p}( \bet + \bm Z), \eta^t_{\p}( \bet + \bm Z'))  \big \lvert \big\} < \infty.$ 
Using \eqref{eq:exp_assumption2}, it is clear that this expectation is finite almost surely if
\ben
\mathbb{E}\Big[\sup_{p \geq 1} \Big\{\frac{1}{p}\norm{\mathbf{Z}(\p)}^2\Big\}\Big] < \infty, \quad \mathbb{E}\Big[\sup_{p \geq 1} \Big\{\frac{1}{p}\norm{\mathbf{Z}'(\p)}^2\Big\}\Big] < \infty,  \quad \text{ and } \quad  \mathbb{E}\Big[\sup_{p \geq 1} \Big\{\frac{1}{p}\norm{\bet(\p)}^2\Big\}\Big] < \infty,
\een
where we have made the dependence of the vectors on the dimension $\p$ explicit.
But Lemma \ref{eq:maximal} immediately implies the above since $ \mathbb{E}[B^2\max\{0, \log B\}] < \infty $ by assumption \textbf{(A2)}.

Now by dominated convergence we have,
\begin{align*}
&\E_{\bm Z, \bm Z'}  \Big \{\plim_p \frac{1}{p}  \Big \lvert \psi_{\p}(\bet, \eta^s_{\p}( \bet + \bm Z), \eta^t_{\p}( \bet + \bm Z')) - \psi_{\p}(\bet, h^s( \bet + \bm Z), h^t( \bet + \bm Z')) \Big \lvert\Big\} \\
&= \plim_p \frac{1}{p} \E_{\bm Z, \bm Z'} \Big \lvert \psi_{\p}(\bet, \eta^s_{\p}( \bet + \bm Z), \eta^t_{\p}( \bet + \bm Z')) - \psi_{\p}(\bet, h^s( \bet + \bm Z), h^t( \bet + \bm Z')) \Big \lvert\\
&\geq \plim_p \frac{1}{p} \Big \lvert \E_{\bm Z, \bm Z'}\Big\{ \psi_{\p}(\bet, \eta^s_{\p}( \bet + \bm Z), \eta^t_{\p}( \bet + \bm Z)) \Big\} - \E_{\bm Z, \bm Z'} \Big\{\psi_{\p}(\bet, h^s( \bet + \bm Z), h^t( \bet + \bm Z'))\Big\} \Big \lvert.
\end{align*}
Then the above implies the desired result \eqref{eq:exp_result1} from assumption \eqref{eq:exp_assumption1}.
\end{proof}

First consider the second limit in \eqref{eq:lim1}. By Cauchy-Schwarz, \eqref{eq:yue_limit} of Lemma \ref{lem:yue_sep} implies that
$\big \lvert \bet^\top \eta^t_{\p}( \bet + \bm Z) - \bet^\top h^t( \bet + \bm Z) \big \lvert/\p  \rightarrow 0,$ as $\p \rightarrow \infty.$
This follows because
\[ \big \lvert \bet^\top \eta^t_{\p}( \bet + \bm Z) - \bet^\top h^t( \bet + \bm Z) \big \lvert/\p \leq \norm{\bet} \norm{\eta^t_{\p}( \bet + \bm Z) - h^t( \bet + \bm Z)}/\p.\]
Then the right side of the above $\goto 0$ with growing $\p$ because ${\norm{\bet}}/{\sqrt{\p}}$ limits to a constant as justified above (this is the limit in \eqref{eq:lim1}), and the other term $\goto 0$ by  \eqref{eq:yue_limit} of  Lemma \ref{lem:yue_sep}.  This means that assumption \eqref{eq:exp_assumption1} of Lemma~\ref{lem:expectations} is satisfied.  Assumption \eqref{eq:exp_assumption2} of Lemma~\ref{lem:expectations} is also satisfied since both $ \eta^t_{\p} $ and $ h^t $ are Lipschitz(1), by Cauchy-Schwarz inequality.
Therefore Lemma~\ref{lem:expectations} implies
$  \big \lvert  \E_{\bm Z} \{\bet^\top \eta^t_{\p}( \bet + \bm Z)\} -  \E_{\bm Z} \{\bet^\top h^t( \bet + \bm Z) \} \big \lvert/\p \rightarrow 0$, as $\p \rightarrow \infty.$
Therefore,
\begin{align*}
 \plim_{\p\to\infty} \mathbb{E}_{\mathbf{Z}}[\bet^\top\eta_{\p}^t(\bet + \mathbf{Z})]/\p
 &=\plim_{\p\to\infty} \sum_{i=1}^{\p}\beta_{0,i}  \mathbb{E}_{Z}\{h^t (\beta_{0,i}+Z_i)\}/\p= \mathbb{E}[B h^t(B+Z)],
\end{align*}
where $B, Z$ are univariate.
By the Cauchy-Schwarz inequality, $\E[Bh^t(B+Z)] < \infty$ if $\E[B^2]<\infty$ and $\E[h^t(B+Z)^2]<\infty$. Since $\E[B^2] = \sigma_{\bet}^2 <\infty$ is given by our assumption, it suffices to show $\E[h^t(B+Z)^2]<\infty$.  But this follows from the fact that $h^t(\cdot)$ is Lipschitz(1) and therefore $\E[h^t(B+Z)^2]<\E[(B+Z)^2] \leq \E[B^2] + \E[Z^2] = \sigma_{\bet}^2  + \Sigma_{11} < \infty.$

Finally consider the third limit in \eqref{eq:lim1}.  Similarly to the work in studying the second limit in \eqref{eq:lim1}, we will appeal to Lemma~\ref{lem:expectations}. First we will show that
\begin{align}
  \big \lvert \eta_{\p}^s(\bet + \mathbf{Z}')^\top \eta_{\p}^t(\bet + \mathbf{Z}) - h^s(\bet + \mathbf{Z}')^\top h^t(\bet + \mathbf{Z})  \big \lvert/\p  \rightarrow 0,  \quad \text{ as } \quad \p \rightarrow \infty,
 \label{eq:P2_res2}
 \end{align}
meaning that assumption \eqref{eq:exp_assumption1} of Lemma~\ref{lem:expectations} is satisfied.  Then, again, assumption \eqref{eq:exp_assumption2} of Lemma~\ref{lem:expectations} is satisfied since both $ \eta^t_{\p} (\cdot)$ and $ h^t (\cdot)$ are Lipschitz(1), using Cauchy-Schwarz. 

Now we want to prove \eqref{eq:P2_res2}.  
By repeated applications of Cauchy-Schwarz it is not hard to show,
\begin{align*}
&\plim_{\p} \big \lvert \eta_{\p}^s(\bet + \mathbf{Z}')^\top \eta_{\p}^t(\bet + \mathbf{Z}) - h^s(\bet + \mathbf{Z}')^\top h^t(\bet + \mathbf{Z})  \big \lvert  /\p \\
&\leq \plim_{\p} \norm{h^s(\bet + \mathbf{Z}')} \norm{ \eta_{\p}^t(\bet + \mathbf{Z})   -  h^t(\bet + \mathbf{Z})}/\p +  \plim_{\p} \norm{h^t(\bet + \mathbf{Z})} \norm{ \eta_{\p}^s(\bet + \mathbf{Z}')   -  h^s(\bet + \mathbf{Z}')}/\p \\
&\qquad + \plim_{\p} \norm{\eta_{\p}^s(\bet + \mathbf{Z}') - h^s(\bet + \mathbf{Z}')} \norm{\eta_{\p}^t(\bet + \mathbf{Z}) - h^t(\bet + \mathbf{Z})}/\p .
\end{align*}
Now, \eqref{eq:P2_res2} follows since the right side of the above goes to $0$ as $\p$ grows.  This follows since, by  \eqref{eq:yue_limit} of  Lemma \ref{lem:yue_sep}, as $\p \rightarrow \infty,$
\[\norm{\eta_{\p}^s(\bet + \mathbf{Z}') - h^s(\bet + \mathbf{Z}')} /\sqrt{\p} \rightarrow 0 \quad \text{ and } \quad \norm{ \eta_{\p}^t(\bet + \mathbf{Z})   -  h^t(\bet + \mathbf{Z})}/\sqrt{\p} \rightarrow 0.\]
Moreover, since $h^s(\cdot)$ and $h^t(\cdot)$ are separable, by the Law of Large Numbers,
\ben
\begin{split}
&\plim_{\p} \norm{h^s(\bet + \mathbf{Z}')}^2/\p = \plim_{\p}  \sum_{i=1}^{\p} [h^s(\beta_{i} + Z'_i)]^2/\p = \mathbb{E}[(h^s(B + Z'))^2]   < \infty, \\
&\plim_{\p} \norm{h^t(\bet + \mathbf{Z})}^2/\p= \plim_{\p}\sum_{i=1}^{\p} [h^t(\beta_{i} + Z_i)]^2/\p = \mathbb{E}[(h^t(B + Z))^2]   < \infty,
\end{split}
\een
where the inequalities follow since $\mathbb{E}[(h^s(B + Z'))^2] \leq \mathbb{E}[(B + Z')^2] \leq \sigma_{\bet}^2 + \Sigma_{22} < \infty$ and $\mathbb{E}[(h^t(B + Z))^2] \leq \mathbb{E}[(B + Z)^2] \leq \sigma_{\bet}^2 + \Sigma_{11} < \infty$.  This proves \eqref{eq:P2_res2} and therefore we can apply Lemma~\ref{lem:expectations}.

Then Lemma~\ref{lem:expectations} implies, 
\[ \big \lvert  \E_{\bm Z, \bm Z'}\{\eta_{\p}^s(\bet + \mathbf{Z}')^\top \eta_{\p}^t(\bet + \mathbf{Z})\} -  \E_{\bm Z, \bm Z'}\{h^s(\bet + \mathbf{Z}')^\top h^t(\bet + \mathbf{Z}) \} \big \lvert/\p \rightarrow 0,  \quad \text{ as } \quad \p \rightarrow \infty.\]
But now, using the above, we find that
\begin{align*}
 \plim_{\p\to\infty} \E_{\bm Z, \bm Z'}\{\eta_{\p}^s(\bet + \mathbf{Z}')^\top \eta_{\p}^t(\bet + \mathbf{Z})\}/\p 
 &=\plim_{\p\to\infty} \sum_{i=1}^{\p} \E_{\bm Z, \bm Z'}\{h^s(\beta_{i} + Z'_i) h^t(\beta_{i} + Z_i) \}/\p  \\
 = \mathbb{E}[h^s(B+Z') h^t(B+Z)],
\end{align*}
where $B, Z',$ and $Z$ are univariate and $\E[h^s(B+Z')h^t(B+Z)] < \infty$ by Cauchy-Schwarz and the fact that $h^s(\cdot)$ and $h^t(\cdot)$ are Lipschitz(1).  Namely, this gives the bound
\begin{align*}
&\Big(\E[h^s(B+Z')h^t(B+Z)]\Big)^2 \leq \E[(h^s(B+Z'))^2] \E[(h^t(B+Z))^2]  \leq \E[(B+Z')^2] \E[(B+Z)^2] \\
&=  ( \E[B^2]+ \E[Z'^2] )(\E[B^2]+ \E[Z^2] )= (\sigma_{\bet}^2+\Sigma_{22})(\sigma_{\bet}^2+\Sigma_{11}) <\infty.
\end{align*}
We have now shown that property \textbf{(P2)} is true.

\end{proof}


\section{Proof of Fact \ref{fact:limits} }\label{app_fact}
\begin{proof}
The fact follows from the asymptotic separability of the proximal operator \cite[Proposition 1]{SLOPEasymptotic} (restated in Lemma~\ref{lem:yue_sep}) and the dominated convergence theorem \cite{royden1968real} allowing for interchange of limit and expectation. We sketch the proof of the existence of the limit in \eqref{eq:SE2} (and the result for the limit in \eqref{eq:lambda_func} follows similarly).  By Lemma \ref{lem:yue_sep}, the weak convergence of $\bfalph(\p)$ to $A$, and the Weak Law of Large Numbers, one can argue that
	\begin{align}
	\lim_{\p}  \norm{  \prox_{J_{\bfalph(p) \tau_*}}(\mathbf{B} + \tau_* \mathbf{Z}) - \mathbf{B}}^2/(\delta \p) =\mathbb{E} \{( h(B + \tau_* Z) - B)^2\}/\delta,
	\end{align}
	where $h(\cdot) := h(\cdot; B + \tau_* Z, A \tau_*)$ is the unspecified, separable function of Lemma \ref{lem:yue_sep}. This is consistent with [Lemma 29, \cite{SLOPEasymptotic}]. The limit in \eqref{eq:SE2} exists if $\mathbb{E} \{( h(B + \tau_* Z) - B)^2\}/\delta<\infty$ and
	\begin{align*}
	&\mathbb{E} \{( h(B + \tau_* Z) - B)^2\} \leq 2\mathbb{E} \{ h(B + \tau_* Z)^2 + B^2 \} \leq 2\mathbb{E} \{(B + \tau_* Z)^2 + B^2\}
	\\
	&\leq 2\mathbb{E} \{ 2 B^2 + 2\tau_*^2  Z^2 + B^2\} =6\mathbb{E} \{B^2\}+4\tau_*^2<\infty.
	\end{align*}
	Here the first and third inequalities follow from $(x-y)^2\leq 2(x^2+y^2)$ and the second inequality follows from $h$ being Lipschitz(1): $|h(x)|=|h(x)-h(0)|\leq |x-0|=|x|$.
\end{proof}

 
\section{Proof of Lemma~\ref{lem:bounded_vals}} \label{app:bounded_vals}


\begin{proof}
First, the proof of \eqref{eq:res1} follows from Theorem~\ref{thm:main_result1}.  To see this, note that by \eqref{eq:AMP0}, we have
$\bet^{t+1}=\prox_{J_{\thet_t}}(\X^\top \z^t+\bet^t) =\eta^t_{\p}(\X^\top \z^t+\bet^t),$
and therefore we apply Theorem~\ref{thm:main_result1} with uniformly pseudo-Lipschitz function $\psi_{\p} (  \bet^{t} + \X^\top \z^t, \bet) = \norm{\eta^t_{\p} (\bet^{t} + \X^\top \z^t)}^2/\p$ to get
\begin{align}
\plim_{\p} \,\, \norm{ \bet^{t}}^2/\p \overset{p}{=}\plim_{\p} \,\, \mathbb{E}_{\mathbf{Z}}[  \norm{\eta^t_{\p} (  \bet + \tau_t \mathbf{Z})}^2]/\p ,
\label{eq:lim_1}
\end{align}
for $\mathbf{Z}  \sim \mathcal{N}(0, \mathbb{I}_{\p})$.  By the Lipschitz property of $\eta^t_{\p}$ (Assumption \textbf{(A4)}), we have
$\mathbb{E}_{\mathbf{Z}}[  \norm{\eta^t_{\p} (  \bet + \tau_t \mathbf{Z})}^2]  \leq \mathbb{E}_{\mathbf{Z}}[  \norm{  \bet + \tau_t \mathbf{Z}}^2] \leq 2 \norm{\bet}^2 + 2 \p \tau_t^2.$
Plugging into \eqref{eq:lim_1}, we find
$
%
\plim_{\p} \norm{ \bet^{t}}^2/\p \overset{p}{=} 2\plim_{\p}  \norm{\bet}^2/\p  + 2 \tau_t^2 =   2\sigma_{\bet}^2 + 2 \tau_t^2,
$
where the final inequality follows by Assumption \textbf{(A2)}.


Now consider the $\widehat{\bet}$ result in \eqref{eq:res2}. First, note that by definition $\mathcal{C}(\widehat{\bet}) \leq \mathcal{C}(\mathbf{0})$ where the cost function $\mathcal{C}(\cdot)$ is defined in \eqref{eq:SLOPE_est}.  Using that 
\be
\mathcal{C}(\mathbf{0}) = \frac{1}{2}\| \y \|^2 = \frac{1}{2}\| \X \bet + \w \|^2 \leq \| \X \bet\|^2 + \|\w \|^2 \leq \sigma^2_{\max}(\X) \| \bet\|^2 + \|\w \|^2,
\label{eq:C0_bound}
\ee
where $ \sigma_{\max}(\X)$ is the maximum singular value of $\X$.  We note that this value, $ \sigma_{\max}(\X)$, is bounded almost surely as $\p \rightarrow \infty$ using standard estimates on the singular values of random matrices since $\X$ has i.i.d.\ Gaussian entries by Assumption \textbf{(A1)} (see, for example, \cite[Lemma F.2]{nonseparable}).  Therefore,
\be
\plim_{\p} \, \mathcal{C}(\widehat{\bet})/\p  \leq \plim_{\p} \, \sigma^2_{\max}(\X) \| \bet\|^2/\p  + \plim_{\p} \,  \|\w \|^2/\p  \leq \textsf{B}_{max} \sigma_{\bet}^2 + \sigma_{w}^2,
\label{eq:C0_limit}
\ee
where we've defined $\textsf{B}_{max}$ to be a bound on the limit of the maximum singular value, i.e.\ $ \lim_{\p} \sigma^2_{\max}(\X) \leq \textsf{B}_{max}$, and the final inequality holds by Assumptions \textbf{(A2)} and \textbf{(A3)}.


Now we will relate $\frac{1}{\p}\norm{\widehat{\bet}}^2$ to $\frac{1}{\p}\mathcal{C}(\widehat{\bet})$ and other terms lower-bounded by a  constant with high probability. We write $\widehat{\bet} =\widehat{\bet}^{\perp} + \widehat{\bet}^{\parallel}$ where $\widehat{\bet}^{\perp} \in ker(\X)^{\perp}$ and $\widehat{\bet}^{\parallel} \in ker(\X)$.   Since $\widehat{\bet}^{\parallel} \in \ker(\X)$ and $\ker(\X)$ is a random subspace of size $\p - n = \p(1 - \delta)$, by Kashin Theorem (Theorem H.\ref{thm:Kashin}.), we have that for some constant $\nu_1 = \nu_1(\delta)$, with high probability
\begin{equation}
 \|\widehat{\bet}^{\parallel}\|_2^2  \leq \nu_1 \|\widehat{\bet}^{\parallel}\|_1^2/\p .
\label{eq:kashin}
\end{equation}  
Then we have the following bound
\begin{equation}
\begin{split}
\norm{\widehat{\bet}}^2 = \norm{\widehat{\bet}^{\parallel}}^2 +  \norm{\widehat{\bet}^{\perp}}^2 &\overset{(a)}{\leq} \nu_1\norm{\widehat{\bet}^{\parallel}}_1^2/\p  + \norm{\widehat{\bet}^{\perp}}^2 \overset{(b)}{\leq}  2\nu_1  \norm{\widehat{\bet}}_1^2/\p  + (2 \nu_1 +1) \norm{\widehat{\bet}^{\perp}}^2, 
\label{eq:hatx_split}
\end{split}
\end{equation}
where step $(a)$ holds by \eqref{eq:kashin} and step $(a)$ by the Triangle Inequality and Cauchy-Schwarz as follows
\[\norm{\widehat{\bet}^{\parallel}}_1^2  =\norm{\widehat{\bet} - \widehat{\bet}^{\perp}}_1^2 \leq (\norm{\widehat{\bet}}_1 + \norm{\widehat{\bet}^{\perp}}_1)^2 \leq 2 \norm{\widehat{\bet}}_1^2 + 2 \norm{\widehat{\bet}^{\perp}}_1^2 \leq 2 \norm{\widehat{\bet}}_1^2 + 2 \p \norm{\widehat{\bet}^{\perp}}^2.\]
%
%
%
Now we bound the second term on the right side of \eqref{eq:hatx_split}. Define $\hat{\sigma}_{min}(\X)$ as the minimum non-zero singular value of $\X$.  By standard results in linear algebra, $\hat{\sigma}^2_{min}(\X) \norm{\widehat{\bet}^{\perp}}^2 \leq \norm{\X \widehat{\bet}^{\perp}}^2$.  Therefore,
\ben
\begin{split}
\hat{\sigma}^2_{min}(\X) \norm{\widehat{\bet}^{\perp}}^2 \leq \norm{\X \widehat{\bet}^{\perp}}^2  \leq   \norm{\X \widehat{\bet}^{\perp} - \y + \y}^2 \leq 2  \norm{\y - \X \widehat{\bet}^{\perp}}^2 +  2  \norm{\y}^2 &\leq 2\mathcal{C}(\widehat{\bet})  +  2\mathcal{C}(\mathbf{0}) \leq 2\mathcal{C}(\mathbf{0}).
%
\end{split}
\een
Therefore, using \eqref{eq:C0_bound} and \eqref{eq:C0_limit}, we have
\be
\begin{split}
\plim_p \frac{1}{\p} \norm{\widehat{\bet}^{\perp}}^2 \leq \plim_p \frac{\frac{2}{\p} \mathcal{C}(\mathbf{0})}{\hat{\sigma}^2_{min}(\X)} \leq \frac{2(\textsf{B}_{max} \sigma_{\bet}^2 + \sigma_{w}^2)}{ \textsf{B}_{min}}.
\label{eq:xperp_norm}
\end{split}
\ee
where we've defined $\textsf{B}_{min}$ to be a bound on the limit of the minimum non-zero singular value, i.e.\ $ \lim_{\p} \hat{\sigma}^2_{min}(\X) \geq \textsf{B}_{min}.$


Now we bound the first term on the right side of \eqref{eq:hatx_split}.  Recall the definition of the sort-ed $\ell_1$ norm, i.e.\ $J_{\blam}(\bm b) = \sum \lambda_i|\bm b|_{(i)}$, then using $\lambda_{min} := \lim_p \min(\blam)$ to lower bound the threshold values,
\[\lambda_{min} \norm{\widehat{\bet}}_1 =  \sum \lambda_{min} |\widehat{\bet}_{i}| =  \sum \lambda_{min} |\widehat{\bet}|_{(i)}  \leq  \sum \lambda_i|\widehat{\bet} |_{(i)} = J_{\blam}(\widehat{\bet}) \leq \mathcal{C}(\widehat{\bet}) \leq \mathcal{C}(\mathbf{0}). \]
Then, using \eqref{eq:C0_bound} and \eqref{eq:C0_limit}, we see
\be
\begin{split}
\plim_p \frac{1}{\p} \norm{\widehat{\bet}}_1 \leq \plim_p\frac{1}{ \lambda_{min}} \Big(\frac{1}{\p}\mathcal{C}(\mathbf{0}) \Big)\leq \frac{1}{ \lambda_{min}}(\textsf{B}_{max} \sigma_{\bet}^2 + \sigma_{w}^2).
\label{eq:xperp_norm}
\end{split}
\ee
By \eqref{eq:xperp_norm}, along with the upper bound in  \eqref{eq:hatx_split}, we have
\ben
\plim_p \frac{\norm{\widehat{\bet}}^2}{\p}   \leq  2\nu_1 \plim_p  \frac{ \norm{\widehat{\bet}}_1^2}{\p^2}   + (2 \nu_1 +1) \plim_p \frac{ \norm{\widehat{\bet}^{\perp}}^2}{\p}  \leq \Big[\frac{2\nu_1 (\textsf{B}_{max} \sigma_{\bet}^2 + \sigma_{w}^2)}{ \lambda_{min}}\Big]^2 +  \frac{2(2 \nu_1 +1)(\textsf{B}_{max} \sigma_{\bet}^2 + \sigma_{w}^2)}{ \textsf{B}_{min}}.
\een
\end{proof}


\section{Proof of Lemma \ref{lem:4.3}}\label{app_5743}

The proof of Lemma \ref{lem:4.3} relies on the following result, Lemma \ref{lem:5.7}, about the exponential rate of the convergence of the state evolution sequence defined in \eqref{eq:Sigma_def}.  We state and prove Lemma \ref{lem:5.7}, and Lemma \ref{lem:4.3} is proved afterward.
	
\begin{lemma}\label{lem:5.7}
	Assume $\bfalph>\bm A_{\min}(\delta)$ and let $\{\Sigma_{s,t}\}_{s, t \geq 0}$ be defined by the recursion \eqref{eq:Sigma_def} with initial condition \eqref{eq:E0_def}. Then there exists constants $B_1, r_1>0$ such that for all $t\geq 0$, letting $\tau_* := \lim_t \tau_t$,
	\begin{align*}
	|\Sigma_{t,t}-\tau_*^2|&\leq B_1 e^{-r_1 t},
	\qquad \text{ and } \qquad
	|\Sigma_{t,t+1}-\tau_*^2|\leq B_1 e^{-r_1 t}.
	\end{align*}
\end{lemma}

\begin{proof}
Throughout the proof, we use the $\{\eta_{\p}^t\}_{\p \in  \mathbb{N}_{>0}}$ notation introduced in Section~\ref{sec_mainproof} and defined in \eqref{eq:eta_def} with a slight modification to explicitly state the thresholds. Namely, we consider a sequence of denoisers $\eta_{\p}: \R^{\p \times \p} \rightarrow \R^{\p}$ to be those that apply the proximal operator $\prox_{J_{\bfalph \tau_t}}(\cdot)$ defined in \eqref{eq:prox}, i.e.\ $\eta_{\p}(\mathbf{v}; \bfalph \tau_t) := \prox_{J_{\bfalph \tau_t}}( \mathbf{v})$ for a vector $\mathbf{v} \in \R^{\p}$.

Then, per the definition in \eqref{eq:Sigma_def}, we have 
\ben
\Sigma_{s+1, t+1} =  \sigma_w^2 + \lim_{\p}   \mathbb{E}\big\{[\eta_{\p}(\mathbf{B} + \tau_s \mathbf{Z}_s; \bfalph \tau_s) - \mathbf{B}]^\top [\eta_{\p}(\mathbf{B} + \tau_t \mathbf{Z}_t; \bfalph \tau_t) - \mathbf{B}]\big\}/(\delta \p),
\een
where $\mathbf{B} \sim B$ i.i.d.\ elementwise, independent of  length$-\p$ jointly Gaussian vectors $ \mathbf{Z}_s$ and $\mathbf{Z}_r$ having $\mathbb{E}[\mathbf{Z}_s] = \mathbb{E}[\mathbf{Z}_r] = \mathbf{0}$, with covariance $\mathbb{E}\{([\mathbf{Z}_s]_i)^2\} = \mathbb{E}\{([\mathbf{Z}_r]_i)^2\} = 1$ for any element $i \in [\p]$, and $\mathbb{E}\{[\mathbf{Z}_s]_i [\mathbf{Z}_r]_j\} = \frac{\Sigma_{s, r}}{\tau_r \tau_s} \mathbb{I}\{i=j\}$.  Recall, $\Sigma_{t, t} = \tau_t^2$ defined in \eqref{eq:SE2} and by Theorem~\ref{thm:SE1} we know that $\{E_{t,t}\}_{t \geq 0}$ is monotone and converges to $\tau_*^2$ as $t \rightarrow \infty$.  To prove exponential convergence of $\{E_{t-1,t}\}_{t \geq 0}$ as claimed in the lemma statement, we construct a discrete dynamical system below.

	For $t \geq 1$, define the vector $\bm y_t=(y_{t,1},y_{t,2},y_{t,3}) \in \R^3$ as
	\begin{eqnarray}
	y_{t,1} \equiv \Sigma_{t-1,t-1}=\tau_{t-1}^2\, ,\;\;\;
	y_{t,2} \equiv \Sigma_{t,t}=\tau_{t}^2\, ,
	\;\;\;\;
	y_{t,3} \equiv \Sigma_{t-1,t-1}-2\Sigma_{t,t-1}+\Sigma_{t,t}\, .
	\end{eqnarray}
A careful argument shows that the vector $\bm y_t=(y_{t,1},y_{t,2},y_{t,3})$
	belongs to $\mathbb{R}_+^3$.  Essentially this requires showing that a matrix $R_T := $ as in \cite[Lemma 5.8]{lassorisk} is strictly positive definite.	Using the definition of the $\Sigma$ recursion in $\eqref{eq:Sigma_def}$, it is immediate to see that this sequence is updated according to the mapping $\bm y_{t+1} = G(\bm y_t)$ where
	\begin{eqnarray}
	G_1(\bm y_t) & \equiv & y_{t,2}\, ,\\
	G_2(\bm y_t) & \equiv & \sigma_w^2 + \lim_{\p}   \mathbb{E}\big\{\norm{\eta_{\p}(\mathbf{B} + \sqrt{y_{t,2}}\mathbf{Z}_t; \bfalph \sqrt{y_{t,2}}) - \mathbf{B}}^2\big\}/(\delta \p), 
	\label{eq:G2}\\
	G_3(\bm y_t) & \equiv &   \lim_{\p}  \mathbb{E}\big\{\norm{\eta_{\p}(\mathbf{B} + \sqrt{y_{t,2}}\mathbf{Z}_t; \bfalph \sqrt{y_{t,2}}) - \eta_{\p}(\mathbf{B} + \sqrt{y_{t,1}}\mathbf{Z}_{t-1}; \bfalph \sqrt{y_{t,1}})}^2\big\}/(\delta \p),
	%
	\end{eqnarray}
	where $(\mathbf{Z}_t, \mathbf{Z}_{t-1)}$ are length$-\p$ jointly Gaussian vectors, independent of $\mathbf{B} \sim B$ i.i.d.\ elementwise, having $\mathbb{E}[\mathbf{Z}_t] = \mathbb{E}[\mathbf{Z}_{t-1}] = \mathbf{0}$ and with covariance
	$\mathbb{E}\{([\mathbf{Z}_t]_i)^2\} = \mathbb{E}\{([\mathbf{Z}_{t-1}]_i)^2\} = 1$ for any element $i \in [\p]$, and $\mathbb{E}\{[\mathbf{Z}_t]_i [\mathbf{Z}_{t-1}]_j\} = \frac{\Sigma_{t, t-1}}{\tau_t \tau_{t-1}} \mathbb{I}\{i=j\}$.  Notice that $\E\{ \norm{\sqrt{y_{t,2}}\bm Z_t- \sqrt{y_{t,1}} \bm Z_{t-1}}^2\}=y_{t,3}$, where we emphasize that $G_3(\bm y_t)$ depends on $y_{t,3}$ through the covariance of $\bm Z_t$ and $\bm Z_{t-1}$.  Moreover, if $\sigma_w^2 > 0$, then $y_{t,1}$ and $y_{t,2}$ are both strictly positive and by the map defined above it is easy to see that $y_{t,3}$ for all $t \geq 0$.
%
	This mapping is defined for $y_{t,3}\le 2(y_{t,1}+y_{t,2})$.
	
	In the following, we will show by induction on $t$, for $t \geq 1$, that the stronger
	inequality $y_{t,3}< (y_{t,1}+y_{t,2})$ holds. The initial condition implied by Eq.\ \eqref{eq:E0_def} is
	\begin{align*}
	y_{1,1} & = \sigma_w^2 +  \E[B^2]/\delta,\qquad
	y_{1,2}  = \sigma_w^2 +  \lim_{\p}   \mathbb{E}\big\{\norm{\eta_{\p}(\mathbf{B} + \tau_0 \mathbf{Z}_0; \bfalph \tau_0) - \mathbf{B}}^2\big\}/(\delta \p),\\
	y_{1,3}  &=  \lim_{\p}  \mathbb{E}\big\{\norm{\eta_{\p}(\mathbf{B} + \tau_0 \mathbf{Z}_0; \bfalph \tau_0)}^2\big\}/(\delta \p),
	\end{align*}
	It follows that
	\begin{align*}
	y_{1,1}  + y_{1,2} - y_{1,3} 
	& = 2 \sigma_w^2 +   2\lim_{\p} \mathbb{E}\big\{\mathbf{B}^{\top} \big(\mathbf{B} - \eta_{\p}(\mathbf{B} + \tau_0 \mathbf{Z}_0; \bfalph \tau_0)\big)\big\} /(\delta \p) \\
	&= 2 \sigma_w^2 +   2\lim_{\p} \mathbb{E}_{\bm B}\big\{\mathbf{B}^{\top} \big(\mathbf{B} - \mathbb{E}_{\bm Z_0} \{\eta_{\p}(\mathbf{B} + \tau_0 \mathbf{Z}_0; \bfalph \tau_0)\}\big)\big\} /(\delta \p).
	\end{align*}
	Using the above, it is easy to show
	$y_{1,3}<y_{1,1}+y_{1,2}$.  This follows since $\mathbb{E}_{\bm B}\big\{\mathbf{B}^{\top} \big(\mathbf{B} - \mathbb{E}_{\bm Z_0} \{\eta_{\p}^0(\mathbf{B} + \tau_0 \mathbf{Z}_0)\}\big)\big\}$ is asymptotically separable using Lemma~\ref{lem:yue_sep} and because  the function $x \mapsto x-
	\E_Zh^0(x+\tau_0 Z)$ is monotone increasing.  It follows 
	that $ \lim_{\p}  \mathbb{E}_{\bm B}\big\{\mathbf{B}^{\top} \big(\mathbf{B} - \mathbb{E}_{\bm Z_0} \{\eta_{\p}^0(\mathbf{B} + \tau_0 \mathbf{Z}_0)\}\big)\big\}/(\delta \p)>0$.
	
	Suppose that $y_{t,3} < y_{t,1} + y_{t,2}$, we want to show $y_{t+1,3} < y_{t+1,1} + y_{t+1,2}$.  By the induction hypothesis, $\mathbb{E}\{[\mathbf{Z}_t]_i [\mathbf{Z}_{t-1}]_i\} = \frac{y_{t,1} + y_{t,2} - y_{t,3}}{2\sqrt{y_{t,1} y_{t,2}}} > 0$, so elementwise $\mathbf{Z}_t$ and $\mathbf{Z}_{t-1}$ are positively correlated.
	\be
	\begin{split}
		\label{eq:induct1}
		&y_{t+1,1}+y_{t+1,2}-y_{t+1,3} \\
		&= 2\sigma_w^2 +  \lim_{\p} 2 \mathbb{E}\big\{[\eta_{\p}(\mathbf{B} + \sqrt{y_{t,2}}\mathbf{Z}_t; \bfalph\sqrt{y_{t,2}}) - \mathbf{B} ]^{\top} [\eta_{\p}(\mathbf{B} + \sqrt{y_{t,1}}\mathbf{Z}_{t-1}; \bfalph\sqrt{y_{t,1}})-\mathbf{B} ]\big\}/(\delta \p).
	\end{split}
	\ee
Notice that $x\mapsto\eta(b+c\cdot x \, ; \, \theta)-b$ is monotone for any constants $b$ and $c>0$ and consider the following result: for $ g $, a monotone function, and $ X_1 $ and $ X_2 $, two positively correlated standard Gaussians, $ \E[g(X_1)g(X_2)] \geq 0 $. This is a special case of a theorem in \cite{pitt1982positively}, which shows $ \E[g(X_1)g(X_2)] \geq \E[g(X_1)]\E[g(X_2)] = (\E[g(X_1)])^2 > 0 $.
Then since $\mathbf{Z}_t$ and $\mathbf{Z}_{t-1}$ are positively correlated, $ \mathbb{E}\big\{[\eta_{\p}(\mathbf{B} + \sqrt{y_{t,2}}\mathbf{Z}_t; \bfalph\sqrt{y_{t,2}}) - \mathbf{B} ]^{\top} [\eta_{\p}(\mathbf{B} + \sqrt{y_{t,1}}\mathbf{Z}_{t-1}; \bfalph\sqrt{y_{t,1}})-\mathbf{B} ]\big\} \ge 0$, which yields $y_{t+1,3}< (y_{t+1,1}+y_{t+1,2})$. 

	We can hereafter therefore assume $y_{t,3}< y_{t,1}+y_{t,2}$
	for all $t$.
	
	We will consider the above iteration for arbitrary
	initialization $y_0$ (satisfying  $y_{0,3}< y_{0,1}+y_{0,2}$)
	and will show the following three facts:
	\begin{itemize}
		\item[]{\bf Fact (i).} $y_{t,1},y_{t,2}\to\tau_*^2$ as $t\to\infty$.
		Further the convergence is monotone.
		\item[]{\bf Fact (ii).} If $y_{0,1}=y_{0,2}=\tau_*^2$ and $y_{0,3}\le 2\tau_*^2$,
		then $y_{t,1}=y_{t,2}=\tau_*^2$ for all $t$ and $y_{t,3}\to 0$.
		\item[]{\bf Fact (iii).} The Jacobian $J=J_{G}(y_*)$ of $G$
		at $y_* = (\tau_*^2,\tau_*^2,0)$ has spectral radius $\sigma(J)<1$.
	\end{itemize}
	By simple compactness arguments, Facts (i) and (ii) imply $y_t\to y_*$
	as $t\to\infty$. (Notice that $y_{t,3}$ remains bounded
	since $y_{t,3}\le (y_{t,1}+y_{t,2})$ and by the convergence
	of $y_{t,1},y_{t,2}$.)
	Fact (iii) implies that convergence is exponentially
	fast.
	
	\textbf{\emph{Proof of Fact (i).}} Notice that $y_{t,2}$ evolves independently by $y_{t+1,2} = G_2(y_t) = F(y_{2,t},\bfalph\sqrt{y_{2,t}})$,
	with $F(\,\cdot\,,\,\cdot\,)$ the state evolution mapping introduced in \eqref{eq:SE_F}. It follows from Proposition 1.3 that $y_{t,2}\to \tau_*^2$ monotonically for any initial condition.
	Since $y_{t+1,1} = y_{t,2}$, the same happens for $y_{t,1}$.
	
	\textbf{\emph{Proof of Fact (ii).}} Consider the function 
	\[G_*(x) =G_3(\tau_*^2,\tau_*^2,x)=\lim_{\p} \mathbb{E}\big\{\norm{\eta_{\p}(\mathbf{B} + \tau_*\mathbf{Z}_t; \bfalph \tau_*) - \eta_{\p}(\mathbf{B} + \tau_*\mathbf{Z}_{t-1}; \bfalph \tau_*)}^2\big\}/(\delta \p),\]
	where
\[\mathbb{E}\{[\mathbf{Z}_t]_i [\mathbf{Z}_{t-1}]_i\} = \frac{y_{t,1} + y_{t,2} - y_{t,3}}{2\sqrt{y_{t,1} y_{t,2}}} = \frac{2\tau_*^2 - x}{2\tau_*^2}\] 
is no longer time-dependent. This function is defined for $x\in[0,2\tau_*^2]$. Further $G_*$ can be represented as follows in terms of the independent random vectors $\bm Z$, $\bm W \sim N(0,\mathbb{I})$:
	\begin{eqnarray*}
		&G_*(x) = \lim_{\p} \frac{1}{\delta \p} \mathbb{E}\big\{\norm{\eta_{\p}(\mathbf{B} + \bm Z \sqrt{\tau_*^2-\frac{1}{4}x}+ \bm W(\frac{1}{2}\sqrt{x}); \bfalph \tau_*) - \eta_{\p}(\mathbf{B} +\bm Z \sqrt{\tau_*^2-\frac{1}{4}x}- \bm W (\frac{1}{2}\sqrt{x}) ; \bfalph \tau_*)}^2\big\},
		%
	\end{eqnarray*}
	where 
	\[(\tau_*\mathbf{Z}_{t-1}, \tau_*\mathbf{Z}_t) \overset{d}{=}\Big(\bm Z \sqrt{\tau_*^2-\frac{1}{4}x} - \bm W (\frac{1}{2}\sqrt{x}) ,\bm Z \sqrt{\tau_*^2-\frac{1}{4}x}+  \bm W (\frac{1}{2}\sqrt{x})\Big). \]
	Obviously $G_*(0) = 0$. A simple Taylor expansion about the first argument around $\mathbf{B}$ yields (recall higher derivatives of $\eta$ are 0 almost everywhere)
	\begin{align*}
	G_*(x) &=  \lim_{\p}  \E\Big\{\norm{\eta_p(\mathbf{B}; \bfalph \tau_*)+ \Big(\bm Z \sqrt{\tau_*^2-\frac{1}{4}x}+ \bm W(\frac{1}{2}\sqrt{x})\Big) \odot \partial_1\eta_p(\mathbf{B}; \bfalph \tau_*) \\
	&\hspace{3cm} -\eta_p(\mathbf{B}; \bfalph \tau_*)- \Big(\bm Z \sqrt{\tau_*^2-\frac{1}{4}x}- \bm W(\frac{1}{2}\sqrt{x})\Big) \odot \partial_1\eta_p(\mathbf{B}; \bfalph \tau_*)]}^2\Big\}/(\delta \p)
	\\
	&=  \lim_{\p} x \E\big\{ \norm{\bm W \odot \partial_1\eta_p(\mathbf{B}; \bfalph \tau_*)]}^2\big\}/(\delta \p)=  \lim_{\p} x \E\big\{ \norm{\partial_1\eta_p(\mathbf{B}; \bfalph \tau_*)]}^2\big\}/(\delta \p).
	\end{align*}
Using the above, we study $G'_*(x)$.  First, we can exchange the limit and differentiation because $f_p(x):=x \E\big\{ \norm{\partial_1\eta_p(\mathbf{B}; \bfalph \tau_*)]}^2\big\}/(\delta \p)$ converges uniformly to $f(x) :=\lim_{\p} x \E\big\{ \norm{\partial_1\eta_p(\mathbf{B}; \bfalph \tau_*)]}^2\big\}/(\delta \p)$. To see this, notice $f_p, f$ are linear in $x$ and defined on $[0,2\tau_*^2]$. 
Hence for every $\epsilon>0$, there exists $p_0$ such that 
\begin{align*}
|f_{p_0}(x)-f(x)|&=x \Big \lvert \frac{1}{\delta \p_0} \E\big\{ \norm{\partial_1\eta_{p_0}(\mathbf{B}; \bfalph \tau_*)]}^2\big\} -\lim_{\p}  \frac{1}{\delta \p} \E\big\{ \norm{\partial_1\eta_p(\mathbf{B}; \bfalph \tau_*)]}^2\big\} \Big \lvert \\
&\leq 2\tau_*^2 \Big \lvert \frac{1}{\delta \p_0} \E\big\{ \norm{\partial_1\eta_{p_0}(\mathbf{B}; \bfalph \tau_*)]}^2\big\} -\lim_{\p}  \frac{1}{\delta \p} \E\big\{ \norm{\partial_1\eta_p(\mathbf{B}; \bfalph \tau_*)]}^2\big\} \Big \lvert<\epsilon.
\end{align*}
	By uniform convergence we have,
	\begin{align*}
	G'_*(x)&= \lim_{\p}  \frac{1}{\delta \p} \E\big\{ \norm{\partial_1\eta_p(\mathbf{B}; \bfalph \tau_*)]}^2\big\} =G'_*(0)\leq \lim_{\p}  \frac{1}{\delta \p} \sum_{i=1}^{\p} \E\big\{  [\partial_1\eta_p(\mathbf{B}; \bfalph \tau_*)]_i\big\}. 
	\end{align*}
	Hence $G_*'(0)<1$, using \eqref{eq:blam_alph_mapping} since $\blam>\bm 0$. Then $y_{t,3}= [G_*'(0)]^ty_{0,3}\to
	0$ as $t \goto \infty$ as claimed.  
	
	\textbf{\emph{Proof of Fact (iii).}}
	By the definition of $G$, the Jacobian is given by
	\begin{align*}
	J_{G}(y_*) = \left(\begin{array}{ccc}
	0 & 1 & 0\\
	0 & \F'(\tau_*^2) & 0\\
	a & G_*'(0) & b\\
	\end{array}\right)
	\end{align*}
	denoting $\F'(\tau_*^2) \equiv
\frac{\partial \F}{\partial \tau^2}(\tau^2,\bfalph \tau)$ evaluated at $\tau^2=\tau^2_*$ with $a$ and $b$ constants whose values are not important to the proof. Computing the eigenvalues of the Jacobian,
	we get
	$
	\sigma(J) = \max\big\{\, \F'(\tau_*^2)\,,\, G_*'(0) \,\big\}.
	$
	Since $G_*'(0)<1$ proved above and $\F(\tau_*^2)<1$
	by Theorem~\ref{thm:SE1}, the claim follows.
\end{proof}
%
%
\begin{proof}[Proof of Lemma \ref{lem:4.3}]
We show that Lemma \ref{lem:4.3} follows by Lemmas \ref{lem:5.7} and \ref{Thm42}. By Lemma~\ref{Thm42},
\begin{align*}
\plim_{\n} \big(\norm{\z^t - \z^{t-1}}^2/n - (\tau_t^2 - 2\Sigma_{t, t-1}+ \tau_{t-1}^2)\big) = 0, \\
\plim_{\p} \big(\norm{\bet^{t+1} - \bet^{t}}^2/(\delta \p) - (\tau_t^2 - 2\Sigma_{t, t-1}+ \tau_{t-1}^2)\big) = 0,
\end{align*}
and so it is sufficient to show that $\lim_t (\tau_t^2 - 2\Sigma_{t, t-1}+ \tau_{t-1}^2) = 0$.  Note that this follows from Lemma \ref{lem:5.7} since $\tau_t^2 = \Sigma_{t,t}$ and $\tau_{t-1}^2 = \Sigma_{t-1, t-1}$ both converge to $\tau_*^2$ as does $\Sigma_{t, t-1}.$

\end{proof}


\section{Technical Details for the Condition (3) Proof}\label{app_34}
We first introduce some notation and ideas that will be used throughout the proof. The proof is similar to \cite[Section 5.3]{lassorisk}, with the key difference being the concept of equivalence classes as described in Section \ref{sec:prel-slope-amp}.

We now introduce a more general recursion than the AMP algorithm in \eqref{eq:AMP0}-\eqref{eq:AMP1}.  Given $\w \in \mathbb{R}^{\n}$ and $\bet \in \mathbb{R}^{\p}$, define the  column vectors $\mathbf{h}^{t+1}, \mathbf{q}^{t+1} \in \mathbb{R}^{\p}$ and $\mathbf{b}^t, \mathbf{m}^t \in \mathbb{R}^{\n}$, recursively, for $t \geq 0$ as follows, starting with initial condition $\bet^0=0$ and $\z^0=\mathbf{y}$. 
\begin{equation}
\begin{split}
\mathbf{h}^{t+1} = \bet - (\mathbf{X}^\top \z^t + \bet^t), \qquad &  \mathbf{q}^t  =\bet^{t} - \bet, \qquad \mathbf{b}^t = \w- \z^t,\qquad  \mathbf{m}^t  =-\z^t.
\end{split}
\label{eq:hqbm_def_AMP}
\end{equation}
Note that these definitions of $\bm h^t$ and $\bm m^t$ match those used in Section~\ref{app_AMP}.

Denoting $[\bm u| \bm v]$ to mean the matrix of concatenating vectors $\bm u,  \bm v$ horizontally, we define
\be
\begin{split}
\underbrace{[\mathbf{h}^1+\mathbf{q}^0|\cdots|\mathbf{h}^t+\mathbf{q}^{t-1}]}_{\bm A_t}
&=\X^\top \underbrace{[\mathbf{m}^0|\cdots|\mathbf{m}^{t-1}]}_{\bm M_t},
\\
\underbrace{[\mathbf{b}^0|\mathbf{b}^1+\kappa_1\mathbf{m}^0|\cdots|\mathbf{b}^{t-1}+\kappa_{t-1}\mathbf{m}^{t-2}]}_{\bm Y_t}
&=\X\underbrace{[\mathbf{q}^0|\cdots|\mathbf{q}^{t-1}]}_{\bm Q_t},
\label{eq:matrix_defs}
\end{split}
\ee
where the scalars $\kappa_t$ are defined as 
$\kappa_t :=  -[\nabla \eta^{t-1}(\bet-\mathbf{h}^{t-1})]/\n.$

Define the $\sigma$-algebra generated by $\mathbf{b}^0,\cdots,\mathbf{b}^{t-1},\mathbf{m}^0,\cdots,\mathbf{m}^{t-1},\mathbf{h}^1,\cdots,\mathbf{h}^{t},\mathbf{q}^0,\cdots,\mathbf{q}^{t}$ as $\mathfrak{S}_t$. Then \cite{amp2, nonseparable}, says that the conditional distribution of the random matrix $\X$ given $\mathfrak{S}_t$ is
\begin{align}
\label{eq:cond_A}
\X|_{\mathfrak{S}_t}\overset{d}{=}\bm E_t+ \bm P_{\bm M_t}^\perp\tilde{\X} \bm P_{\bm Q_t}^\perp,
\end{align}
where $\tilde{\X}\overset{d}{=}\X$ is independent of the conditioning sigma-algebra $\mathfrak{S}_t$ and $\bm E_t=\E(\X|\mathfrak{S}_t)$ is given by:
\begin{align*}
\bm E_t:=\bm Y_t(\bm Q_t^\top \bm Q_t)^{-1}\bm Q_t^\top+\bm M_t(\bm M_t^\top \bm M_t)^{-1}\bm A_t^\top+\bm M_t(\bm M_t^\top \bm M_t)^{-1}\bm M_t^\top \bm Y_t(\bm Q_t^\top \bm Q_t)^{-1}\bm Q_t^\top.
\end{align*}
In \eqref{eq:cond_A}, we use the notation $\bm P_{\bm M_t}^\perp=\mathbb{I}- \bm P_{\bm M_t}$ and$\bm P_{\bm Q_t}^\perp=\mathbb{I}- \bm P_{\bm Q_t}$ where $\bm P_{\bm Q_t}$ and $\bm P_{\bm M_t}$ are orthogonal projectors onto column spaces of $\bm Q_t, \bm M_t$ respectively. From now on, since $t$ is fixed, we will drop the subscript $t$ when it is clear.  A proof of \eqref{eq:cond_A} can be found in \cite[Lemma 11]{amp2}.  We note that there are no differences in this conditional distribution in the nonseparable case, since the analysis (in both cases) is just that of an i.i.d.\ Gaussian matrix conditional on linear constraints.

Given the above notations, we claim that Lemma \ref{lemma:MinS} is implied by the following statement.
\begin{lemma}
	\label{lemma:ConcreteMinS}
	Let $s$ be a set of maximal atoms in $[\p]$ such that $|s|\le \p(\delta-\gamma)$,
	for some $\gamma>0$. Then there exists $\alpha_1=\alpha_1(\gamma)>0$
	(independent of $t$) and
	$\alpha_2=\alpha_2(\gamma,t)>0$ (depending on $t$ and $\gamma$)
	with
	\begin{eqnarray*}
		\PP\Big\{\min_{\|\bm v\|=1,\,\supp^*(\bm v)\subseteq s}\big\| \bm E \bm v +
		\bm P_{\bm M}^{\perp}\tilde{\X} \bm P_{\bm Q}^{\perp} \bm v\big\|\le \alpha_2\,\Big|\,
		\mathfrak{S}_t\Big\}\le \, e^{-p\alpha_1}\, ,
	\end{eqnarray*}
	eventually almost surely as $p\to\infty$, with $\bm E \bm v = \bm Y(\bm Q^*\bm Q)^{-1}\bm Q^*\bm P_{\bm Q} \bm v + \bm M(\bm M^* \bm M)^{-1} \bm X^*\bm P_{\bm Q}^{\perp} \bm v$.
\end{lemma}
We prove such implication in the next section now.

\begin{proof}[Proof of Lemma \ref{lemma:MinS}]
The proof is adapted from \cite[Section 5.3.1]{lassorisk}.
First note that by Borel-Cantelli, it is sufficient to show
that, for $s$ measurable on $\mathfrak{S}_t$ and
$|s|\le \p(\delta-c)$ there exist $a_1=a_1(c)>0$ and $a_2=a_2(c,t)>0$, such that
\begin{eqnarray*}
	\PP\Big\{\min_{|s'|\le a_1\p}\;
	\min_{\|\bm v\|=1, \,\, \supp^*(\bm v)\subseteq s\cup s'}\|\X \bm v\|< a_2\Big\} \le {1}/{\p^2}\,,
\end{eqnarray*}
for all $\p$ large enough, using $\sigma_{\rm min}(\X_{S_t\cup S'}) = \min_{\|\bm v\|=1, \,\, \supp^*(\bm v)\subseteq s\cup s'}\|\X \bm v\|$.  To shorten notation, the set $\{\|\bm v\|=1, \,\, \supp^*(\bm v)\subseteq s\cup s'\}$ is denoted $\bm v(s')$.  Now, conditioning on $\mathfrak{S}_t$, by a union bound,
\be
\begin{split}
\label{eq:entropy}
&\PP\{\min_{|s'|\le a_1\p} \min_{\bm v(s')}\| \X \bm v\| <a_2\big|{\mathfrak{S}_t}\}\leq \sum_{|s'|\le a_1\p}\PP\{\min_{\bm v(s')}\| \X \bm v\|<a_2\big|{\mathfrak{S}_t}\}\\
&\leq \Big[\sum_{k=1}^{a_1\p}{\p\choose k}\Big]\max_{|s'|\le \p a_1}\PP\{\min_{\bm v(s')}\| \X \bm v\|<a_2\big|{\mathfrak{S}_t}\} \le e^{\p h(a_1)}\max_{|s'|\le a_1 \p}\PP\{\min_{\bm v(s')}\| \X \bm v\|<a_2\big|{\mathfrak{S}_t}\}\,,
\end{split}
\ee
where $h(a)=-a\log a-(1-a)\log(1-a)$ is the binary entropy function (cf. \cite[Chapter 10, Corollary 9]{macwilliams1977theory}). Therefore, using iterated expectation and \eqref{eq:entropy},
\begin{align*}
\PP\Big\{\min_{|s'|\le a_1\p}\; 
\min_{\bm v(s')}\|\X \bm v\|< a_2\Big\} &= \E\Big\{\PP\Big\{\min_{|s'|\le a_1\p}\;
\min_{\bm v(s')}\|\X \bm v\|< a_2\Big|\,\mathfrak{S}_t\Big\}
\Big\} \\
&\le e^{\p h(a_1)}\E\Big\{\max_{|s'|\le a_1 \p}\PP\Big\{
\min_{\bm v(s')}\|\X \bm v\|< a_2\Big|\,\mathfrak{S}_t\Big\}
\Big\}\,,
\end{align*}
Now, we fix $a_1<c/2$ in such a way that $h(a_1)\le \frac{1}{2}\alpha_1(\frac{c}{2})$
and let $a_2=\frac{1}{2}\alpha_2(\frac{c}{2},t)$ where $\alpha_1$ and $\alpha_2$ are defined by Lemma \ref{lemma:ConcreteMinS}.
Then,
\begin{align*}
\PP\Big\{\min_{|s'|\le a_1\p}\; 
\min_{\bm v(s')}\|\X \bm v\|< a_2\Big\} &\leq e^{\frac{1}{2}\p \alpha_1(\frac{c}{2})} \E\Big\{\max_{|s'|\le a_1 \p}\PP\Big\{
\min_{\|\bm v\|=1, \,\, \supp^*(\bm v)\subseteq s\cup s'}\|\X \bm v\|< \frac{1}{2}
\alpha_2(\frac{c}{2},t)\Big|\,\mathfrak{S}_t\Big\}
\Big\}\\
&\leq e^{\frac{1}{2}\p \alpha_1(\frac{c}{2})}\;\E\Big\{\max_{|s''|\le \p(\delta-\frac{c}{2})}\PP\Big\{
\min_{\|\bm v\|=1, \,\, \supp^*(\bm v)\subseteq s''}\|\X \bm v \|<\frac{1}{2}
\alpha_2(\frac{c}{2},t)\Big|\,\mathfrak{S}_t\Big\}
\Big\} \, .
\end{align*}
Finally, using (cf. \cite[Lemma 5.1]{lassorisk}),
\begin{align}
\X \bm v|_{\mathfrak{S}} \stackrel{d}{=} \bm Y(\bm Q^*\bm Q)^{-1}\bm Q^*\bm P_{\bm Q} \bm v + \bm M(\bm M^* \bm M)^{-1} \bm X^*\bm P_{\bm Q}^{\perp} \bm v +
\bm P_{\bm M}^{\perp}\tilde{\X} \bm P_{\bm Q}^{\perp} \bm v \, .
\end{align}
to estimate $\X \bm v$ and applying Lemma \ref{lemma:ConcreteMinS}, we get,
for all $\p$ large enough,
\begin{align*}
\PP\Big\{\min_{|s'|\le a_1\p}\; 
\min_{\bm v(s')}\|\X \bm v\|< a_2\Big\} \leq  e^{\frac{1}{2}\p\alpha_1}\;\E\big\{\max_{|s''|\le \p(\delta-\frac{c}{2})} e^{-\p\alpha_1}
\big\}\le {1}/{\p^2} \, .
\end{align*}
\end{proof}
Now we prove Lemma \ref{lemma:ConcreteMinS}, using a proof that 
%
%
%
is similar to that of \cite[Section 5.3.2]{lassorisk}. We first state some lemmas that will be used in the proof, but we will not migrate the full proofs from \cite{lassorisk} for the sake of brevity. Instead, we describe the key points of proofs with an emphasis on the technical differences for the SLOPE problem and provide pointers to the original proofs.

The concept of maximal atoms are reflected in these lemmas via the sets $s$ and correspondingly $\bm P_s$, where $\bm P_s$ is the $\p \times \p$ projector matrix onto the subspace of vectors whose $\supp^*$ equals $s$. In the LASSO case where $\supp^*\equiv\supp$ and $s\equiv S$, the projector is orthogonal, but in general, we must define $\bm P_s[\cdot,j]=\frac{1}{|I|}\sum_{i\in I} \bm e_i$ for $j \in I$ where $\bm P_s[\cdot,j]$ is the $j^{th}$ column of $\bm P_s$ for $1 \leq j \leq \p$ and $\bm e_i$ is the $i^{th}$ vector of the standard basis. 
For example, when $\p=4$ and $s=\{\{1\},\{2,4\}\}$,
$$
\bm P_s=
\begin{pmatrix}
1&0&0&0
\\
0&1/2&0&1/2
\\
0&0&0&0
\\
0&1/2&0&1/2
\end{pmatrix}.
$$
Such a projector is not necessarily orthogonal and its rank is described via $|s|$ (the number of equivalence classes), not via $|S|$ (the number of non-zero elements) as for the LASSO. We may view this projector as an orthogonal projector onto the subspace of maximal atoms: for a maximal atom $I\in s$, the projector maps elements whose indices belong to $I$ onto their average value.

We begin with the auxiliary lemmas.
\begin{lemma}{[Adapted from \cite[Lemma 5.4]{lassorisk}]}
	\label{lemma:Pitagora}
	Let $s$ be a set of maximal atoms in $[\p]$ such that $|s|\le p(\delta-\gamma)$,
	for some $\gamma>0$. Recall that
	$\bm E \bm v =\bm Y(\bm Q^\top \bm Q)^{-1} \bm Q^\top \bm P_{\bm Q} \bm v + \bm M(\bm M^\top \bm M)^{-1}\bm A^\top \bm P_{\bm Q}^{\perp} \bm v$
	and consider the event
	\begin{align*}
		&\varepsilon_1 := \\
		&\Big\{ \big\| \bm E \bm v +
		\bm P_{\bm M}^{\perp}\tilde{\X} \bm P_{\bm Q}^{\perp} \bm v\big\|^2\ge  \frac{\gamma}{4\delta}\big\| \bm E \bm v
		-\bm P_{\bm M} \tilde{\X} \bm P_{\bm Q}^{\perp} \bm v\big\|^2+
		\frac{\gamma}{4\delta}\big\|\tilde{\X} \bm P_{\bm Q}^{\perp} \bm v\big\|^2\, \;\forall \, \bm v
		\mbox{ s.t. } \|\bm v\|=1\mbox{ and } \supp^*(\bm v)\subseteq s\Big\}.
	\end{align*}
	Then there exists $a=a(\gamma)>0$ such that $\PP\{\varepsilon_1|\mathfrak{S}_t\}\ge 1-e^{-pa}$.
		\label{leme1}
\end{lemma}

\begin{proof}[Sketch proof]
Define an event $\widetilde{\varepsilon}_1$ as follows:
	\begin{eqnarray}
	\widetilde{\varepsilon}_1 = \Big\{|(\bm E \bm v-\bm P_{\bm M}\tilde{\X} \bm P_{\bm Q}^{\perp} \bm v)^\top(\tilde{\X} \bm P_{\bm Q}^{\perp} \bm v)| &\le &
	\Big(1-\frac{\gamma}{2\delta}\Big)^{1/2} \, \|\bm E \bm v-\bm P_{\bm M}\tilde{\X} \bm P_{\bm Q}^{\perp} \bm v \|
	\, \|\tilde{\X} \bm P_{\bm Q}^{\perp} \bm v\| \Big\},\label{claim:ScalarProd}
\end{eqnarray}
where the event $\widetilde{\varepsilon}_1$ is meant to hold for all $\bm v$ such that $\|\bm v\|=1\mbox{ and } \supp^*(\bm v)\subseteq s$.  We claim that $\PP\{\widetilde{\varepsilon}_1|\mathfrak{S}_t\}\ge 1-e^{-pa}$.
To prove the claim,
	we use that for any $\bm v$, the  unit vector
	$\tilde{\X} \bm P_{\bm Q}^{\perp} \bm v /\|\tilde{\X} \bm P_{\bm Q}^{\perp} \bm v \|$ belongs to the random linear space $\im(\tilde{\X} \bm P_{\bm Q}^{\perp}\bm P_s)$ with dimension at most $\p(\delta-\gamma)$.
	Also, $\bm E \bm v-\bm P_{\bm M}\tilde{\X} \bm P_{\bm Q}^{\perp} \bm v$ belongs to space spanned by the column space of the matrices $\bm M$ and of $\bm B$ where $\bm B_t = [\bm b^0| \ldots | \bm b^{t-1}]$ defined in \eqref{eq:hqbm_def_AMP} and \eqref{eq:matrix_defs}, having dimension at most $2t$. Applying Proposition \ref{propo:Concentration} using $m=n, m\lambda=\p(\delta-\gamma), d=2t$ and $\ve=(1-\frac{\gamma}{2\delta})^{1/2}(1-\frac{\gamma}{\delta})^{1/2}$ gives that the event
	\[
	\left(\frac{\bm E \bm v-\bm P_{\bm M}\tilde{\X} \bm P_{\bm Q}^{\perp} \bm v}{\|\bm E \bm v-\bm P_{\bm M}\tilde{\X} \bm P_{\bm Q}^{\perp} \bm v \|}\right)^{\top}\frac{\tilde{\X} \bm P_{\bm Q}^{\perp} \bm v }{\|\tilde{\X} \bm P_{\bm Q}^{\perp} \bm v \|} \leq \sqrt{\lambda}+\ve=\Big(1-\frac{\gamma}{2\delta}\Big)^{1/2}\,,
	\]
	holds with the desired probability, proving the claim. Conditional on event \eqref{claim:ScalarProd}, one can show
	\begin{align*}
	\big\| \bm E \bm v +
	\bm P_{\bm M}^{\perp}\tilde{\X} \bm P_{\bm Q}^{\perp} \bm v\big\|^2
	&\ge \Big(1-\Big(1-\frac{\gamma}{2\delta}\Big)^{1/2}\Big)
	\Big\{\big\| \bm E \bm v-\bm P_{\bm M}\tilde{\X} \bm P_{\bm Q}^{\perp} \bm v \big\|^2+\big\| \tilde{\X} \bm P_{\bm Q}^{\perp} \bm v\big\|^2\Big\}\, .
	\end{align*}
	Finally observe that $1-(1-\frac{\gamma}{2\delta})^{1/2}\geq\frac{\gamma}{4\delta}$ and therefore since event $\widetilde{\varepsilon}_1$ occurring implies $\varepsilon_1$ occurs, giving the desired probability of $\varepsilon_1$ as well.
\end{proof}
Next we estimate the term $\|\tilde{\X} \bm P_{\bm Q}^{\perp} \bm v\|^2$ in the above lower bound.
\begin{lemma}{[Adapted from \cite[Lemma 5.5]{lassorisk}]}
	\label{lemma:Event2}
	Let $s$ be a set of maximal atoms in $[\p]$ such that $|s|\le p(\delta-\gamma)$,
	for some $\gamma>0$. Then there exists constant $c_1=c_1(\gamma)$,
	$c_2 = c_2(\gamma)$ such that the event
	\begin{eqnarray*}
		\varepsilon_2 :=
		\Big\{ \big\|\tilde{\X} \bm P_{\bm Q}^{\perp} \bm v\big\| \ge c_1(\gamma)
		\|\bm P_{\bm Q}^{\perp} \bm v\big\| \, \;\forall \, \bm v\;
		\mbox{ such that }\;\supp^*(\bm v)\subseteq s\Big\}\,
	\end{eqnarray*}
	holds with probability $\PP\{\varepsilon_2|\mathfrak{S}_t\}\ge 1-e^{-pc_2}$.
		\label{leme2}
\end{lemma}
\begin{proof}[Sketch proof]
	Let $V$ be the linear space $V=\im(\bm P_{\bm Q}^{\perp}\bm P_s)$ having dimension at most $\p(\delta-\gamma)$. For all $\bm v$ with  $\supp^*(\bm v)\subseteq s$,
	\begin{eqnarray}
	\big\|\tilde{\X} \bm P_{\bm Q}^{\perp} \bm v\big\| \ge \sigma_{\rm min}(\tilde{\X}|_{V}) \, \|\bm P_{\bm Q}^{\perp} \bm v\big\|\, ,
	\end{eqnarray}
	where $\tilde{\X}|_V$ refers to the restriction of $\tilde{\X}$ to $V$.
	Then $\sigma_{\rm min}(\tilde{\X}|_{V})$ is distributed as the minimum singular value
	of a Gaussian matrix of dimensions $\p\delta\times {\rm dim}(V)$, which is almost surely bounded away from $0$ as $\p\to\infty$ (see Theorem G.
	\ref{prop:marchenko-pastur}). Large deviation estimates
	\cite{litvak2005smallest} imply that the probability that $\sigma_{\min}$ is smaller than a constant $c_1(\gamma)$ is exponentially small.
\end{proof}

In the next step we estimate the norm $\bm E \bm v$ by quoting the following result.
\begin{lemma}\cite[Lemma 5.6]{lassorisk}\label{lemma:EvBound}
	There exists a constant $c = c(t)>0$ such that, defining the event,
	\begin{eqnarray}
	\mathcal{E}_3 :=
	\big\{\| \bm E\bm P_{\bm Q} \bm v\|\ge c(t)\|\bm P_{\bm Q} \bm v\|\, ,
	\| \bm E\bm P_{\bm Q}^{\perp} \bm v\|\le c(t)^{-1}\| \bm P^{\perp}_{\bm Q} \bm v\|
	,\;\mbox{ for all }\; \bm v\in \mathbb{R}^\p
	\big\}\, ,
	\end{eqnarray}
	we have that $\mathcal{E}_3$ holds eventually almost surely as $\p\to\infty$.
	\label{leme3}
\end{lemma}

Finally, we can now prove Lemma \ref{lemma:ConcreteMinS} with the ingredients given in Lemmas~\ref{leme1}-\ref{leme3}. We restate the proof from \cite[Lemma 5.3]{lassorisk} with minor changes.

\begin{proof}[Proof of Lemma \ref{lemma:ConcreteMinS}]
	We start with Lemma \ref{lemma:EvBound} by which we assume that event $\mathcal{E}_3$ holds
	for some function $c= c(t)$ (without loss of generality $c<1/2$).
	For $\alpha_2(t)>0$ small enough, let $\mathcal{E}$ be the event
	\begin{eqnarray}
	\mathcal{E}:=\Big\{ \min_{\|\bm v\|=1,\,\supp^*(\bm v)\subseteq s}\big\| \bm E \bm v +
	\bm P_{\bm M}^{\perp}\tilde{\X} \bm P_{\bm Q}^{\perp} \bm v\big\|\le \alpha_2(t)\Big\}\, .
	\end{eqnarray}
	First assume $\|\bm P_{\bm Q}^{\perp} \bm v\|\le c^2/10$, from which it follows,
	\begin{align*}
	\| \bm E \bm v-\bm P_{\bm M}\tilde{\X} \bm P_{\bm Q}^{\perp} \bm v\| &\ge \|\bm E\bm P_{\bm Q} \bm v\|-\|\bm E\bm P_{\bm Q}^{\perp} \bm v\|- \|\bm P_{\bm M}\tilde{\X} \bm P_{\bm Q}^{\perp} \bm v\|\\
	&\ge c\|\bm P_{\bm Q} \bm v\| - (c^{-1}+\|\tilde{\X}\|_2)\|\bm P_{\bm Q}^{\perp} \bm v\|\ge \frac{c}{2}-\frac{c}{10}-\|\tilde{\X}\|_2\frac{c^2}{10} =
	\frac{2c}{5}-\|\tilde{\X}\|_2\frac{c^2}{10}\,,
	\end{align*}
	where the last inequality uses $\|\bm P_{\bm Q} \bm v\|=\sqrt{1-\|\bm P_{\bm Q}^{\perp} \bm v\|^2}\ge 1/2$ under the assumption $\|\bm P_{\bm Q}^{\perp} \bm v\|\le c^2/10$.
	Therefore, using Lemma \ref{lemma:Pitagora}, we get
	\begin{align*}
	\PP\{\mathcal{E}|\mathfrak{S}_t\}\le
	\PP\Big\{\frac{2c}{5}-\|\tilde{\X}\|_2\frac{c^2}{10}\le \Big(\frac{4\delta}{\gamma}\Big)^{{1}/{2}}
	\alpha_2(t)\Big|
	\mathfrak{S}_t\Big\}+ e^{-\p a}\, ,
	\end{align*}
	and the thesis follows from large deviation bounds on the
	norm $\|\tilde{\X}\|_2$ (see \cite{Ledoux}) by first taking  $c$
	small enough,
	and then choosing $\alpha_2(t)<\frac{c}{5}\sqrt{\frac{\gamma}{4\delta}}$.
	
	Next assume $\|\bm P_{\bm Q}^{\perp} \bm v\|\ge c^2/10$.
	By Lemma \ref{lemma:Pitagora} and \ref{lemma:Event2},
	we can assume events $\mathcal{E}_1$ and $\mathcal{E}_2$ hold. Therefore
	$
		\big\| \bm E \bm v +
		\bm P_{\bm M}^{\perp}\tilde{\X} \bm P_{\bm Q}^{\perp} \bm v\big\|\ge (\frac{\gamma}{4\delta})^{{1}/{2}}
		\|\tilde{\X} \bm P_{\bm Q}^{\perp} \bm v\big\|\ge (\frac{\gamma}{4\delta})^{{1}/{2}}c_1(\gamma)
		\|\bm P_{\bm Q}^{\perp} \bm v\|\,,
$
	proving our thesis.
\end{proof}


\section{Some Useful Auxiliary Material}\label{app_F}
We collect some auxiliary results that are necessary in our proof.  Most of these are results that were initially stated in \cite{lassorisk} that we repeat here for the reader.

The following proposition is used in the proof of Lemma \ref{lemma:Pitagora}. The proof is identical to that of \cite[Proposition E.1]{lassorisk} and it follows from a standard concentration of measure argument in \cite{Ledoux}.  For this reason, we don't repeat it here.

\begin{proposition}\label{propo:Concentration}
	Let $V\subseteq\mathbb{R}^m$ a uniformly random linear space of dimension
	$d$. For $\lambda\in (0,1)$, let $\bm P_{\lambda}$ denote the projector onto the first $m\lambda$ maximal atoms in $[m]$: assume that $s=\{I_1,...,I_d\}$, is the set of maximal atoms, then the $j^{th}$ column, $\bm P_{\lambda}[:,j]=\frac{1}{|I_r|}\sum_{i\in I_r} \bm e_i$ if $j\in I_r$ for some $r\leq m\lambda$; otherwise $\bm P_{\lambda}[:,j]=\bm 0$.
	Define $Z(\lambda) :=\sup\{ \| \bm P_{\lambda} \bm v\|\, :\;  \bm v\in V,\; \| \bm v\|=1\}$.
	Then, for any $\ve>0$ there exists $c(\ve)>0$ such that, for all $m$
	large enough (and $d$ fixed)
$
	\PP\{|Z(\kappa)-\sqrt{\lambda}|\ge \ve\}\le e^{-m\,c(\ve)}.
$
\end{proposition}

We next state a result due to Kashin \cite{kashin1977diameters} relating to the equivalence of $\ell^2$ and $\ell^1$ norms on random vector spaces (cf.\ also \cite[Theorem F.1]{lassorisk}).

\begin{theoremG}{\cite{kashin1977diameters}}
	\label{thm:Kashin}
	For any positive number $\upsilon$ there exist a universal constant $c_\upsilon$ such that for any
	$n\ge 1$, with probability at least $1-2^{-n}$, for a uniformly random subspace $V_{n,\upsilon}$ of dimension $\lfloor n(1-\upsilon)\rfloor$, for all $x\in V_{n,\upsilon}$, we have $c_\upsilon\|x\|_2\le \|x\|_1/\sqrt{n}.$
\end{theoremG}
%
%
Finally we state a general result about the limit behavior of extreme singular values of random matrices, as proved in \cite{bai2008limit} (cf. also \cite[Theorem F.2]{lassorisk}).
\begin{theoremG}{\cite{bai2008limit}}
	\label{prop:marchenko-pastur}
	Let $\bm A\in\mathbb{R}^{\n\times \p}$ have i.i.d.\ entries
	with $\E\{A_{ij}\}=0$, $\E\{ A_{ij}^2\}=1/\n$, and $\n/\p = \delta$.
	Let $\sigma_{\max}(\bm A)$ be it largest singular value, and
	$\hat{\sigma}_{\min}(\bm A)$ be its smallest non-zero singular value.
	Then,
	\begin{eqnarray*}
	\lim_{\p\to\infty}\sigma_{\max}(\bm A) &\overset{a.s.}= &{1}/{\sqrt{\delta}}+1, \qquad \text{ and } \qquad \lim_{\p\to\infty}\hat{\sigma}_{\min}(\bm A) \overset{a.s.}= {1}/{\sqrt{\delta}}-1.
	\end{eqnarray*}
\end{theoremG}

\end{document}